\documentclass[letterpaper]{article}
\usepackage[preprint]{aaai2027}
\usepackage[hyphens]{url}
\usepackage{graphicx}
\usepackage{natbib}
\usepackage{caption}
\usepackage{algorithm}
\usepackage{algorithmic}

\usepackage{amsthm}
\newtheorem{lemma}{Lemma}
\newtheorem{remark}{Remark}
\usepackage{newfloat}
\usepackage{listings}
\DeclareCaptionStyle{ruled}{labelfont=normalfont,labelsep=colon,strut=off}
\floatstyle{ruled}
\newfloat{listing}{tb}{lst}{}
\floatname{listing}{Listing}

\usepackage{booktabs}
\usepackage{multirow}
\usepackage{booktabs}
\usepackage{pifont}
\usepackage{xcolor}
\usepackage{amsfonts}
\usepackage{amsmath}

\usepackage{tikz}
\usetikzlibrary{positioning, calc, arrows.meta, shapes.geometric, backgrounds, fit, matrix, spy}
\usepackage{graphicx}
\usetikzlibrary{positioning}
\usepackage{subcaption}
\usepackage{rotating}
\usepackage{amssymb}\usepackage{enumitem}
\usepackage{cleveref}

\newcommand{\model}{TORF}

\usepackage{booktabs}
\usepackage{multirow}

\newcommand{\ross}{\text{ROSS}}
\newcommand{\scl}{\text{LSL}}

\newcommand{\cmark}{\textcolor{green!70!black}{\ding{51}}}
\newcommand{\xmark}{\textcolor{red}{\ding{55}}}

\definecolor{blockgray}{RGB}{195,195,195}
\definecolor{blockblue}{RGB}{110,165,225}
\definecolor{triblue}{RGB}{65,120,195}
\definecolor{tridark}{RGB}{25,75,155}
\definecolor{grnfill}{RGB}{205,245,205}

\definecolor{timeintra}{RGB}{230, 240, 255}
\definecolor{timeinter}{RGB}{180, 210, 255}
\definecolor{actcolor}{RGB}{255, 230, 200}
\definecolor{chanintra}{RGB}{230, 255, 230}
\definecolor{chaninter}{RGB}{180, 255, 180}

\tikzset{
	block/.style={draw, thick, rectangle, align=center, fill=gray!5},
	tensor/.style={draw, thick, fill=white, minimum width=1.5cm, minimum height=1.5cm, align=center},
	flow/.style={->, thick, >=Stealth, black},
	infer/.style={->, thick, >=Stealth, dashed, red},
	mat/.style={fill=white, draw=black, thick, rounded corners},
	mathlabel/.style={font=\small\itshape, text depth=0pt, text height=1ex}
}

\title{Two-stage Odd Residual Flows for Mean-Preserving Probabilistic Time Series Forecasting}
\author{
	Kiran Madhusudhanan\thanks{Corresponding author. Email: kiranmadhusud@ismll.de},
	Christian Kl\"{o}tergens,
	Lars Schmidt-Thieme, 
	Vijaya Krishna Yalavarthi
}
\affiliations{
    Institute of Computer Science\\
    University of Hildesheim\\
    Hildesheim, Germany
}

\begin{document}\sloppy

\maketitle

\begin{abstract}

Probabilistic forecasting plays an essential role in risk-sensitive decision-making, particularly in long-horizon settings. However, existing approaches often face a fundamental trade-off between distributional flexibility and accurate mean prediction. Traditional parametric methods, such as Mean Variance Estimation (MVE), can suffer from degraded point accuracy when trained under joint Negative Log-Likelihood (NLL) objectives, while modern-flexible generative models, including Normalizing Flows and Diffusion Models, typically rely on costly Monte Carlo sampling and may yield suboptimal mean estimates. To address this limitation, we propose Two-stage Odd Residual Flows (TORF), a framework that decouples mean forecasting from uncertainty estimation. In the first stage, a pre-trained deterministic model is used to produce an accurate mean prediction. In the second stage, a Restricted Normalizing Flow, with strictly odd functions learns flexible residual distributions around the point forecast, guaranteeing mean preservation from the first stage without sampling. Experiments show that TORF achieves state-of-the-art deterministic accuracy (NMAE) while providing strong density estimation performance (CRPS) on short and long-horizon forecasting.
\end{abstract}

\section{Introduction}\label{sec:intro}

Time Series Forecasting (TSF) is central to decision-making in domains such as demand planning, energy management, and finance \cite{bandara2019sales,Dimoulkas2018ElectricityLoad,luo2018neural}, where quantifying uncertainty is key to risk-aware decisions. While deterministic mean prediction has advanced rapidly \cite{tslib2026}, uncertainty quantification (UQ) \cite{Zhang2023ProbTSBP} still faces a trade-off between distributional flexibility and mean accuracy.

\begin{figure}[t]
	\centering
		\begin{minipage}{\columnwidth}
			\centering
			\begin{subfigure}{0.45\columnwidth}
				\centering
				\includegraphics[width=\linewidth]{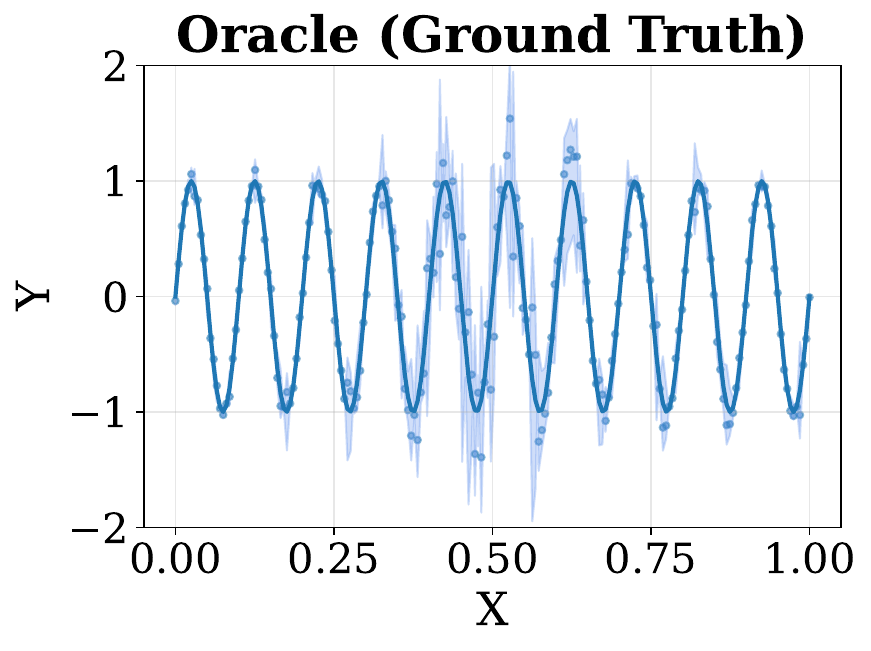}
				\caption{Oracle; $y = \sin(2\pi f x) + \text{Student-t}(\nu(x), 0, \gamma(x))$}
				\label{fig:oracle}
			\end{subfigure}
			\begin{subfigure}{0.45\columnwidth}
				\centering
				\includegraphics[width=\linewidth]{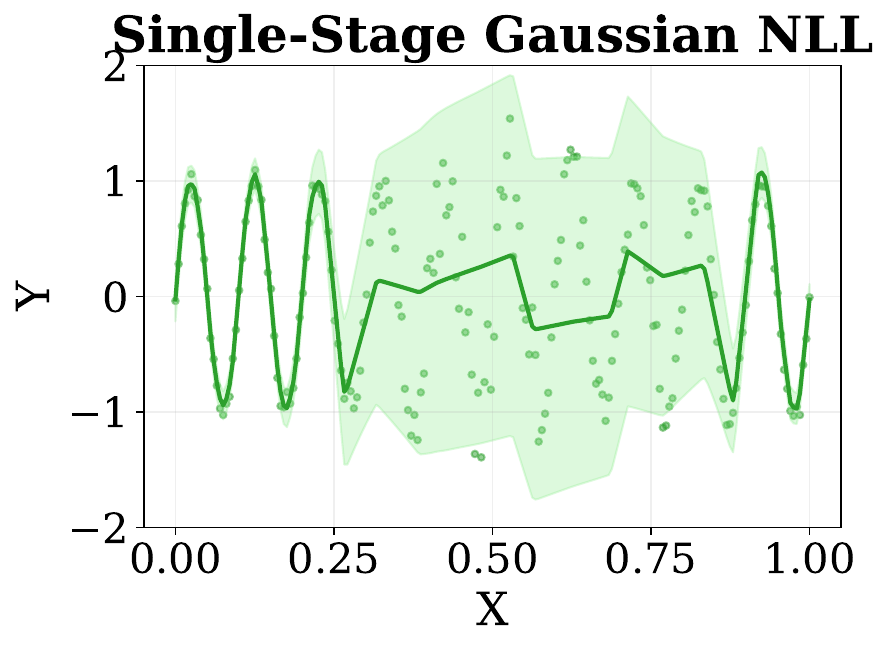}
				\captionsetup{justification=centering}
				\caption{Single-stage \\CRPS: 0.2628; MAE: 0.2628}
				\label{fig:single_stage}
			\end{subfigure}

			\vspace{0.2cm}

			\begin{subfigure}{0.45\columnwidth}
				\centering
				\includegraphics[width=\linewidth]{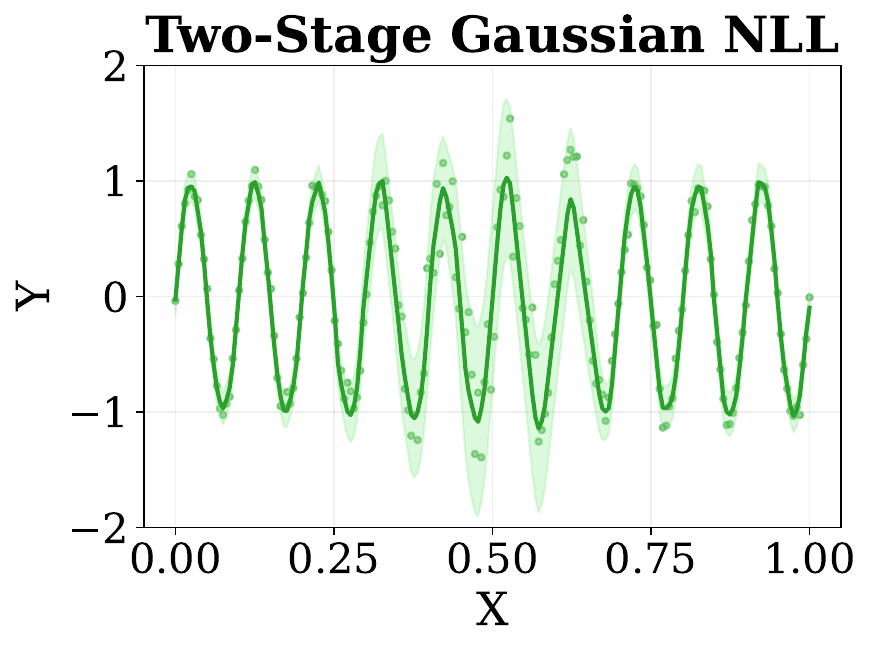}
				\captionsetup{justification=centering}
				\caption{Two-stage Gauss \\CRPS: 0.0986; MAE: 0.1383}
				\label{fig:two_stage}
			\end{subfigure}
			\begin{subfigure}{0.45\columnwidth}
				\centering
				\includegraphics[width=\linewidth]{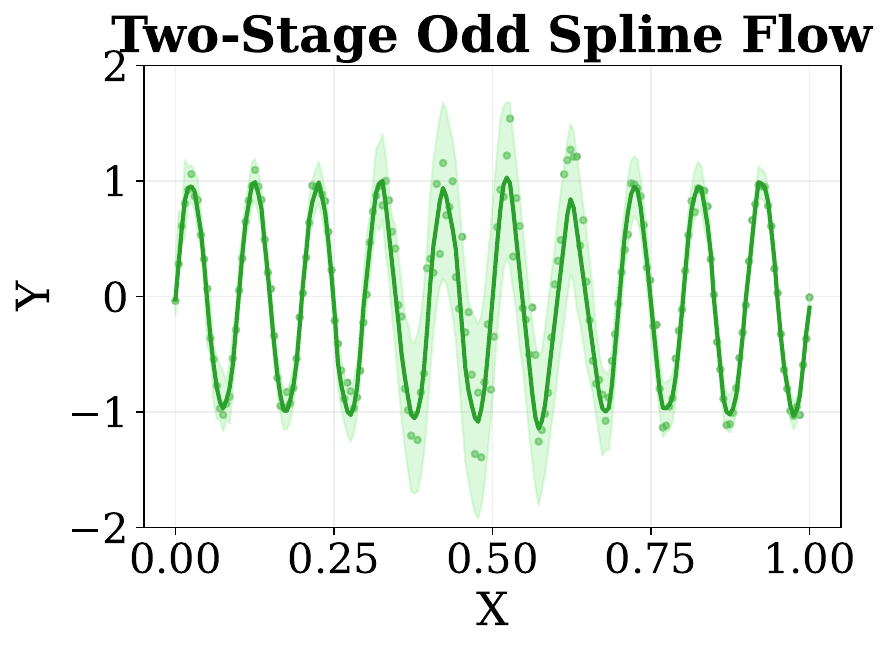}
				\captionsetup{justification=centering}
				\caption{Two-stage Flow \\CRPS: 0.0967; MAE: 0.1383}
				\label{fig:two_stage_odd_flow}
			\end{subfigure}
		\end{minipage}\caption{Collapse of Gaussian negative log likelihood learning based single stage approaches on a simple sinusoidal signal with Student's t-distribution noise.}
	\label{fig:pitfalls}
\end{figure}

The de-facto method for probabilistic prediction is Mean-Variance Estimation (MVE) \cite{Nix1994EstimatingTM}, which assumes a parametric target distribution (typically a heteroscedastic Gaussian) and optimizes its parameters via Negative Log-Likelihood (NLL) loss \cite{salinas2020deepar}. Although MVE methods provide an analytical, direct mean, their Gaussian distribution assumption is limiting in practice.
Furthermore, training MVE models via joint NLL optimization triggers an inherent trade-off between mean and variance learning \cite{Seitzer2022PitfallsOfUncertainty}.
For example, consider~\Cref{fig:oracle}, where the target follows a simple sinusoidal function with non-Gaussian noise,
$y = \sin(2\pi f x) + \text{Student-t}(\nu(x), 0, \gamma(x))$. Direct NLL training~\Cref{fig:single_stage} fails to recover the true mean in high-noise regions because the loss function heavily prioritizes low-noise regimes~\cite{Seitzer2022PitfallsOfUncertainty}, whereas a two-stage approach that first fits the mean
and then the variance recovers both~\Cref{fig:two_stage}.

\begin{figure}
	\centering
	\includegraphics[width=\linewidth]{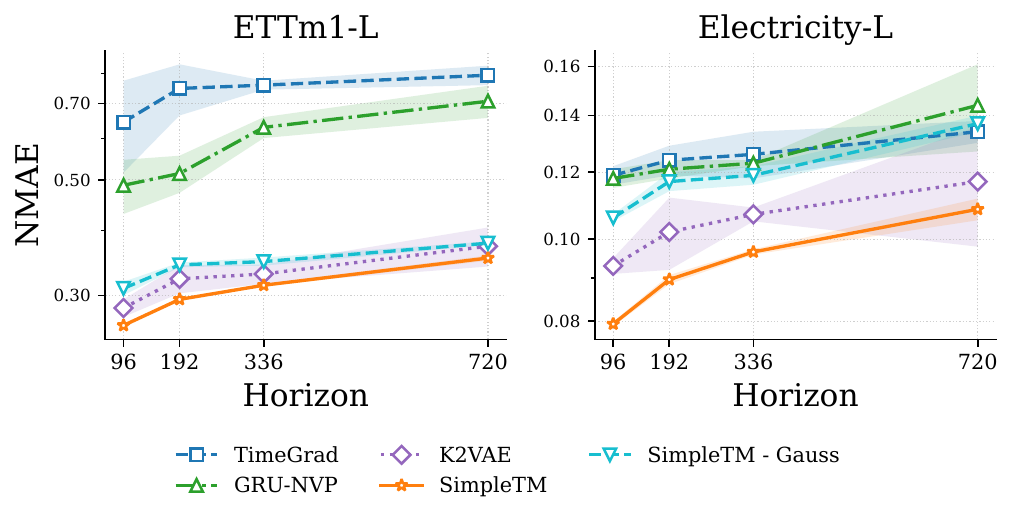}
	\caption{NMAE results on two time series datasets. Current probabilistic frameworks (TimeGrad, GRU-NVP, $K^2$VAE, SimpleTM-Gauss/MVE(SimpleTM)) exhibit higher NMAE than the point-prediction model (SimpleTM).}
	\label{fig:NMAEComparison}
\end{figure}

Flexible density modeling approaches based on Diffusion \cite{tashiro2021csdi}, Normalizing Flows \cite{rasulmultivariate}, or VAEs \cite{wu2025kvae} sidestep rigid distributional assumptions and can model complex densities. However, they inherit poor mean accuracy for two distinct reasons. First, like MVE, they optimize a distributional objective (e.g., the ELBO or a log-likelihood loss) rather than a mean-focused loss, so the predictive mean is only an implicit by-product of fitting the full density and is never directly supervised. Second, unlike MVE, they lack an analytical mean altogether: recovering it requires expensive Monte Carlo sampling, which introduces additional estimation error that worsens in high dimensions as samples become scarce relative to the space they must cover. These effects compound, and as a result state-of-the-art probabilistic models like $K^2$VAE~\citep{wu2025kvae} consistently trail deterministic point-predictors like SimpleTM~\citep{chen2025simpletm} in mean accuracy (\Cref{fig:NMAEComparison}).

To break this trade-off, we introduce Two-stage Odd Residual Flows (\model{}), completely isolating mean prediction from uncertainty estimation through a decoupled framework:
\textit{Stage 1 (Deterministic)}: An off-the-shelf point-prediction model (e.g., SimpleTM) captures the conditional mean. Free from NLL-driven variance distortion,
it retains a high point prediction accuracy.
\textit{Stage 2 (Probabilistic)}: A Restricted Normalizing Flow (RNF) \cite{Kobayashi2023DesignOR} models the residual density. By restricting the flow architecture to
strictly odd functions, we capture highly flexible, symmetric error distributions while mathematically guaranteeing that the exact analytical mean from Stage 1 is preserved without sampling.
As shown in \Cref{fig:two_stage_odd_flow} the two-stage approach retains the point-prediction accuracy of its Stage-1 baseline and builds the predictive uncertainty around that fixed mean, circumventing the joint NLL trade-off entirely.

\begin{enumerate}
	\item We show that a SimpleTM based two-stage Gaussian baseline (MVE-2S) already matches or beats $K^2$VAE on 11/12 distributional-accuracy comparisons (\Cref{tab:energyscore_crps}).
	\item We introduce \model{}, which extends this decoupled framework with an odd-constrained residual flow (\ross), modeling flexible, non-Gaussian residuals while provably preserving the exact Stage-1 mean, without sampling.
	\item We design a lightweight CNN architecture that efficiently maps time-series contexts to \ross{}'s parameters.
	\item \model{} with SimpleTM~\citep{chen2025simpletm} as 1st stage model sets a new state of the art, winning 7/9 CRPS and 9/9 NMAE on long-horizon forecasting and 6/8 CRPS and 5/8 NMAE on short-horizon forecasting, with \model{} adding a consistent gain over $K^2$VAE and MVE-2S.
\end{enumerate}

\section{Related works}
\label{sec:related}

\textbf{Deterministic forecasting} has progressed from classical statistical
models like ARIMA~\citep{BoxJenkins1994} to
modern deep architectures like Transformers.
These include Informer~\citep{zhou2021informer},
Autoformer~\citep{wu2021autoformer}, FEDformer~\citep{zhou2022fedformer},
Yformer~\citep{madhusudhanan2022u}, and Pyraformer
\citep{liu2022pyraformer} to name a few. Among them iTransformer~\citep{liu2024itransformer} and
PatchTST~\citep{nie2023patchtst} remain among the strongest performers.
Recently, SimpleTM~\citep{chen2025simpletm}
achieved state-of-the-art point accuracy combining wavelet preprocessing with a
geometric-product attention. Owing to its strong performance,
we adopt SimpleTM as the Stage~1 model in our experiments. \model{} is
agnostic to this choice.

\begin{table}[t]
\centering
\footnotesize
\begin{tabular}{lcccc}
\toprule
 & \shortstack{Analytic\\Mean} & \shortstack{Analytic\\Median} & \shortstack{Flexible\\Density} & \shortstack{Exact\\NLL} \\
\midrule
$K^2$VAE               & \xmark & \xmark & \cmark & \xmark \\
TimeGrad            & \xmark & \xmark & \cmark & \xmark \\
MVE              & \cmark & \cmark & \xmark & \cmark \\
GRU-NVP             & \xmark & \cmark & \cmark & \cmark \\
\textbf{\model{} (Ours)} & \cmark & \cmark & \cmark & \cmark \\
\bottomrule
\end{tabular}
\caption{Comparison of Probabilistic Forecasting Models.}
\label{tab:model_comparison}
\end{table}

\textbf{Probabilistic forecasting} aims to recover the full predictive distribution
rather than a single point estimate. Autoregressive models such as
DeepAR~\citep{salinas2020deepar} learn Gaussian parameters via an RNN.
Modern diffusion-based methods treat forecasting as an iterative denoising process,
like TimeGrad~\citep{rasul2021autoregressive}, CSDI~\citep{tashiro2021csdi}, and
TSDiff~\citep{kollovieh2023predict}. Conditional normalizing flows such as
GRU-NVP~\citep{rasulmultivariate} and TSFlow~\citep{kollovieh2025flow} instead
model complex joint densities without a fixed parametric form.
$K^2$VAE~\citep{wu2025kvae}, a VAE-based method for long-horizon probabilistic
forecasting, combines a Koopman latent dynamical system with a KalmanNet-based correction.
Despite this diversity, all of these methods either commit to a fixed distributional
family or fall back on Monte Carlo sampling to recover a predictive mean.
\Cref{tab:model_comparison} differentiates these approaches from \model{}, which decouples
mean estimation from density estimation entirely, obtaining both an analytical mean and a
flexible density, inspired by \citet{Kobayashi2023DesignOR}'s restriction of the flow
architecture to be odd.

\section{Preliminaries}
\label{sec:prelim}

\paragraph{Probabilistic Time Series Forecasting}
\label{sec:prelim:tsf}

A \emph{forecasting instance} is a tuple
$(\mathbf{X},\mathbf{Y})\in\mathbb{R}^{L\times C}\times\mathbb{R}^{H\times C}$,
where $\mathbf{X}=(\mathbf{z}_{t-L+1},\ldots,\mathbf{z}_{t})$ is the
look-back window of length $L$ and
$\mathbf{Y}=(\mathbf{z}_{t+1},\ldots,\mathbf{z}_{t+H})$ is the target
of horizon $H$ over $C$ channels.
Probabilistic forecasting aims to estimate the full conditional
distribution $p_\theta(\mathbf{Y}\mid\mathbf{X})$.
Given $N$ i.i.d.\ instances drawn from an unknown joint distribution
$P$, the standard training objective is the NLL:
\begin{equation}
	\min_\theta\;\mathbb{E}_{(\mathbf{X},\mathbf{Y})\sim P}
	\!\left[-\log p_\theta(\mathbf{Y}\mid\mathbf{X})\right].
	\label{eq:nll}
\end{equation}
A de-facto method is assuming $p$ to be Gaussian and train for Gaussian NLL.
\paragraph{Pitfalls of NLL Training}
\label{sec:prelim:pitfalls}

\citet{Seitzer2022PitfallsOfUncertainty} show that jointly training a Gaussian NLL head creates a mean-variance conflict. The gradient of the NLL with respect to the mean $\boldsymbol{\mu}$ is scaled element-wise by $1/\boldsymbol{\sigma}^{2}$, so high-variance regions contribute less to mean learning (\Cref{fig:pitfalls}). This effect worsens in high-dimensional, long-horizon settings, where density heads must model complex distributions, degrading point accuracy relative to a stand-alone mean predictor.

\paragraph{Faithful Heteroscedastic Regression}
\label{sec:prelim:faithful}

\citet{StirnWSPSK23} eliminate this conflict by \emph{fully decoupling} the two objectives. A first-stage mean predictor $f_{\theta_1}$ is trained under an MSE objective and frozen, then a second-stage network models only the residual distribution $p_{\theta_2}(\mathbf{Y}-f_{\theta_1}(\mathbf{X})\mid\mathbf{X})$. Because the mean network is fixed, the joint model's mean is guaranteed to be \emph{at least as accurate} as the stand-alone predictor, a property they call \emph{faithfulness}. However, they restrict the residual distribution to Gaussian, limiting expressiveness for heavy-tailed or multi-modal targets. \textsc{TORF} keeps this faithfulness guarantee while replacing the Gaussian residual with a Normalizing Flow, adding flexible, non-Gaussian density estimation without sacrificing mean accuracy.
\paragraph{Normalizing Flows}
\label{sec:prelim:nf}

A Normalizing Flow (NF)~\citep{papamakarios2021normalizing} defines a
bijective, differentiable map $f:\mathbb{R}^K\!\to\!\mathbb{R}^K$ from
a tractable base distribution $p_\mathbf{u}(\mathbf{u})$, typically
$\mathcal{N}(\mathbf{0},\mathbf{I})$, to a complex target
$p_\mathbf{v}(\mathbf{v})$ via the change-of-variables formula:
\begin{equation}
	p_\mathbf{v}(\mathbf{v};\theta)
	= p_\mathbf{u}\!\left(f^{-1}(\mathbf{v};\theta)\right)
	\left|\det\frac{\partial f^{-1}(\mathbf{v};\theta)}
	{\partial\mathbf{v}}\right|.
	\label{eq:cov}
\end{equation}
For conditional density estimation, $f$ is additionally conditioned on
the context $\mathbf{c}$, written $f(\mathbf{u};\mathbf{c},\theta)$.
Tractability requires an efficient inverse and a Jacobian
log-determinant computable in closed form.

\paragraph{Rational Quadratic Neural Spline Flows (RQ-NSF).}
\citet{durkan2019neural} proposed monotonic piecewise
rational-quadratic splines as the bijective transformation within the
flow, yielding highly non-linear yet analytically invertible maps.
Each dimension of the input is transformed independently by its own
univariate spline, making the Jacobian diagonal and its
log-determinant a sum of scalar log-derivatives.
The spline is supported on a bounded interval $[-B,B]$, defaulting to
the identity outside, and defined by $N_b$ bins parameterized by
widths $\mathbf{w}_b \in \mathbb{R}^{N_b}_{>0}$, heights
$\mathbf{h}_b \in \mathbb{R}^{N_b}_{>0}$, and interior knot derivatives
$\mathbf{d}_b \in \mathbb{R}^{N_b-1}_{>0}$, giving $3N_b - 1$ parameters
per dimension.
For a conditional distribution, these are predicted by a neural
network conditioned on $\mathbf{X}$:
\begin{equation}
	\boldsymbol{\phi} = (\mathbf{w}_b,\;\mathbf{h}_b,\;\mathbf{d}_b)
	= \mathrm{NN}(\mathbf{X};\theta_s),
	\qquad \boldsymbol{\phi} \in \mathbb{R}^{3N_b - 1}_{>0}.
	\label{eq:spline-params}
\end{equation}

\paragraph{Odd Flows for Analytically Tractable Mean}
\label{sec:prelim:odd}

For an unconstrained NF the predictive mean
$\mathbb{E}[\mathbf{v}]=\int \mathbf{v}\,p_\mathbf{v}(\mathbf{v})\,d\mathbf{v}$
is generally intractable.
\citet{Kobayashi2023DesignOR} showed that restricting $f$ to the class of
\emph{odd functions},
\begin{equation}
	f_{\mathrm{odd}}(-\mathbf{u}) = -f_{\mathrm{odd}}(\mathbf{u}),
	\quad\forall\,\mathbf{u}\in\mathbb{R}^K,
	\label{eq:odd}
\end{equation}
guarantees $\mathbb{E}[f_{\mathrm{odd}}(\mathbf{u})]=0$ whenever
$p_\mathbf{u}$ is symmetric about the origin, making the mean
analytically tractable without Monte Carlo sampling.
For spline flows, odd symmetry is enforced by requiring $f(0)=0$ and
mirroring bin parameters exactly across positive and negative domains.

\section{Methodology}
\label{sec:method}

\subsection{Two-Stage Residual Modeling}
\label{sec:method:twostage}

As established in \S\ref{sec:prelim:pitfalls}, jointly optimizing
mean and density under the NLL degrades mean accuracy.
\textsc{TORF} resolves this by fully decoupling the two objectives
into a sequential two-stage framework.

\paragraph{\textnormal{\emph{Stage 1: Mean Estimation.}}}
Any point-predictor $f_{\theta_1}$ trained independently under
an MSE objective to produce the conditional mean estimate
$\hat{\mathbf{Y}} = f_{\theta_1}(\mathbf{X})$.
\textsc{TORF} is model-agnostic at this stage i.e., any point-predictor
can serve as $f_{\theta_1}$.

\paragraph{\textnormal{\emph{Stage 2: Residual Density Estimation.}}}
Let $\boldsymbol{\varepsilon} = \mathbf{Y} - f_{\theta_1}(\mathbf{X})$ denote the residual.
Since $\mathbf{Y} \mapsto \boldsymbol{\varepsilon}$ is a translation, its Jacobian is the
identity and $p_\varepsilon(\boldsymbol{\varepsilon}\mid\mathbf{X}) =
p_Y(f_{\theta_1}(\mathbf{X})+\boldsymbol{\varepsilon}\mid\mathbf{X})$.
With $\theta_1$ frozen, a normalizing flow $p_{\theta_2}$ is trained
to model $p(\boldsymbol{\varepsilon}\mid\mathbf{X})$ directly, constrained by construction
to satisfy $\mathbb{E}_{p_{\theta_2}}[\boldsymbol{\varepsilon}\mid\mathbf{X}] = \mathbf{0}$,
and the full predictive distribution is recovered as:
\begin{equation}
	\hat p(\mathbf{Y}\mid\mathbf{X}) = p_{\theta_2}(\mathbf{Y} - f_{\theta_1}(\mathbf{X})\mid\mathbf{X}).
	\label{eq:pred}
\end{equation}
A sample $\hat{Y} \sim \hat p(\mathbf{Y}\mid\mathbf{X})$ can be read as $f_{\theta_1}(x) + \hat{\varepsilon}$ where $\hat\varepsilon\sim p_{\theta_2}(\varepsilon\mid\mathbf{X})$.
The design of $p_{\theta_2}$ is described in the following subsections. We make two assumptions on $p(\boldsymbol{\varepsilon}\mid\mathbf{X})$:

\begin{enumerate}
	\item \textbf{Conditional Independence.}
	$p(\boldsymbol{\varepsilon}\mid\mathbf{X}) =
	\displaystyle\prod_{c,h} p(\varepsilon_{h,c}\mid\mathbf{X})$
	\item \textbf{Symmetry.}
	$p(\boldsymbol{\varepsilon}\mid\mathbf{X}) = p(-\boldsymbol{\varepsilon}\mid\mathbf{X})$
\end{enumerate}
These assumptions trade a small amount of distributional generality for the exact mean-preservation guarantee at the core of our method, and they are far milder in practice than they may appear: even under both, \model{}\ outperforms state-of-the-art probabilistic forecasters on CRPS, their primary evaluation metric. In the ablation study, we introduce a mixture-of-\model{}\ extension that captures asymmetric distributions. In future work, we discuss extending \model\ to multivariate distributions.

\subsection{Context Embedding}
\label{sec:method:embed}

The flow components are conditioned on $\mathbf{X}$ through a shared context
embedding $\mathbf{F} \in \mathbb{R}^{H \times C}$.
There are two natural choices for constructing $\mathbf{F}$: (i) introduce a
separate encoder trained from scratch alongside the second stage, or
(ii) reuse the encoder learned by $f_{\theta_1}$.
When $f_{\theta_1}$ has differentiable parameters, as is the case for
neural network-based predictors, option (ii) is preferable, as the
encoder has already learned high-quality temporal representations
under MSE, providing a strong initialization at no additional
parameter cost.
When $f_{\theta_1}$ is non-differentiable (e.g., gradient boosting),
option (i) is necessary, and a separate neural encoder such as
SimpleTM can be used instead.

In our experiments, $f_{\theta_1}$ is SimpleTM, so we adopt option
(ii).
All $f_{\theta_1}$ layers are frozen and the final linear prediction head is duplicated and
kept trainable, allowing $\mathbf{F}$ to
adapt to the needs of the flow without disturbing the learned
encoder:
\begin{equation}
	\mathbf{F} = \mathbf{W}\cdot\mathrm{sg}\!\left[\mathrm{enc}(\mathbf{X};\, \theta_1)\right] + \mathbf{b},
	\label{eq:ctx}
\end{equation}
Here, $\mathrm{enc}(\mathbf{X};\,\theta_1)$ denotes the output of frozen encoder layers of $f_{\theta_1}$. $\mathbf{W}$ and $\mathbf{b}$ are the parameters of the duplicated, unfrozen final layer of $f_{\theta_1}$, which are initialized from the pre-trained weights and fine-tuned during the second stage. $\mathrm{sg}\left[\cdot\right]$ denotes the stop-gradient operator, which blocks gradients from flowing into $\mathrm{enc}(\mathbf{X};\,\theta_1)$. All subsequent modules take $\mathbf{F}$ as their conditioning input.

\begin{algorithm}[t]
	\caption{\textsc{TORF} Forward Pass}
	\label{alg:torf}
	\textbf{Input}: history $\mathbf{X}$, target $\mathbf{Y}$,
	pretrained $f_{\theta_1}$\\
	\textbf{Parameter}: flow layers $K$, $\theta_2 = \{\mathbf{W}, \mathbf{b}, \boldsymbol{\omega},
	\boldsymbol{\Theta_s}\}$\\
	\textbf{Output}: training loss $\mathcal{L}$
	\begin{algorithmic}[1]
	\STATE $\hat{\mathbf{Y}} \leftarrow f_{\theta_1}(\mathbf{X})$
	\STATE $\mathbf{F} \leftarrow \mathbf{W} \cdot \mathrm{sg}[\mathrm{enc}(\mathbf{X};\,\theta_1)] + \mathbf{b}$
	\STATE $\mathbf{U}^{(0)} \leftarrow \mathbf{Y} - \hat{\mathbf{Y}}$
	\STATE $\mathbf{S} \leftarrow \mathrm{ScaleNet}(\mathbf{F};\,\boldsymbol{\omega})$
	\STATE $\mathrm{ldj} \leftarrow 0$
	\FOR{$k = 0$ \TO $K - 1$}
	\STATE $\boldsymbol{\Phi}^{(k)} \leftarrow \mathrm{SplineNet}^{(k)}(\mathbf{F};\,\boldsymbol{\theta}_s^{(k)})$
	\STATE $\mathbf{U}^{(k+\frac{1}{2})},\,\mathrm{ldj}^{(k)}_{\mathrm{s}}
	\leftarrow \scl(\mathbf{U}^{(k)};\,\mathbf{S})$
	\STATE $\mathbf{U}^{(k+1)},\,\mathrm{ldj}^{(k)}_{\mathrm{r}}
	\leftarrow \ross(\mathbf{U}^{(k+\frac{1}{2})};\,\boldsymbol{\Phi}^{(k)})$
	\STATE $\mathrm{ldj} \leftarrow \mathrm{ldj}
	+ \mathrm{ldj}^{(k)}_{\mathrm{s}} + \mathrm{ldj}^{(k)}_{\mathrm{r}}$
	\ENDFOR
	\STATE $\mathbf{U}^{(K+1)},\,\mathrm{ldj}^{(K)}_{\mathrm{s}}
	\leftarrow \scl(\mathbf{U}^{(K)};\,\mathbf{S})$
	\STATE $\mathrm{ldj} \leftarrow \mathrm{ldj} + \mathrm{ldj}^{(K)}_{\mathrm{s}}$
	\RETURN $\mathcal{L} = -\log \mathcal{N}\!\left(\mathbf{U}^{(K+1)};\,
	\mathbf{0},\mathbf{I}\right) - \mathrm{ldj}$
	\end{algorithmic}
\end{algorithm}

\subsection{Residual Odd Splines (\ross)}
\label{sec:method:ross}

An unconstrained $p_{\theta_2}$ can introduce a location shift in
$p(\boldsymbol{\varepsilon}\mid\mathbf{X})$, decoupling the predictive mean of the model from
$f_{\theta_1}(\mathbf{X})$.
We prevent this by restricting the spline to the class of odd
functions (\S\ref{sec:prelim:odd}), analytically enforcing
$\mathbb{E}_{p_{\theta_2}}[\boldsymbol{\varepsilon}\mid\mathbf{X}]=\mathbf{0}$.

Let $S$ be a monotonic rational-quadratic spline defined over
$[0,B]$, defaulting to the identity outside.
Rather than mirroring bin parameters across both domains, we process
only the magnitude of the input through $S$ and restore the sign $\left(\mathrm{sgn}(\cdot)\right)$
afterward.
The forward and inverse transformations for each scalar element
$u_{h,c}$ are:
\begin{align}
	v_{h,c} &= \mathrm{sgn}(u_{h,c})\cdot
	S\!\left(|u_{h,c}|;\,\boldsymbol{\phi}_{h,c}\right),
	\label{eq:ross-fwd}\\
	u_{h,c} &= \mathrm{sgn}(v_{h,c})\cdot
	S^{-1}\!\left(|v_{h,c}|;\,\boldsymbol{\phi}_{h,c}\right),
	\label{eq:ross-inv}
\end{align}
where $\boldsymbol{\phi}_{h,c} = (\mathbf{w}_b, \mathbf{h}_b, \mathbf{d}_b)$ are the
per-element spline parameters defined in \S\ref{sec:prelim:nf}.
Since the spline maps $[0,B]$ onto itself, $S(0)=0$ holds by
construction, and the convention $\mathrm{sgn}(0)=1$ then ensures
continuity at the origin; $S^{-1}$ is available analytically since
$S$ is monotone.
This focuses the entire capacity of $S$ on the non-negative
half-axis, with odd symmetry satisfied exactly by construction.
The element-wise application yields a diagonal Jacobian whose
forward-direction log-determinant is
$\mathrm{ldj}_{\mathrm{r}} = \sum_{h,c}\log S'(|u_{h,c}|;\,\boldsymbol{\phi}_{h,c})$.
Combining \eqref{eq:ross-fwd}--\eqref{eq:ross-inv} and the
log-determinant:
\begin{equation}
	\mathbf{V},\, \mathrm{ldj}_{\mathrm{r}} = \ross(\mathbf{U};\, \boldsymbol{\Phi}).
	\label{eq:ross-op}
\end{equation}

\paragraph{SplineNet: Parameterizing \ross.}
Recent works have adopted spline-based flows for probabilistic
regression~\citep{Madhusudhanan2025.TabResFlow} and
forecasting~\citep{yalavarthi2024marginalization},
parameterizing $\boldsymbol{\phi}$ via MLPs.
However, in the multivariate sequential setting, predicting
$\boldsymbol{\Phi} \in \mathbb{R}_{>0}^{H\times C\times(3N_b-1)}$ naively with
an MLP incurs $O(H^2C^2)$ parameter cost, tractable for univariate
targets but prohibitive for long-horizon forecasting.
Since $\mathbf{F}$ already encodes long-range temporal structure from
$f_{\theta_1}$, local refinements over the time axis are sufficient
to produce spline parameters at each time step.
We further assume that residual distributions vary smoothly across
time, making local convolution a reasonable architectural choice.
We therefore use a lightweight 1D convolutional network operating
over the time axis of $\mathbf{F}$:
\begin{equation}
	\boldsymbol{\Phi} = \mathrm{Conv1D}\!\left(\mathrm{ReLU}\!\left(
	\mathrm{Conv1D}(\mathbf{F})\right)\right)
	\in \mathbb{R}^{H \times C \times (3N_b-1)},
	\label{eq:conv}
\end{equation}
where the first Conv1D maps $\mathbf{F}$ to a
hidden representation via $C_{\mathrm{hid}}$ filters.
$\boldsymbol{\Phi}$ is sliced along the last dimension to yield $\mathbf{w}_b$,
$\mathbf{h}_b$, and $\mathbf{d}_b$, each passed through a softplus to ensure
positivity.
This reduces parameter cost to $O(k \cdot C_{\mathrm{hid}} \cdot
N_b)$ for kernel size $k \ll H$, independent of the forecast
horizon.

\subsection{Linear Scaling Layer (\scl)}
\label{sec:method:scl}

\ross\ is defined over $[0, B]$, defaulting to the identity outside.
Residuals $\boldsymbol{\varepsilon}$ are not guaranteed to lie within this range,
reducing the effective capacity of the spline.
We therefore prepend a conditional scaling layer that adapts the
input range to the spline support:
\begin{equation}
	v_{h,c} = u_{h,c} \cdot s_{h,c},
	\qquad
	u_{h,c} = v_{h,c} / s_{h,c},
	\label{eq:scl}
\end{equation}
where $s_{h,c} > 0$ is a positive scale conditioned on $\mathbf{X}$,
predicted independently per channel by a network whose parameters are shared across channels (see ScaleNet below).
\citet{Yalavarthi_Scholz_Born_Schmidt-Thieme_2025} use both scaling
and translation as element-wise flow components; however, a
translation term shifts the residual distribution, breaking odd
symmetry and decoupling the predictive mean of $p(\mathbf{Y}\mid\mathbf{X})$
from $f_{\theta_1}(\mathbf{X})$.
\scl\ therefore restricts to scaling only, which is an odd function
and preserves $\mathbb{E}_{p_{\theta_2}}[\boldsymbol{\varepsilon}\mid\mathbf{X}]=\mathbf{0}$.
Its forward-direction log-determinant is:
\begin{equation}
	\mathbf{V},\, \mathrm{ldj}_{\mathrm{s}} = \scl(\mathbf{U};\, \mathbf{S}),
	\qquad \mathrm{ldj}_{\mathrm{s}} = \textstyle\sum_{hc} \log s_{h,c}.
	\label{eq:scl-op}
\end{equation}

\paragraph{ScaleNet: Parameterizing \scl.}
The scale matrix $\mathbf{S} \in \mathbb{R}^{H\times C}_{>0}$ is predicted from $\mathbf{F}$
via a bounded MLP applied independently at each channel:
\begin{equation}
	\mathbf{s}_{:,c} = \exp\!\left(a\cdot\tanh\!\left(
	\frac{\mathrm{MLP}(\mathbf{F}_{:,c};\,\boldsymbol{\omega})}{a}\right)\right),
	\qquad \mathbf{s}_{:,c} \in \mathbb{R}^{H}_{>0},
	\label{eq:scale}
\end{equation}
where $\mathbf{F}_{:,c} \in \mathbb{R}^{H}$ is the context of channel $c$,
and $a > 0$ constrains each $\log s_{h,c} \in [-a, a]$ keeping the scale bounded.
The parameters $\boldsymbol{\omega}$ are shared across all
$C$ channels.

\begin{table*}[t]
    \centering
    \resizebox{\textwidth}{!}{\footnotesize
    \setlength{\tabcolsep}{1mm}
    \begin{tabular}{c|c|cccccccccc}
    \toprule
        Model & Metric & ETTm1-L & ETTm2-L & ETTh1-L & ETTh2-L & Electricity-L & Traffic-L & Weather-L & Exchange-L & ILI-L \\ \midrule
        \multirow{2}{*}{PatchTST} & CRPS &$0.304\scriptstyle\pm0.029$	&${{0.229}\scriptstyle\pm0.036}$	&$0.323\scriptstyle\pm0.020$	&$0.304\scriptstyle\pm0.018$	&$0.127\scriptstyle\pm0.015$	&$0.214\scriptstyle\pm0.001$	&$0.142\scriptstyle\pm0.005$	&$0.097\scriptstyle\pm0.007$
        &$0.233 \scriptstyle{\pm 0.019}$ \\
        ~ & NMAE & $0.382\scriptstyle\pm0.066$	&${{0.288}\scriptstyle\pm0.034}$	&$0.428\scriptstyle\pm0.024$	&$0.371\scriptstyle\pm0.021$	&$0.164\scriptstyle\pm0.024$	&${{0.253}\scriptstyle\pm0.012}$	&$0.152\scriptstyle\pm0.029$	&$0.126\scriptstyle\pm0.001$	&$0.287\scriptstyle\pm0.023$ \\ \hline
        \multirow{2}{*}{iTrans.} & CRPS & $0.455\scriptstyle\pm0.021$	&$0.311\scriptstyle\pm0.024$	&$0.350\scriptstyle\pm0.019$	&$0.542\scriptstyle\pm0.015$	&$0.109\scriptstyle\pm0.044$	&$0.284\scriptstyle\pm0.004$	&$0.133\scriptstyle\pm0.004$	&$0.087\scriptstyle\pm0.023$
        &$0.222 \scriptstyle{\pm 0.020}$ \\
        ~ & NMAE &$0.490\scriptstyle\pm0.038$	&$0.385\scriptstyle\pm0.042$	&$0.449\scriptstyle\pm0.022$	&$0.667\scriptstyle\pm0.012$	&$0.140\scriptstyle\pm0.009$	&$0.361\scriptstyle\pm0.030$	&$0.147\scriptstyle\pm0.019$	&$0.113\scriptstyle\pm0.015$	&$0.278\scriptstyle\pm0.017$ \\ \hline
        \multirow{2}{*}{GRU NVP} & CRPS &$0.546\scriptstyle\pm0.036$	&$0.561\scriptstyle\pm0.273$	&$0.502\scriptstyle\pm0.039$	&$0.539\scriptstyle\pm0.090$	&$0.114\scriptstyle\pm0.013$	&${{0.211}\scriptstyle\pm0.004}$	&$0.110\scriptstyle\pm0.004$	&$0.079\scriptstyle\pm0.009$	&$0.307\scriptstyle\pm0.005$  \\
        ~ & NMAE & $0.707\scriptstyle\pm0.050$	&$0.749\scriptstyle\pm0.385$	&$0.643\scriptstyle\pm0.046$	&$0.688\scriptstyle\pm0.161$	&$0.144\scriptstyle\pm0.017$	&$0.264\scriptstyle\pm0.006$	&$0.135\scriptstyle\pm0.008$	&$0.103\scriptstyle\pm0.009$	&$0.333\scriptstyle\pm0.005$ \\ \hline
        \multirow{2}{*}{TimeGrad} & CRPS & $0.621\scriptstyle\pm0.037$	&$0.470\scriptstyle\pm0.054$	&$0.523\scriptstyle\pm0.027$	&$0.445\scriptstyle\pm0.016$	&$0.108\scriptstyle\pm0.003$	&$0.220\scriptstyle\pm0.002$	&$0.113\scriptstyle\pm0.011$	&$0.099\scriptstyle\pm0.015$	&$0.295\scriptstyle\pm0.083$ \\
        ~ & NMAE & $0.793\scriptstyle\pm0.034$	&$0.561\scriptstyle\pm0.044$	&$0.672\scriptstyle\pm0.015$	&$0.550\scriptstyle\pm0.018$	&${{0.134}\scriptstyle\pm0.004}$	&$0.263\scriptstyle\pm0.001$	&$0.136\scriptstyle\pm0.020$	&$0.113\scriptstyle\pm0.016$	&$0.325\scriptstyle\pm0.068$  \\ \hline
        \multirow{2}{*}{CSDI} & CRPS & $0.448\scriptstyle\pm0.038$	&$0.239\scriptstyle\pm0.035$	&$0.528\scriptstyle\pm0.012$	&$0.302\scriptstyle\pm0.040$	&\textemdash	&\textemdash	&${{0.087}\scriptstyle\pm0.003}$	&$0.143\scriptstyle\pm0.020$	&$0.283\scriptstyle\pm0.012$  \\
        ~ & NMAE & $0.578\scriptstyle\pm0.051$	&$0.306\scriptstyle\pm0.040$	&$0.657\scriptstyle\pm0.014$	&$0.382\scriptstyle\pm0.030$	&\textemdash	&\textemdash	&$ {{0.102}\scriptstyle\pm0.005}$	&$0.173\scriptstyle\pm0.020$	&$0.299\scriptstyle\pm0.013$ \\ \hline
        \multirow{2}{*}{$K^2$VAE} & CRPS & $\underline{\textit{0.294}}\scriptstyle\pm0.026$	&$\underline{\textit{0.221}}\scriptstyle\pm0.023$	&$\underline{\textit{0.314}}\scriptstyle\pm0.011$	&$\underline{\textit{0.280}}\scriptstyle\pm0.014$	&$\textbf{0.057}\scriptstyle\pm0.005$	&$\textbf{0.200}\scriptstyle\pm0.001$	&$\underline{\textit{0.084}}\scriptstyle\pm0.003$	&$\underline{\textit{0.069}}\scriptstyle\pm0.005$	&$\underline{\textit{0.142}} \scriptstyle{\pm 0.008}$ \\
        ~ & NMAE & $\underline{\textit{0.373}}\scriptstyle\pm0.032$	&$\underline{\textit{0.275}}\scriptstyle\pm0.035$	&$\underline{\textit{0.396}}\scriptstyle\pm0.012$	&$\underline{\textit{0.278}}\scriptstyle\pm0.020$	&$\underline{\textit{0.117}}\scriptstyle\pm0.019$	&$\underline{\textit{0.248}}\scriptstyle\pm0.010$	&$\underline{\textit{0.099}}\scriptstyle\pm0.009$	&$\underline{\textit{0.084}}\scriptstyle\pm0.017$	&$\underline{\textit{0.167}}\scriptstyle\pm0.007$\\ \hline
        \multirow{2}{*}{\model{}} & CRPS & $\textbf{0.279}\scriptstyle\pm0.001$	&$\textbf{0.166}\scriptstyle\pm0.000$	&$\textbf{0.286}\scriptstyle\pm0.001$	&$\textbf{0.181}\scriptstyle\pm0.000$	&$\underline{\textit{0.081}}\scriptstyle\pm0.000$	&$\underline{\textit{0.203}}\scriptstyle\pm0.001$	&$\textbf{0.079}\scriptstyle\pm0.001$	&$\textbf{0.059}\scriptstyle\pm0.000$	&$\textbf{0.131} \scriptstyle{\pm 0.000}$ \\
        ~ & NMAE & $\textbf{0.354}\scriptstyle\pm0.003$	&$\textbf{0.206}\scriptstyle\pm0.001$	&$\textbf{0.367}\scriptstyle\pm0.001$	&$\textbf{0.229}\scriptstyle\pm0.020$	&$\textbf{0.108}\scriptstyle\pm0.003$	&$\textbf{0.242}\scriptstyle\pm0.004$	&$\textbf{0.096}\scriptstyle\pm0.001$	&$\textbf{0.080}\scriptstyle\pm0.000$	&$\textbf{0.158}\scriptstyle\pm0.007$\\ \hline
        \multirow{2}{*}{Imp (\%)} & CRPS & $+5.1$ & $+24.9$ & $+8.9$ & $+35.4$ & $-42.1$ & $-1.5$ & $+6.0$ & $+14.5$ & $+7.7$ \\
        ~ & NMAE & $+5.1$ & $+25.1$ & $+7.3$ & $+17.6$ & $+7.7$ & $+2.4$ & $+3.0$ & $+4.8$ & $+5.4$ \\
        \bottomrule
    \end{tabular}}
    \caption{Comparison on long-term probabilistic forecasting (horizon 720) across nine real-world datasets. Lower CRPS/NMAE is better. Means and standard errors are from 5 independent runs. \textbf{Bold} = best, \underline{\textit{italic}} = second best. Full results for all four horizons ($96, 192, 336, 720$) are in Appendix \Cref{tab:long_term_fore_CRPS,tab:long_term_fore_NMAE}. '-' indicates results required excessive time and memory consumption.}
    \label{tab: long-term}
\end{table*}

\begin{table*}[t]
    \centering
    \footnotesize
    \setlength{\tabcolsep}{1.5mm}
    \begin{tabular}{c|c|cccccccc}
    \toprule
        Model & Metric & Exchange-S & Solar-S & Electricity-S & Traffic-S & ETTh1-S & ETTh2-S & ETTm1-S & ETTm2-S  \\ \midrule
        \multirow{2}{*}{PatchTST} & CRPS & $0.052\scriptstyle{\pm 0.016}$ & $0.491\scriptstyle{\pm 0.008}$ & $0.063\scriptstyle{\pm 0.003}$ & $0.278\scriptstyle{\pm 0.018}$ & $0.314\scriptstyle{\pm 0.022}$ & $0.207\scriptstyle{\pm 0.006}$ & $0.234\scriptstyle{\pm 0.011}$ & $0.212\scriptstyle{\pm 0.018}$  \\
        ~ & NMAE & $0.069\scriptstyle{\pm 0.013}$ & $0.663\scriptstyle{\pm 0.010}$ & $0.085\scriptstyle{\pm 0.006}$ & $0.363\scriptstyle{\pm 0.023}$ & $0.407\scriptstyle{\pm 0.030}$ & $0.260\scriptstyle{\pm 0.009}$ & $0.271\scriptstyle{\pm 0.009}$ & $0.257\scriptstyle{\pm 0.011}$  \\ \hline
        \multirow{2}{*}{iTrans.} & CRPS & $0.059\scriptstyle{\pm 0.018}$ & $0.504\scriptstyle{\pm 0.012}$ & $0.066\scriptstyle{\pm 0.004}$ & $0.244\scriptstyle{\pm 0.011}$ & $0.317\scriptstyle{\pm 0.020}$ & $0.219\scriptstyle{\pm 0.008}$ & $0.254\scriptstyle{\pm 0.012}$ & $0.201\scriptstyle{\pm 0.018}$  \\
        ~ & NMAE & $0.081\scriptstyle{\pm 0.022}$ & $0.695\scriptstyle{\pm 0.017}$ & $0.087\scriptstyle{\pm 0.006}$ & $0.319\scriptstyle{\pm 0.019}$ & $0.408\scriptstyle{\pm 0.028}$ & $0.276\scriptstyle{\pm 0.017}$ & $0.291\scriptstyle{\pm 0.017}$ & $0.242\scriptstyle{\pm 0.009}$
        \\ \hline
        \multirow{2}{*}{GRU NVP} & CRPS & $0.019\scriptstyle{\pm 0.006}$ & $0.530\scriptstyle{\pm 0.008}$ & $0.062\scriptstyle{\pm 0.003}$ & $0.168\scriptstyle{\pm 0.008}$ & $0.398\scriptstyle{\pm 0.034}$ & $0.309\scriptstyle{\pm 0.023}$ & $0.455\scriptstyle{\pm 0.029}$ & $0.276\scriptstyle{\pm 0.014}$  \\
        ~ & NMAE & $0.024\scriptstyle{\pm 0.007}$ & $0.670\scriptstyle{\pm 0.011}$ & $0.081\scriptstyle{\pm 0.006}$ & $0.209\scriptstyle{\pm 0.013}$ & $0.477\scriptstyle{\pm 0.040}$ & $0.375\scriptstyle{\pm 0.024}$ & $0.584\scriptstyle{\pm 0.047}$ & $0.349\scriptstyle{\pm 0.028}$  \\ \hline
        \multirow{2}{*}{TimeGrad} & CRPS & $\underline{\textit{0.009}}\scriptstyle{\pm 0.001}$& $0.465\scriptstyle{\pm 0.016}$ & $0.057\scriptstyle{\pm 0.002}$ & ${{0.130}\scriptstyle{\pm 0.005}}$ & $0.273\scriptstyle{\pm 0.007}$ &$0.184\scriptstyle{\pm 0.006}$ & $0.186\scriptstyle{\pm 0.003}$ & $0.148\scriptstyle{\pm 0.004}$  \\
        ~ & NMAE & $\underline{\textit{0.012}}\scriptstyle{\pm 0.002}$& $0.609\scriptstyle{\pm 0.015}$ & $0.073\scriptstyle{\pm 0.004}$ & $\textbf{0.155}\scriptstyle{\pm 0.007}$ & $0.356\scriptstyle{\pm 0.013}$ & $0.224\scriptstyle{\pm 0.014}$ &$ 0.246\scriptstyle{\pm 0.007}$ & $0.189\scriptstyle{\pm 0.006}$  \\ \hline
        \multirow{2}{*}{CSDI} & CRPS & $\underline{\textit{0.009}}\scriptstyle{\pm 0.001}$ & $\underline{\textit{0.392}}\scriptstyle{\pm 0.006}$ & $\textbf{0.051}\scriptstyle{\pm 0.001}$ & $0.147\scriptstyle{\pm 0.014}$ & $0.262\scriptstyle{\pm 0.012}$ & ${{0.133}\scriptstyle{\pm 0.006}}$ & $0.140\scriptstyle{\pm 0.012}$ & $0.144\scriptstyle{\pm 0.018}$  \\
        ~ & NMAE & $0.013\scriptstyle{\pm 0.001}$& $\underline{\textit{0.533}}\scriptstyle{\pm 0.007}$& $\textbf{0.066}\scriptstyle{\pm 0.001}$ & $0.175\scriptstyle{\pm 0.013}$ & $0.339\scriptstyle{\pm 0.009}$ & ${{0.161}\scriptstyle{\pm 0.013}}$ & $0.169\scriptstyle{\pm 0.021}$ & $0.181\scriptstyle{\pm 0.024}$  \\ \hline
        \multirow{2}{*}{$K^2$VAE} & CRPS & $\underline{\textit{0.009}}\scriptstyle{\pm 0.001}$& $\textbf{0.367}\scriptstyle{\pm 0.006}$ & $\underline{\textit{0.053}}\scriptstyle{\pm 0.002}$ & $\underline{\textit{0.129}}\scriptstyle{\pm 0.004}$ & $\underline{\textit{0.256}}\scriptstyle{\pm 0.008}$ & $\underline{\textit{0.128}}\scriptstyle{\pm 0.006}$ & $\underline{\textit{0.135}}\scriptstyle{\pm 0.008}$ & $\underline{\textit{0.122}}\scriptstyle{\pm 0.008}$  \\
        ~ & NMAE & $\textbf{0.009}\scriptstyle{\pm 0.001}$ & $\textbf{0.480}\scriptstyle{\pm 0.008}$ & $\underline{\textit{0.068}}\scriptstyle{\pm 0.002}$& $\underline{\textit{0.157}}\scriptstyle{\pm 0.007}$ & $\underline{\textit{0.312}}\scriptstyle{\pm 0.008}$ & $\underline{\textit{0.140}}\scriptstyle{\pm 0.007}$ & $\underline{\textit{0.152}}\scriptstyle{\pm 0.007}$ & $\underline{\textit{0.146}}\scriptstyle{\pm 0.009}$ \\ \hline
        \multirow{2}{*}{\model{}} & CRPS & $\textbf{0.008}\scriptstyle{\pm 0.000}$& ${{0.456}\scriptstyle{\pm 0.002}}$ & $\underline{\textit{0.053}}\scriptstyle{\pm 0.000}$ & $\textbf{0.126}\scriptstyle{\pm 0.000}$ & $\textbf{0.231}\scriptstyle{\pm 0.001}$ & $\textbf{0.106}\scriptstyle{\pm 0.000}$ & $\textbf{0.108}\scriptstyle{\pm 0.000}$ & $\textbf{0.085}\scriptstyle{\pm 0.000}$  \\
        ~ & NMAE & $\textbf{0.009}\scriptstyle{\pm 0.001}$ & ${{0.583}\scriptstyle{\pm 0.025}}$ & ${{0.070}\scriptstyle{\pm 0.001}}$& ${{0.159}\scriptstyle{\pm 0.000}}$ & $\textbf{0.286}\scriptstyle{\pm 0.001}$ & $\textbf{0.132}\scriptstyle{\pm 0.001}$ & $\textbf{0.139}\scriptstyle{\pm 0.001}$ & $\textbf{0.105}\scriptstyle{\pm 0.003}$ \\ \hline
        \multirow{2}{*}{Imp (\%)} & CRPS & $+11.1$ & $-24.3$ & $-3.9$ & $+2.3$ & $+9.8$ & $+17.2$ & $+20.0$ & $+30.3$ \\
        ~ & NMAE & $+0.0$ & $-21.5$ & $-6.1$ & $-2.6$ & $+8.3$ & $+5.7$ & $+8.6$ & $+28.1$ \\
        \bottomrule
    \end{tabular}
    \caption{Comparison on short-term probabilistic forecasting across eight real-world datasets. Lower CRPS/NMAE is better. Means and standard errors are from 5 independent runs. \textbf{Bold} = best, \underline{\textit{Italic}} = second best.}
    \label{tab: short-term}
\end{table*}

\subsection{TORF Architecture}
\textsc{TORF} composes \scl\ and \ross\ into a $K$-block normalizing
flow that transforms the residual $\boldsymbol{\varepsilon}$ into a standard Gaussian.
Each block consists of a \scl\ layer followed by a \ross\ layer,
with a trailing \scl\ after the final block.
The full architecture is illustrated in Alg.~\ref{alg:torf} and,
in Figure~\ref{fig:arch} (Appendix~\ref{sec:ablation:modelArchi}).

For better regularization and parameter efficiency, the \scl\
parameters $\mathbf{s}$ are shared across all $K+1$ \scl\ layers,
while each block $k$ employs its own spline network
$\mathrm{SplineNet}^{(k)}$.
Since both components (\scl, \ross) are strictly odd functions, their composition
is also odd, analytically enforcing
$\mathbb{E}_{p_{\theta_2}}[\boldsymbol{\varepsilon}\mid\mathbf{X}]=\mathbf{0}$ throughout.
When $K=0$, no \ross\ layer is applied and \textsc{TORF} reduces to
a conditional Gaussian model, exactly recovering the MVE-2S baseline. An equivalence proof is
presented in the Appendix \ref{sec:eqscalinggauss} for reference.
Increasing $K$ progressively introduces non-linear flexibility
through the \ross\ layers.

The training loss is the NLL computed via the change-of-variables
formula with the accumulated log-determinant of the forward map
\begin{equation}
	\resizebox{\columnwidth}{!}{$\displaystyle
	\mathcal{L} =
	-\log \mathcal{N}\!\left(\mathbf{U}^{(K+1)};\,\mathbf{0},\mathbf{I}\right)
	- \sum_{k=0}^{K-1}\!\left(
	\mathrm{ldj}^{(k)}_{\mathrm{s}}
	+ \mathrm{ldj}^{(k)}_{\mathrm{r}}\right)
	- \mathrm{ldj}^{(K)}_{\mathrm{s}}
	$}
	\label{eq:loss}
\end{equation}

\section{Experiments}

Following the protocol and datasets from $K^2$VAE \cite{wu2025kvae}, we conduct experiments on 8 short-term and 9 long-term forecasting datasets.
For short-term tasks, we use ETTh1-S, ETTh2-S, ETTm1-S, ETTm2-S, Solar-S, Electricity-S, Traffic-S and Exchange-S, with context lengths $L$ and horizons $H$ set to 24 (or 30 for Exchange). For long-term forecasting, we utilize the ETT series, Electricity-L, Traffic-L, Exchange-L, Weather-L, and ILI-L. The forecasting horizons $H$ range from $\{96, 192, 336, 720\}$ in general except for ILI-L with $H$ range from $\{24, 36, 48, 60\}$, while the context length is fixed at $L=96$ ($L=36$ for ILI) to ensure a fair comparison across all models. Dataset descriptions are provided in Appendix section \ref{sec:datadesc}.

We compare \model{} against 12 baselines, 4 point forecasters and 8 generative probabilistic models. The point forecasters, each fitted with an MVE head to produce a Gaussian predictive distribution, are FITS \cite{xu2024fits}, PatchTST \cite{nie2023patchtst}, iTransformer \cite{liu2024itransformer}, and Koopa \cite{liu2023koopa}. The generative baselines are TSDiff \cite{kollovieh2023predict}, D3VAE \cite{li2022d3vae}, GRU-NVP, GRU-MAF, and Trans-MAF \cite{rasulmultivariate}, TimeGrad \cite{rasul2021autoregressive}, CSDI \cite{tashiro2021csdi}, and $K^2$VAE \cite{wu2025kvae}. The main table reports the 6 strongest; full results for all 12 are in Appendix~\ref{sec:mainResults_allBaselines}.

\model{} uses SimpleTM for Stage-1. We tune hyperparameters (Appendix \ref{sec:appendix:hpselectImplementation}) with Optuna \cite{akiba2019optuna} over 20 trials per configuration for both main and ablation results, and report mean and standard deviation over 5 Stage-2 runs with frozen Stage-1. Code is provided as supplementary material.

We evaluate probabilistic accuracy with CRPS (Continuous Ranked Probability Score) and point accuracy with NMAE (Normalized Mean Absolute Error). Since CRPS only captures univariate marginals, we also report Energy Score in the ablation study to assess multivariate prediction quality. Metric details are described in Appendix~\ref{sec:metrics}.

\subsection{Results}

\begin{table}[t]
\centering
\footnotesize
\setlength{\tabcolsep}{1mm}
\resizebox{\columnwidth}{!}{\begin{tabular}{lcccccc}
\toprule
Variant & 2-Stage & Odd & Spline & \scl & CNN-ROSS & Mixture \\
\midrule
TORF-1S & \xmark & -- & -- & -- & -- & -- \\
TORF  & \cmark & \cmark & \cmark & \cmark & \cmark & \xmark \\
\quad w/o \scl & \cmark & \cmark & \cmark & \xmark & \cmark & \xmark \\
\quad w/o CNN-ROSS & \cmark & \cmark & \cmark & \cmark & \xmark & \xmark \\
\quad w RealNVP & \cmark & \xmark & \xmark & \cmark & \cmark & \xmark \\
\quad w OddRealNVP & \cmark & \cmark & \xmark & \cmark & \cmark & \xmark \\
\quad w Spline & \cmark & \xmark & \cmark & \cmark & \cmark & \xmark \\
\quad w Mixture & \cmark & \cmark & \cmark & \cmark & \cmark & \cmark \\
\bottomrule
\end{tabular}}
\caption{Ablation variants relative to \model{}.}
\label{tab:ablation_variants}
\end{table}

\begin{table}[t]
\centering
\footnotesize
\begin{tabular}{l c c cc}
\toprule
& ETTm1 & \multicolumn{2}{c}{Exchange} & {Solar} \\
\cmidrule(lr){2-2} \cmidrule(lr){3-4} \cmidrule(lr){5-5}
Method & 336 & 336 & 30 & 24 \\
\midrule
TORF-1S & 0.5122 & 0.1390 & 0.1430 & 0.6078 \\
TORF & \textbf{0.2457} & \textbf{0.0382} & \textbf{0.0077} & \textbf{0.4558} \\
\quad w/o \scl & 0.2508 & 0.0399 & 0.0081 & 0.4670 \\
\quad w/o CNN-ROSS & \underline{\textit{0.2483}} & \underline{\textit{0.0383}} & \underline{\textit{0.0078}} & \textbf{0.4558} \\
\bottomrule
\end{tabular}
\caption{Architectural Analysis in CRPS.}
\label{tab:ablation}
\end{table}

\begin{table*}[t]
\centering
\footnotesize
\setlength{\tabcolsep}{1.5mm}
\begin{tabular}{cl | cccccc | cccccc}
\toprule
 & & \multicolumn{6}{c}{EnergyScore} & \multicolumn{6}{c}{CRPS} \\
\cmidrule(lr){3-8} \cmidrule(lr){9-14}
 & & \multicolumn{2}{c}{ETTm1} & \multicolumn{3}{c}{Exchange} & \multicolumn{1}{c}{Solar} & \multicolumn{2}{c}{ETTm1} & \multicolumn{3}{c}{Exchange} & \multicolumn{1}{c}{Solar} \\
\cmidrule(lr){3-4} \cmidrule(lr){5-7} \cmidrule(lr){8-8} \cmidrule(lr){9-10} \cmidrule(lr){11-13} \cmidrule(lr){14-14}
Type & Methods & 96 & 336 & 96 & 336 & 30 & 24 & 96 & 336 & 96 & 336 & 30 & 24 \\
\midrule
\multirow{2}{*}{\rotatebox{90}{}}
& MVE-1S & 0.6020 & 0.3832 & 0.0064 & \underline{\textit{0.0065}} & \underline{\textit{0.0047}} & 59.2827 & 0.2430 & 0.3301 & 0.0210 & 0.0399 & 0.0077 & 0.4774 \\
& $K^2$VAE & 0.4990 & 0.3079 & 0.0099 & 0.0082 & 0.0066 & 49.0566 & 0.2359 & 0.2698 & 0.0322 & 0.0519 & 0.0106 & \textbf{0.4129} \\
\midrule
\multirow{6}{*}{\rotatebox{90}{2-stage}}
& MVE-2S & 0.4697 & 0.2997 & 0.0063 & \textbf{0.0064} & 0.0048 & 48.6663 & 0.2117 & 0.2491 & 0.0204 & 0.0391 & \textbf{0.0074} & 0.4705 \\
& TORF & 0.4651 & \underline{\textit{0.2933}} & \textbf{0.0061} & \textbf{0.0064} & \underline{\textit{0.0047}} & 48.2329 & 0.2088 & 0.2457 & \textbf{0.0198} & \textbf{0.0382} & 0.0077 & \underline{\textit{0.4558}} \\
& \quad w RealNVP & 0.4674 & 0.2956 & 0.0063 & 0.0066 & \textbf{0.0046} & 48.4335 & 0.2101 & 0.2517 & 0.0210 & 0.0412 & 0.0077 & 0.4774 \\
& \quad w Spline & \underline{\textit{0.4587}} & 0.2938 & 0.0065 & \underline{\textit{0.0065}} & \underline{\textit{0.0047}} & 48.2329 & \underline{\textit{0.2051}} & 0.2452 & 0.0213 & 0.0394 & 0.0078 & \underline{\textit{0.4558}} \\
& \quad w OddRealNVP & 0.4656 & 0.2946 & \underline{\textit{0.0062}} & \textbf{0.0064} & \textbf{0.0046} & \underline{\textit{48.1577}} & 0.2086 & \underline{\textit{0.2450}} & \underline{\textit{0.0203}} & 0.0397 & \textbf{0.0074} & 0.4588 \\
& \quad w Mixture & \textbf{0.4559} & \textbf{0.2930} & \underline{\textit{0.0062}} & \textbf{0.0064} & 0.0048 & \textbf{48.0786} & \textbf{0.2038} & \textbf{0.2436} & 0.0208 & \underline{\textit{0.0390}} & \underline{\textit{0.0076}} & 0.4653 \\
\bottomrule
\end{tabular}
\caption{Energy Score (joint, multivariate density) and CRPS (univariate marginal density) results.}
\label{tab:energyscore_crps}
\end{table*}

\Cref{tab: long-term} and \Cref{tab: short-term} report the full probabilistic forecasting results.
In the long-term setting, \model{}\ wins on CRPS for 7 of 9 datasets and on NMAE for all 9 datasets. The largest gains reach $+35.4\%$ on CRPS (ETTh2-L) and $+25.1\%$ on NMAE (ETTm2-L). In the short-term setting, \model{}\ wins on CRPS for 6 of 8 datasets and on NMAE for 5 of 8 datasets. The largest gains reach $+30.3\%$ on CRPS (ETTm2-S) and $+28.1\%$ on NMAE (ETTm2-S). NMAE gains come from the stronger mean predictor, whereas CRPS gains come from the flexible odd flow conditioned on that mean.

The long-term wins are profound. This is likely because SimpleTM, \model{}'s Stage-1 model, was originally designed for long-term forecasting. We investigate the two settings where \model{} fails to achieve the best results. On Traffic-L, the CRPS gap is close to a tie with a 1.5\% difference from $K^2$VAE. This loss is plausibly within evaluation noise rather than a systematic weakness. The Electricity-L result reported for $K^2$VAE deserves closer scrutiny. $K^2$VAE reports a CRPS of $0.057\pm0.005$ at the 720-step horizon, but both, our rerun of the $K^2$VAE code and a concurrent work, report this value as approximately $0.087$ \citep{wang2026beyond}. This also contradicts $K^2$VAE's own per-horizon trend, where CRPS is lower at 720 than at 192 horizon. Corrected, \model{}'s Electricity-L CRPS of $0.081\pm0.000$ would surpass $K^2$VAE's.

The losses in the short-term setting concentrate on Solar-S and Electricity-S. On Solar-S, \model{}\ trails $K^2$VAE by a reported $24.3\%$ on CRPS and $21.5\%$ on NMAE. On Electricity-S the gaps are smaller, at $3.9\%$ and $6.1\%$. Rerunning $K^2$VAE on Solar-S ourselves gave a CRPS of $0.413\pm0.017$ and an NMAE of $0.531\pm0.018$, both only about $10\%$ worse than \model{}, not $24.3\%$. We attribute this remaining gap to the Stage-1 point forecaster rather than the flow. \model{}'s flow is provably mean-preserving, so its distributional accuracy is bounded by the frozen Stage-1 model, and on both datasets \model{}'s NMAE, inherited from SimpleTM, also lags $K^2$VAE's. Because the flow is agnostic to the choice of backbone, a stronger, dataset-appropriate Stage-1 model, such as a Koopman-style predictor, could close this gap while keeping \model{}'s flow for uncertainty quantification. With SimpleTM as its mean predictor, \model{}\ is state-of-the-art on every dataset except Solar-S and Electricity-S, across both horizons.

\subsection{Ablation Study}
\label{sec:ablation}

To isolate each design choice's contribution, we ablate multiple factors as summarized in \Cref{tab:ablation_variants}, and include two Gaussian MVE baselines on the same SimpleTM backbone so that any gap is attributable to the decomposition, and not backbone strength.

In \Cref{tab:ablation_variants}, \textit{\model{}-1S} removes the two-stage decomposition entirely. Among the two-stage variants, \scl{} is dropped in \textit{w/o \scl}; \textit{w/o CNN-\ross{}} is SplineNet's CNN backbone, replaced by an MLP; Odd restricts the coupling transform to be odd (\textit{\model{}, w OddRealNVP}) versus unrestricted (\textit{w Spline, w RealNVP}); Spline uses a rational-quadratic spline (\textit{\model{}, w Spline}) in place of an affine RealNVP-style map;  and \textit{w Mixture} extends \model{} to a mixture of odd splines for asymmetric distribution modelling with analytical-mean. We also compare with two Gaussian MVE baselines: \textit{MVE-1S}, trained jointly for $\mu$ and $\sigma$ via NLL, and \textit{MVE-2S}, a two-stage variant mirroring \model{} that predicts $\hat{\mathbf{Y}}$ from a pre-trained Stage~1 and $\sigma$ in Stage~2. Appendix~\ref{sec:ablation:oddRealNVP} and Appendix~\ref{sec:ablation:mixOddFlows} detail the construction of \model{} w OddRealNVP and \model{} w Mixture, respectively.

\subsubsection{Architectural Analysis}

\Cref{tab:ablation} shows that \model{}-1S underperforms the two-stage variants. Removing \scl{} in \model{} (w/o \scl{})
substantially weakens \model{} across all horizons and datasets, as \scl{} allows \model{} to learn a simple Gaussian distribution using only the \scl{} layer when needed. Adding an MLP in place of Conv1D (w/o CNN-\ross) does not hurt accuracy and achieves the second-best performance, but comes with a real computational cost (\Cref{tab:resource_trimmed}). It runs out of memory even at horizon 336 on ETTm1, making it impractical for longer horizons or more channels, while the CNN-based SplineNet uses a fraction of the memory.

\begin{table}[t]
\centering
\footnotesize
\begin{tabular}{l cc cc}
\toprule
 & \multicolumn{2}{c}{ETTm1-336} & \multicolumn{2}{c}{Exchange-336} \\
\cmidrule(lr){2-3} \cmidrule(lr){4-5}
Method & Param & Mem & Param & Mem \\
\midrule
$K^2$VAE & 1{,}593 & 0.018 & 1{,}611 & 0.019 \\
TORF & 22{,}183 & 0.512 & 722 & 0.106 \\
\quad w/o CNN-ROSS & OOM & OOM & 66{,}881 & 0.354 \\
\quad w Mixture & 65{,}006 & 1.168 & 2{,}116 & 0.213 \\
\bottomrule
\end{tabular}\caption{Trainable parameters (thousands) and peak memory (GB) across methods and datasets. OOM = out of memory.}
\label{tab:resource_trimmed}
\end{table}

\subsubsection{Probabilistic and Point-Forecast Analysis}

Extending the ablation to a shorter 96-horizon in \Cref{tab:energyscore_crps}, \model{}\ and its variants are state-of-the-art on 11 of 12 Energy-Score/CRPS columns, losing only
Solar-24 CRPS to $K^2$VAE, consistent with the known Stage-1 SimpleTM weakness on Solar.
The two-stage decomposition drives most of this gain. MVE-2S, a plain two-stage Gaussian
head, already beats $K^2$VAE on 11/12 columns on distributional metrics as well as on 3/4 on NMAE (\Cref{tab:nmae}). MVE-2S also beats the jointly-trained MVE-1S across the board on both distributional metrics, since joint NLL training fails to learn the best achievable mean. On Exchange specifically, MVE-1S alone already beats $K^2$VAE, plausibly since the residual is close to Gaussian and $K^2$VAE's extra complexity fails to exploit that simplicity. Beyond marginal accuracy, \model{}\ also beats $K^2$VAE on the multivariate metric. \model\ has better Energy Score in Solar-24 even though it loses on the univariate CRPS. Among design choices, spline outperforms the affine RealNVP transform, although the odd-restricted affine variant alone (w/OddRealNVP) still wins 3/12 columns. TORF-Mixture wins on more columns overall (6/12 vs.\ 4/12) but its edge is mostly insignificant, trades-off its analytical median and costs substantially more compute, so we recommend it only where the infrastructure allows, or if the residual is genuinely asymmetric.
In~\Cref{tab:nmae}, all odd-constrained variants (Odd-TORF) are provably mean-preserving and inherit an identical Stage-1 mean, whereas the non-odd w RealNVP and w Spline can change the mean, visibly degrading NMAE (up to 0.629-0.630 vs.\ 0.583 on Solar-24), exactly the failure mode the odd constraint rules out. \model{}'s Stage-1 model-agnosticism is further tested with iTransformer as a backbone in Appendix~\ref{sec:appendix:stage1Transfer}.

\begin{table}[t]
\centering
\footnotesize
\begin{tabular}{cl cccc}
\toprule
 & & ETTm1 & \multicolumn{2}{c}{Exchange} & Solar \\
\cmidrule(lr){3-3} \cmidrule(lr){4-5} \cmidrule(lr){6-6}
Type & Methods & 336 & 336 & 30 & 24 \\
\midrule
\multirow{2}{*}{\rotatebox{90}{}}
& MVE-1S & 0.349 & 0.056 & 0.011 & 0.599 \\
& $K^2$VAE & 0.339 & 0.055 & 0.011 & \textbf{0.531} \\
\midrule
\multirow{3}{*}{\rotatebox{90}{2-stage}}
& MVE-2S & \textbf{0.314} & \textbf{0.051} & \textbf{0.010} & \underline{\textit{0.583}} \\
& Odd-TORF & \textbf{0.314} & \textbf{0.051} & \textbf{0.010} & \underline{\textit{0.583}} \\
& \quad w RealNVP & \underline{\textit{0.315}} & \underline{\textit{0.052}} & \textbf{0.010} & 0.630 \\
& \quad w Spline & 0.319 & \underline{\textit{0.052}} & \textbf{0.010} & 0.629 \\
\bottomrule
\end{tabular}\caption{Point prediction results (NMAE, lower is better).}
\label{tab:nmae}
\end{table}

\section{Conclusion}
\label{sec:conclusion}

We presented \model{}, a two-stage probabilistic framework for TSF that pairs a frozen, model-agnostic point forecaster with a constrained odd flow whose predictive mean exactly matches the Stage~1 forecast, while a lightweight Conv1D SplineNet keeps it efficient enough to scale to long horizons. Across 17 real-world tasks, \model{}\ achieves state-of-the-art results, with gains of up to $+35.4\%$ CRPS and $+28.1\%$ NMAE over the strongest baseline.
\paragraph{Limitations \& Future Work.} \model{} assumes independent, symmetric residuals, which is sufficient here, but limiting when residuals are asymmetric or cross-dimensionally dependent. The Mixture extension relaxes symmetry at a higher parameter cost for only marginal gains, and joint time/channel dependence remains unaddressed. Future work includes component-dependent mixture weights, coupling- or MAF-based components, and mean-preserving diffusion or flow-matching variants.

\bibliography{aaai2027}

\clearpage

\appendix

\section{Additional Related Works}
\label{sec:appendix:related}

Deterministic forecasting has progressed from classical statistical
models ARIMA~\citep{BoxJenkins1994} to modern deep architectures like Transformers.
These include Informer~\citep{zhou2021informer},
Autoformer~\citep{wu2021autoformer}, FEDformer~\citep{zhou2022fedformer},
Yformer~\citep{madhusudhanan2022u}, and Pyraformer
\citep{liu2022pyraformer}. iTransformer~\citep{liu2024itransformer} and
PatchTST~\citep{nie2023patchtst} remain among the strongest performers.
iTransformer attends across variables. PatchTST attends over
channel-independent patches. Recently, SimpleTM~\citep{chen2025simpletm}
achieved state-of-the-art point accuracy. SimpleTM is a lightweight
mixer-style predictor. It combines wavelet preprocessing with a
geometric-product attention variant. Owing to its strong performance,
we adopt SimpleTM as the Stage~1 model in our experiments. \model{} is
agnostic to this choice.

\paragraph{Probabilistic Time Series Forecasting}
Probabilistic forecasting aims to recover the full predictive
distribution rather than a single point estimate. Autoregressive
models such as DeepAR~\citep{salinas2020deepar} learn Gaussian
parameters via an RNN. Diffusion models treat forecasting as an
iterative denoising process. These include TimeGrad
\citep{rasul2021autoregressive}, CSDI~\citep{tashiro2021csdi}, and
TSDiff~\citep{kollovieh2023predict}. Conditional normalizing flows
\citep{rasulmultivariate,kollovieh2025flow} instead model complex
joint densities without a fixed parametric form. $K^2$VAE
\citep{wu2025kvae} is a VAE based method developed for probabilistic
LTSF. It combines a Koopman-linearized latent dynamical system with a
KalmanNet-based correction. This correction fixes error accumulation
over long horizons. Despite this diversity, all of the above methods
either commit to a fixed distributional family or fall back on Monte
Carlo sampling to recover a predictive mean. \model{} instead
decouples mean estimation from density estimation entirely. It obtains
both an exact analytical mean and a flexible, sampling-free residual
density. \Cref{tab:model_comparison} compares popular density-estimation methods and highlights
the need for a flexible density-modeling approach with an analytical mean.

\paragraph{Pitfalls of NLL Training and Two-Stage Fixes}
\citet{Nix1994EstimatingTM} introduced the original Mean-Variance
Estimation (MVE) formulation. It trains a single network to jointly
predict a distribution's mean and variance under NLL. This joint
training is known to trigger a mean-variance conflict.
\citet{Seitzer2022PitfallsOfUncertainty} show that the NLL gradient is
scaled by the inverse predicted variance. As a result, high-uncertainty
regions contribute less to mean learning (\Cref{sec:prelim:pitfalls}).
Several works address this problem through explicit decoupling. The
most direct approach is \citet{StirnWSPSK23}, who formalize
\emph{faithfulness}. They freeze an independently trained mean network
before fitting a residual density. This guarantees mean accuracy at
least as good as the standalone predictor.
\model{} builds on this decoupling principle. We replace the frozen,
Gaussian-only residual head with a flexible, sampling-free odd
normalizing flow. This extends the faithfulness guarantee to arbitrary
symmetric residual densities in the high-dimensional, long-horizon
time series setting.

\paragraph{Odd Transformations for Analytically Tractable Moments}
Normalizing flows have been applied to flexible probabilistic
prediction in irregularly-sampled and tabular domains
\citep{yalavarthi2024marginalization,Madhusudhanan2025.TabResFlow}. They
remain comparatively underexplored for LTSF. One reason is that their
unconstrained transformations make the predictive mean accessible only
through Monte Carlo sampling. Our restriction on the flow architecture
is directly inspired by \citet{Kobayashi2023DesignOR}. In a
reinforcement learning context, they design a Restricted Normalizing
Flow (RNF) for stochastic robot control policies. They constrain the
invertible transformation to be an odd function of a symmetric base
variable. This makes the policy's mean action analytically computable.
\model{} repurposes this odd-function
restriction for the probabilistic LTSF task. By construction, the
residual flow's exact, sampling-free mean and median are zero. This
preserves the Stage~1 point prediction exactly through Stage~2 density
estimation.

\section{Theory}

\begin{lemma}[Analytical mean and median of a single odd flow]
	\label{lem:odd-flow}
	Let $\mathbf{u} \sim p_\mathbf{u}$ be a base random variable whose
	density is symmetric about the origin, i.e.\
	$p_\mathbf{u}(\mathbf{u}) = p_\mathbf{u}(-\mathbf{u})$ for all
	$\mathbf{u} \in \mathbb{R}^K$, and let
	$f : \mathbb{R}^K \to \mathbb{R}^K$ be an odd, invertible
	transformation, $f(-\mathbf{u}) = -f(\mathbf{u})$
	(Eq.~\eqref{eq:odd}). Let $\mathbf{v} = f(\mathbf{u})$ with induced
	density $p_\mathbf{v}$. Then, provided
	$\mathbb{E}[\,\lVert \mathbf{v} \rVert\,] < \infty$,
	\begin{equation*}
		\mathbb{E}[\mathbf{v}] = \mathbf{0},
		\; \text{and} \;\;
		\mathrm{median}(v_j) = 0 \ \text{ for every coordinate } j.
	\end{equation*}
\end{lemma}

\begin{proof}
	\textbf{Symmetry of $p_\mathbf{v}$.}
	Since $f$ is an odd bijection, $f^{-1}$ is also odd:
	$f^{-1}(-\mathbf{v}) = -f^{-1}(\mathbf{v})$.
	Its Jacobian satisfies
	$\big|\det \partial_\mathbf{v} f^{-1}(-\mathbf{v})\big|
	= \big|\det \partial_\mathbf{v} f^{-1}(\mathbf{v})\big|$,
	because differentiating the odd map $f^{-1}$ yields an even Jacobian.
	By the change-of-variables formula (Eq.~\eqref{eq:cov}),
	\begin{align*}
		p_\mathbf{v}(-\mathbf{v})
		&= p_\mathbf{u}\!\big(f^{-1}(-\mathbf{v})\big)\,
		\big|\det \partial_\mathbf{v} f^{-1}(-\mathbf{v})\big| \\
		&= p_\mathbf{u}\!\big(-f^{-1}(\mathbf{v})\big)\,
		\big|\det \partial_\mathbf{v} f^{-1}(\mathbf{v})\big| \\
		&= p_\mathbf{u}\!\big(f^{-1}(\mathbf{v})\big)\,
		\big|\det \partial_\mathbf{v} f^{-1}(\mathbf{v})\big|
		= p_\mathbf{v}(\mathbf{v}),
	\end{align*}
	using $p_\mathbf{u}(-\,\cdot) = p_\mathbf{u}(\cdot)$. Thus
	$p_\mathbf{v}$ is symmetric about the origin.

	\paragraph{Zero mean.} With $\mathbf{w} := -\mathbf{v}$,
	\begin{align*}
		\mathbb{E}[\mathbf{v}]
		&	= \int \mathbf{v}\, p_\mathbf{v}(\mathbf{v})\, d\mathbf{v}
		= \int (-\mathbf{w})\, p_\mathbf{v}(-\mathbf{w})\, d\mathbf{w} \\
		&	= -\int \mathbf{w}\, p_\mathbf{v}(\mathbf{w})\, d\mathbf{w}
		= -\mathbb{E}[\mathbf{v}],
	\end{align*}
	so $\mathbb{E}[\mathbf{v}] = \mathbf{0}$; absolute integrability
	ensures the integral is well defined.

	\paragraph{Zero median.} For any coordinate $j$, symmetry
	$p_\mathbf{v}(\mathbf{v}) = p_\mathbf{v}(-\mathbf{v})$ implies
	$\Pr(v_j \le 0) = \Pr(v_j \ge 0)$. Since these probabilities sum to
	$1$, each is $\tfrac{1}{2}$, so $0$ is the
	median of $v_j$.
\end{proof}

\begin{remark}
	In \textsc{TORF}, both LSL (Eq.~\eqref{eq:scl}, scaling) and ROSS
	(Eq.~\eqref{eq:ross-fwd}, spline) are odd, so their composition
	is odd; applied to the base $\mathcal{N}(\mathbf{0}, \mathbf{I})$
	(symmetric about $\mathbf{0}$) in the generative direction, this gives
	$\mathbb{E}_{p_{\theta_2}}[\boldsymbol{\varepsilon} \mid \mathbf{X}]
	= \mathbf{0}$ and coordinate-wise zero median, so
	$f_{\theta_1}(\mathbf{X})$ is preserved as both the analytical mean
	and median of $\hat{p}(\mathbf{Y} \mid \mathbf{X})$.
\end{remark}

\begin{lemma}[Analytical mean of \textsc{TORF}-Mixture]
	\label{lem:mixture-mean}
	Let each component $i = 1,\dots,M$ have conditional mean
	$\mathbb{E}[\varepsilon^{(i)}_{h,c} \mid \mathbf{F}] = \beta^{(i)}_{h,c}$,
	with mixture weights $\pi^{(i)}_{h,c} \ge 0$,
	$\sum_{i=1}^{M} \pi^{(i)}_{h,c} = 1$. If
	\begin{equation}
		\sum_{i=1}^{M} \pi^{(i)}_{h,c}\, \beta^{(i)}_{h,c} = 0 ,
		\label{eq:centering}
	\end{equation}
	then $\mathbb{E}[\varepsilon_{h,c} \mid \mathbf{F}] = 0$.
\end{lemma}

\begin{proof}
	Each component has mean $\beta^{(i)}_{h,c}$ (odd transform of a
	symmetric base, shifted by $\beta^{(i)}_{h,c}$;
	Lemma~\ref{lem:odd-flow}). By the Law of total expectation rule,
	\begin{equation}
		\mathbb{E}[\varepsilon_{h,c} \mid \mathbf{F}]
		= \sum_{i=1}^{M} \pi^{(i)}_{h,c}\, \beta^{(i)}_{h,c}
		= 0 ,
	\end{equation}
\end{proof}

\begin{remark}[Mean, not median]
	Because the constraint fixes only the weighted average of the
	$\beta^{(i)}_{h,c}$ rather than each individually, the components are
	generally centered at distinct nonzero locations, so the mixture is
	asymmetric. Consequently the coordinate-wise median of
	$\varepsilon_{h,c}$ need not be $0$: unlike the single-flow case
	(Lemma~\ref{lem:odd-flow}), \textsc{TORF}-Mixture guarantees mean
	preservation but not median preservation.
\end{remark}
\section{Synthetic Data Experiment}
\label{appendix:syntheticDataExp}

To evaluate the robustness of our model under non-Gaussian, heteroscedastic conditions, we define a synthetic data generating
process where the observed variable $y_i$ is a function of a latent signal $\mu(x_i)$ and input-dependent heavy-tailed noise $\epsilon(x_i)$:

\begin{equation}
    y_i = \mu(x_i) + \epsilon(x_i), \quad x_i \in [0, 1]
\end{equation}

The deterministic component $\mu(x)$ is a high-frequency sinusoidal signal:
\begin{equation}
    \mu(x) = \sin(2\pi f x), \quad \text{with } f = 10.0
\end{equation}

The stochastic component $\epsilon(x)$ follows a Student's t-distribution characterized by input-dependent scale $\gamma(x)$ and degrees of freedom $\nu(x)$:
\begin{equation}
    \epsilon(x) \sim \text{Student-t}(\nu(x), 0, \gamma(x))
\end{equation}

To simulate realistic "noise traps" where the uncertainty structure changes across the input space, we define $\gamma(x)$ and $\nu(x)$ using Gaussian kernels centered at $x=0.5$:

\begin{itemize}
    \item \textbf{Scale Parameter:} $\gamma(x) = 0.05 + 0.3 \exp\left( -\frac{(x - 0.5)^2}{0.02} \right)$
    \item \textbf{Degrees of Freedom:} $\nu(x) = 3.0 + 7.0 \exp\left( -\frac{(x - 0.5)^2}{0.02} \right)$
\end{itemize}

Under this formulation, the noise exhibits heavier tails (lower $\nu$) away from the center, while the magnitude of the residuals ($\gamma$) peaks at the center. Importantly, since the Student's t-distribution is symmetric about zero, the conditional expectation remains $\mathbb{E}[y|x] = \mu(x)$, allowing us to strictly evaluate the decoupling performance of the \model{} framework.

\section{\model{} Architecture}
\label{sec:ablation:modelArchi}

\begin{figure*}[t]
	\centering
	\resizebox{\textwidth}{!}{\begin{tikzpicture}[
		font=\huge,
		>=Stealth,
		node distance=7mm and 9mm,
		data/.style={draw, thick, rounded corners=2pt, fill=white,
			minimum height=7mm, minimum width=11mm, align=center, inner sep=2pt},
		frozen/.style={draw, thick, rounded corners=2pt, fill=gray!18,
			minimum height=9mm, text width=35mm, align=center},
		trainable/.style={draw=triblue, very thick, rounded corners=2pt, fill=white,
			minimum height=9mm, text width=35mm, align=center},
		scalenetstyle/.style={draw, thick, rounded corners=2pt, fill=chanintra,
			minimum height=9mm, text width=35mm, align=center},
		splinenetstyle/.style={draw, thick, rounded corners=2pt, fill=timeintra,
			minimum height=13mm, text width=35mm, align=center},
		flowop/.style={draw, thick, rounded corners=2pt, fill=actcolor,
			minimum height=9mm, minimum width=14mm, align=center},
		gate/.style={draw, circle, fill=white, inner sep=1.5pt, font=\small},
		loopbox/.style={draw, thick, dashed, rounded corners=5pt, inner sep=9pt},
		insetbox/.style={draw, thin, dotted, rounded corners=3pt, inner sep=6pt, fill=gray!3},
		al/.style={font=\scriptsize, midway, above, align=center},
		bl/.style={font=\scriptsize, midway, below, align=center},
		shared/.style={->, thick, dashed, tridark},
		]

	\node[data] (X) at (0,0) {$\mathbf{X}$};

	\node[frozen, right=14mm of X] (enc) {Encoder\\$\mathrm{enc}(\mathbf{X};\theta_1)$};
	\draw[->, thick] (X) -- (enc);
    \node[anchor=south east, inner sep=2pt] at (enc.south east) {
        \includegraphics[width=15pt]{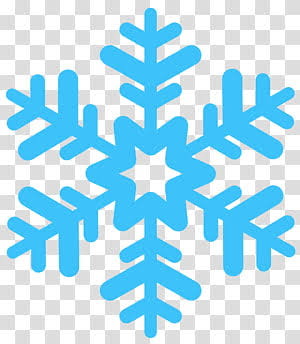}
    };

	\node[frozen, below right=9mm and 16mm of enc, text width=28mm, minimum height=9mm] (headfrozen) {Linear \\ Head};
	\node[anchor=south east, inner sep=2pt] at (headfrozen.south east) {
        \includegraphics[width=15pt]{fig/snow.png}
    };

	\node[trainable, above right=9mm and 16mm of enc, text width=28mm, minimum height=9mm] (headtrain) {Linear \\ Head$'$};
	\draw[->, thick] (enc.east) -- ++(4mm,0) |- (headfrozen.west);
	\draw[->, thick] (enc.east) -- ++(4mm,0) |- (headtrain.west);

	\node[data, below=12mm of headfrozen] (Yhat) {$\hat{\mathbf{Y}}$};
	\draw[->, thick] (headfrozen) -- (Yhat);

	\node[data, below=9mm of Yhat] (Y) {$\mathbf{Y}$};
	\node[gate, right=9mm of Yhat, yshift=-8mm] (minus) {$-$};
	\draw[->, thick] (Yhat) -- (minus);
	\draw[->, thick] (Y) -- (minus);

	\node[data, right=9mm of minus] (U0) {$\mathbf{U}^{(0)}$};
	\draw[->, thick] (minus) -- (U0);

	\node[data, right=13.5mm of headtrain] (F) {$\mathbf{F}$};
	\draw[->, thick] (headtrain) -- (F);

	\node[scalenetstyle, right=25mm of headfrozen] (lsl) {\scl\\ ($\cdot\,;\mathbf{s}$)};
	\draw[->, thick] (U0) |- ([yshift=-2mm]lsl.west);

	\node[data, right=9mm of lsl] (Uk) {$\mathbf{U}^{(k)}$};
	\draw[->, thick] (lsl) -- (Uk);

	\node[flowop, right=13mm of Uk] (rossk) {\ross};
	\draw[->, thick] (Uk) -- (rossk);
	\node[data, right=9mm of rossk] (Uhalf) {$\mathbf{U}^{(k+\frac{1}{2})}$};
	\draw[->, thick] (rossk) -- (Uhalf);

	\node[splinenetstyle, above=of rossk, right=62mm of F] (splinenetk) {
		\begin{tikzpicture}[scale=0.32, baseline]
			\foreach \i in {0,1,2}{
				\draw[thick, fill=white] (0,\i*0.55) rectangle (2.1,\i*0.55+0.4);
			}
			\draw[thick, tridark] (0.35,-0.15) rectangle (1.0,1.75);
		\end{tikzpicture}\\[1pt]
		SplineNet$^{(k)}$
	};
	\node[data, below = 13mm of splinenetk] (Phi) {$\boldsymbol{\Phi}^{(k)}$};
	\draw[->, thick] (splinenetk) -- (Phi);
	\draw[->, thick] (Phi) -- (rossk);

	\node[splinenetstyle, below = 8mm of F] (scalenet) {ScaleNet $(\cdot\,;\boldsymbol{\omega})$};

	\node[data, below = 8mm of scalenet] (s) {$\mathbf{s}$};

	\draw[->, thick] (F.south) -- ([yshift=2mm]scalenet.north);
	\draw[->, thick] (scalenet.south) -- ([yshift=2mm]s.north);
	\draw[->, thick] (s.south) |- ([yshift=2mm]lsl.west);

	\draw[->, thick] (F.east)  -- (splinenetk.west);

	\node[scalenetstyle, right=10mm of Uhalf] (scalenettrail) {\scl\\ ($\cdot\,;\mathbf{s}$)};
	\draw[->, thick] (Uhalf.east) -- (scalenettrail);

	\begin{scope}[on background layer]
		\node[loopbox, minimum height=90mm, fit=(scalenettrail)(Uk)(rossk)(Uhalf)(splinenetk)(Phi),
			label={[anchor=north east]north east:$\times K$}] (loop) {};
	\end{scope}

	\node[data, right=10mm of scalenettrail] (UK1) {$\mathbf{U}^{(K+1)}$};
	\draw[->, thick] (scalenettrail) -- (UK1);

	\node[right=10mm of UK1, align=center] (gauss) {
		\begin{tikzpicture}[baseline]
			\draw[->] (-0.55,0) -- (0.55,0);
			\draw[thick, tridark] plot[domain=-0.5:0.5, samples=20] (\x, {0.4*exp(-8*\x*\x)});
		\end{tikzpicture}\\[-1mm]
		{\scriptsize $\sim\mathcal{N}(\mathbf{0},\mathbf{I})$}
	};
	\draw[->, thick] (UK1) -- (gauss);

	\node[below=20mm of UK1, font=\large, align=center, text width=60mm]
		(loss) {$\mathcal{L}=-\log\mathcal{N}(\mathbf{U}^{(K+1)};\mathbf{0},\mathbf{I})-\mathrm{ldj}$};
	\draw[->, thick, dotted] (UK1.south) -- (loss.north);

	\draw[->, thick, red, dashed] ([yshift=-2mm]gauss.west) -- ([yshift=-2mm]UK1.east);
	\draw[->, thick, red, dashed] ([yshift=-2mm]UK1.west) -- ([yshift=-2mm]scalenettrail.east);
	\draw[->, thick, red, dashed] ([yshift=-2mm]scalenettrail.west) -- ([yshift=-2mm]Uhalf.east);
	\draw[->, thick, red, dashed] ([yshift=-2mm]Uhalf.west) -- ([yshift=-2mm]rossk.east);
	\draw[->, thick, red, dashed] ([yshift=-2mm]rossk.west) -- ([yshift=-2mm]Uk.east);
	\draw[->, thick, red, dashed] ([yshift=-2mm]Uk.west) -- ([yshift=-2mm]lsl.east);
	\draw[->, thick, red, dashed] ([yshift=-4mm]lsl.west) -| ([xshift=2mm]U0.north);

	\node[right=16mm of U0, align=center, yshift=-8mm] (multimodal0) {
		\begin{tikzpicture}[baseline, scale=1.6]
			\draw[gray] (-0.8,0) -- (0.8,0);
			\draw[->] (-0.8,0) -- (0.9,0);
			\draw[thick, red] plot[domain=-0.8:0.8, samples=60]
				(\x, {0.35*exp(-40*(\x+0.22)*(\x+0.22)) + 0.35*exp(-40*(\x-0.22)*(\x-0.22))});
		\end{tikzpicture}\\[-1mm]
		{\small $\sim p(\mathbf{0})$}
	};
	\draw[->, thick, red, dashed] (U0.south) |- (multimodal0.west);

	\node[right=14mm of multimodal0, align=center] (multimodalY) {
		\begin{tikzpicture}[baseline, scale=1.6]
			\draw[gray] (-0.8,0) -- (0.8,0);
			\draw[->] (-0.8,0) -- (0.9,0);
			\draw[thick, red] plot[domain=-0.8:0.8, samples=60]
				(\x, {0.35*exp(-40*(\x+0.22)*(\x+0.22)) + 0.35*exp(-40*(\x-0.22)*(\x-0.22))});
		\end{tikzpicture}\\[-1mm]
		{\small $\sim p(\hat{\mathbf{Y}})$}
	};
	\draw[->, thick, red, dashed] (multimodal0.east) -- (multimodalY.west);
	\draw[->, thick, red, dashed] (Yhat.east) -| (multimodalY.north);

	\node[anchor=north west, align=left] (legend) at ([yshift=-10mm]X.south west) {
		\begin{tikzpicture}[baseline, every node/.style={anchor=west, font=\large}]
			\node[anchor=south east, inner sep=1pt] at (4mm,-3mm) {\includegraphics[width=15pt]{fig/snow.png}};
			\node at (6mm,0mm) {Frozen: pretrained};

			\draw[thick, rounded corners=1pt, fill=gray!18] (0mm,-8mm) rectangle ++(4mm,3mm);
			\node at (6mm,-6.5mm) {Stage-1 mean model};

			\draw[triblue, very thick, rounded corners=1pt, fill=white] (0mm,-15mm) rectangle ++(4mm,3mm);
			\node at (6mm,-13.5mm) {Duplicated trainable head};

			\draw[thick, rounded corners=1pt, fill=chanintra] (0mm,-22mm) rectangle ++(4mm,3mm);
			\node at (6mm,-20.5mm) {ScaleNet: shared weights $\boldsymbol{\omega}$};

			\draw[->, thick] (0mm,-36mm) -- (4mm,-36mm);
			\node at (6mm,-36mm) {Forward / training pass};

			\draw[->, thick, red, dashed] (0mm,-42mm) -- (4mm,-42mm);
			\node at (6mm,-42mm) {Inverse / generative pass};
		\end{tikzpicture}
	};

	\end{tikzpicture}}
	\caption{\textsc{\model{}} architecture. Forward pass (black): the frozen encoder and head give $\hat{\mathbf{Y}}$; a trainable duplicate head gives context $\mathbf{F}$, which conditions the shared \textsc{ScaleNet} and per-block \textsc{SplineNet}$^{(k)}$. After $K$ many layers \model{} transforms the residual distribution into a standard normal distribution $\mathbf{U}^{(K+1)}\sim\mathcal{N}(\mathbf{0},\mathbf{I})$. Both \textsc{\scl} and \textsc{\ross} being odd preserves $\hat{\mathbf{Y}}$ as the analytical mean. Inverse pass (red, dashed): sampling back through the flow yields the multimodal residual $p(\mathbf{0})$, which shifts to $p(\hat{\mathbf{Y}})$ once $\hat{\mathbf{Y}}$ is added.}
	\label{fig:arch}
\end{figure*}

\Cref{fig:arch} shows the full \model{} computation graph, forward (training,
black) and inverse (sampling, red dashed) passes overlaid on the same nodes.

\paragraph{Training: forward pass.}
The frozen Stage-1 encoder $\mathrm{enc}(\mathbf{X};\theta_1)$
feeds two heads: the frozen Linear Head, which reproduces the Stage-1 point
forecast $\hat{\mathbf{Y}}$, and the duplicated, trainable Linear Head$'$,
which produces the context $\mathbf{F}$ that conditions every
downstream network. The residual $\mathbf{U}^{(0)} = \mathbf{Y} -
\hat{\mathbf{Y}}$ then enters the $K$-block loop (dashed box, $\times K$):
$\mathbf{F}$ is mapped once, by the shared ScaleNet, into the scale
$\mathbf{s}$ (\Cref{eq:scale}) reused by every \scl\ in the loop, and once per
block, by SplineNet$^{(k)}$, into the spline parameters $\boldsymbol{\Phi}^{(k)}$
(\Cref{eq:conv}) that condition that block's \ross. Each block applies
\scl$(\cdot;\mathbf{s})$ to obtain $\mathbf{U}^{(k+\frac12)}$ and then
\ross$(\cdot;\boldsymbol{\Phi}^{(k)})$ to obtain $\mathbf{U}^{(k+1)}$, and a
trailing \scl\ after block $K$ produces $\mathbf{U}^{(K+1)}$, with the
log-determinant $\mathrm{ldj}$ accumulated at every \scl\ and \ross\ along the
way (\Cref{alg:torf}). Training pushes $\mathbf{U}^{(K+1)}$ toward
$\mathcal{N}(\mathbf{0},\mathbf{I})$ through the loss $\mathcal{L} =
-\log\mathcal{N}(\mathbf{U}^{(K+1)};\mathbf{0},\mathbf{I}) - \mathrm{ldj}$
(\Cref{eq:loss}); since gradients never reach the frozen path
(snowflake nodes), $\hat{\mathbf{Y}}$ and the Stage-1 mean forecast it encodes
are entirely untouched by this Stage-2 objective.

\paragraph{Inference: inverse pass.}
At inference, the same flow runs in reverse (red, dashed arrows) to produce
the predictive distribution $\hat p(\mathbf{Y}\mid\mathbf{X})$
 by sampling rather than by density evaluation, so no query
point $\mathbf{Y}$ is needed. A draw $\mathbf{u}\sim\mathcal{N}(\mathbf{0},\mathbf{I})$
is pushed back through the trailing \scl, then through each block's \ross\ and
\scl\ applied in their inverse direction, in reverse block order, re-using
the same $\mathbf{s}$ and $\boldsymbol{\Phi}^{(k)}$ already computed from
$\mathbf{F}$, until it reaches $\mathbf{U}^{(0)}$. Because every \scl\ and
\ross\ is odd and $\mathcal{N}(\mathbf{0},\mathbf{I})$ is
symmetric, $\mathbf{U}^{(0)}$ is a draw from a symmetric, possibly
multimodal residual density, centered exactly at the origin,
rather than from any fixed parametric family. Adding
back the frozen Stage-1 forecast, $\hat{\mathbf{Y}} + \mathbf{U}^{(0)}$,
translates this density into the predictive distribution
$\hat p(\mathbf{Y}\mid\mathbf{X})$, centered at $\hat{\mathbf{Y}}$ by
construction: samples from \model{}
still average back to the exact Stage-1 mean (\Cref{lem:odd-flow}).

\section{Data Description}
\label{sec:datadesc}

\model{} is evaluated on the dataset collection\footnote{Datasets available at \url{https://drive.google.com/drive/folders/}\\\url{1l0c4H57xYKKQQ5Tm7kd4C8M2nCepky-y}} and split, assembled by
$K^2$VAE~\citep{wu2025kvae}, which in turn builds on ProbTS~\citep{Zhang2023ProbTSBP}, a general-purpose benchmark for probabilistic
forecasters spanning short and long horizons.
The short-horizon split contains eight series ETTh1-S, ETTh2-S, ETTm1-S,
ETTm2-S, Solar-S, Electricity-S, Traffic-S, and
Exchange-S for which the context window and the forecast horizon coincide:
$L=H=24$ for every series except Exchange-S,
where $L=H=30$. The long-horizon split spans nine series, the four ETT variants, Electricity-L, Traffic-L, Exchange-L,
Weather-L, and ILI-L evaluated at $H\in\{96,192,336,720\}$, with ILI-L instead swept over $H\in\{24,36,48,60\}$; the
context window is held fixed within each split at $L=96$ ($L=36$ for ILI-L) so that every
model sees the same amount of history.

A dataset and its counterpart from the other split can share a name yet be different data: Electricity-S and Electricity-L,
for instance, are separate extracts with their own channel counts and lengths rather than the same recording sliced at two
horizons. \Cref{tab:dataset_stats} lists, for every series, the number of channels, the value range, the sampling
frequency, the total number of recorded timesteps, and a brief description.

\section{Evaluation Metrics}
\label{sec:metrics}

Following the ProbTS protocol~\citep{Zhang2023ProbTSBP}, we report point-forecast
accuracy with NMAE and marginal probabilistic accuracy with CRPS. We
additionally report the Energy Score, which scores the predictive distribution
jointly over the full $H{\times}C$ target grid, to check whether \model{}'s
multivariate structure is well calibrated, unlike NMAE and CRPS, both of which
score one $(h,c)$ entry at a time.

\begin{itemize}
    \item \textbf{Normalized Mean Absolute Error (NMAE)} measures the accuracy
    of the Stage-1 point forecast $\hat{\mathbf{Y}} = f_{\theta_1}(\mathbf{X})$
    against the ground truth $\mathbf{Y}$, summed over the $H$ horizon steps
    and $C$ channels and normalized to be comparable across datasets of
    different scale:
    \begin{equation}
        \mathrm{NMAE} = \frac{\sum_{c=1}^{C}\sum_{h=1}^{H} \big|Y_{h,c} - \hat{Y}_{h,c}\big|}
        {\sum_{c=1}^{C}\sum_{h=1}^{H} \big|Y_{h,c}\big|}.
        \label{eq:nmae}
    \end{equation}

    \item \textbf{Continuous Ranked Probability Score (CRPS)} measures how well
    a single scalar predictive distribution matches a single scalar
    observation $y$, by comparing its predictive CDF $F$ against $y$:
    \begin{equation}
        \mathrm{CRPS}(F, y) = \int_{-\infty}^{\infty} \big(F(z) - \mathbb{I}\{y \le z\}\big)^2\, dz.
        \label{eq:crps}
    \end{equation}
    $F$ here is this scalar predictive CDF, unrelated to the context
    embedding $\mathbf{F}$ of \Cref{eq:ctx}. We evaluate \eqref{eq:crps} at
    every entry $Y_{h,c}$ of the target, using $F$'s own marginal of
    $\hat p(\mathbf{Y}\mid\mathbf{X})$ at that entry, and
    report the mean over $(h,c)$.

    \item \textbf{Energy Score (ES)} extends CRPS to the full joint predictive
    distribution~\citep{Szekely2013energy}: rather than scoring each $(h,c)$
    entry against its own marginal, it flattens the entire $H{\times}C$ target
    grid into one $HC$-dimensional vector $\mathbf{Y}$ and scores it jointly
    against samples drawn from $\hat p(\mathbf{Y}\mid\mathbf{X})$, rewarding a
    predictive distribution that captures dependence across time steps and
    channels together, not just per-entry accuracy:
    \begin{equation}
        \mathrm{ES}\big(\hat p(\mathbf{Y}\mid\mathbf{X}), \mathbf{Y}\big)
        = \mathbb{E}\big\|\mathbf{Y}^{(1)} - \mathbf{Y}\big\|_2
        - \tfrac{1}{2}\, \mathbb{E}\big\|\mathbf{Y}^{(1)} - \mathbf{Y}^{(2)}\big\|_2,
        \label{eq:energyscore}
    \end{equation}
    where $\|\cdot\|_2$ is the Euclidean norm over all $HC$ flattened entries
    and $\mathbf{Y}^{(1)}, \mathbf{Y}^{(2)} \overset{\text{i.i.d.}}{\sim}
    \hat p(\mathbf{Y}\mid\mathbf{X})$ are two independent samples from
    \model{}'s predictive distribution.
\end{itemize}

Both CRPS and the Energy Score are proper scoring rules: each is minimized
exactly when the predictive distribution matches the true data-generating
distribution, so a lower score reflects better calibration, not merely a
narrower interval. Neither has a closed form for \model{}'s predictive
distribution, so we estimate both from samples. Drawing $M$ i.i.d.\ samples
$\mathbf{Y}^{(1)},\dots,\mathbf{Y}^{(M)} \sim \hat p(\mathbf{Y}\mid\mathbf{X})$
per test window via the inverse flow, we
evaluate \eqref{eq:crps} by replacing $F$ with the empirical CDF
$\hat F(z) = \tfrac{1}{M}\sum_{i=1}^{M} \mathbb{I}\{Y^{(i)} \le z\}$ at each
$(h,c)$ entry, and \eqref{eq:energyscore} with the standard empirical
estimator over the same $M$ samples,
\begin{equation}
    \widehat{\mathrm{ES}} = \frac{1}{M}\sum_{i=1}^{M} \big\|\mathbf{Y}^{(i)} - \mathbf{Y}\big\|_2
    - \frac{1}{2M^2}\sum_{i=1}^{M}\sum_{j=1}^{M} \big\|\mathbf{Y}^{(i)} - \mathbf{Y}^{(j)}\big\|_2,
    \label{eq:energyscore-empirical}
\end{equation}
computed per test window and averaged over the evaluation set. We use
$M=100$ samples throughout. Because the raw \eqref{eq:energyscore-empirical}
grows with the dimensionality $HC$ of the flattened target, the ES values we
report (\Cref{tab:energyscore_crps}) additionally divide this per-window
average by the horizon length $H$. This does not fully normalize the score,
since $\widehat{\mathrm{ES}}$ scales roughly with $\sqrt{HC}$ rather than
$H$, but it removes the dominant horizon-driven growth and keeps values
across the different horizons compared in that table on a roughly similar
scale.

\section{Conditional Scaling Achieves an Exact Gaussian Transformation}
\label{sec:eqscalinggauss}

This appendix proves the $K{=}0$ special case of the $K$-block construction
(\Cref{sec:method:scl}, Algorithm~\ref{alg:torf}): with no \ross\ blocks, the
only transform applied is the trailing \scl, mapping the residual
$\mathbf{U}^{(0)} = \boldsymbol{\varepsilon}$ directly to $\mathbf{U}^{(1)}$.
We show this reduces \model{} exactly to a heteroscedastic Gaussian model,
recovering the MVE-2S baseline.

Recall from \Cref{eq:scl} that \scl\ maps $u_{h,c} \mapsto v_{h,c} = u_{h,c}\,s_{h,c}$
(forward) and $v_{h,c} \mapsto u_{h,c} = v_{h,c}/s_{h,c}$ (inverse), with the
scale $s_{h,c} > 0$ predicted by ScaleNet from the context $\mathbf{F}_{:,c}$
(\Cref{eq:scale}). With $K{=}0$, $u_{h,c} = \varepsilon_{h,c}$ is the raw
residual and $v_{h,c} = U^{(1)}_{h,c}$ is assumed standard normal, so the
induced distribution of the residual is $\varepsilon_{h,c} = v_{h,c}/s_{h,c}$
with $v_{h,c}\sim\mathcal{N}(0,1)$, i.e.\
$\varepsilon_{h,c}\sim\mathcal{N}(0,\sigma_{h,c}^2)$ with
$\sigma_{h,c} := 1/s_{h,c}$ the implied per-entry standard deviation. The
forward log-determinant is exactly \Cref{eq:scl-op}, restated here for
reference,
\begin{equation}
    \mathrm{ldj}_{\text{scale}} = \sum_{h,c} \log s_{h,c},
    \label{eq:scale-ldj}
\end{equation}
since the map is diagonal and factorizes over all $(h,c) \in [H]\times[C]$ entries.

\paragraph{Equivalence to a direct Gaussian NLL.}
If the flow terminates here, i.e.\ $\mathbf{U}^{(1)} \sim \mathcal{N}(\mathbf{0},\mathbf{I})$
exactly, the resulting model is mathematically identical to directly
parameterizing $\varepsilon_{h,c} \mid \mathbf{F}_{:,c} \sim \mathcal{N}(0,\sigma_{h,c}^2)$
and minimizing its negative log-likelihood. Writing out the per-entry
change-of-variables loss,
\begin{align}
    \mathcal{L}_{\text{scale}}
    &= -\log \mathcal{N}(v_{h,c};0,1) - \mathrm{ldj}_{\text{scale}} \nonumber \\
    &= \frac{1}{2}\log(2\pi) + \frac{1}{2}\big(u_{h,c}\, s_{h,c}\big)^{2}
       - \log s_{h,c} \nonumber \\
    &= \frac{1}{2}\log\big(2\pi\,\sigma_{h,c}^2\big) + \frac{\varepsilon_{h,c}^2}{2\,\sigma_{h,c}^2},
    \label{eq:scale-loss-equiv}
\end{align}
which is exactly the negative log-density of $\mathcal{N}(0,\sigma_{h,c}^2)$
evaluated at $\varepsilon_{h,c}$, which is the same objective realized by a direct
second stage heteroscedastic Gaussian model (MVE-2S) trained with \texttt{GaussianNLLLoss} and a
zero mean. This equivalence holds regardless of the specific link function
used to map ScaleNet's raw output to $s_{h,c}$ (equivalently $\sigma_{h,c}$),
as long as the same function is reproducible by both parameterizations; in
our implementation we use the bounded, numerically stable form of
\Cref{eq:scale}. \textbf{The practical implication is that conditional
affine scaling, by itself, can only ever transform the residual into a
distribution that is exactly Gaussian} (given a perfectly trained
ScaleNet): it can stretch or compress the density along each axis as a
function of $\mathbf{F}_{:,c}$, but it has no mechanism to alter the
\emph{shape} of the distribution beyond second-order (variance) effects.
Any genuinely non-Gaussian structure in the residual like heavy tails,
multi-modality, or sharply peaked unimodal mass near zero is invisible
to a purely affine transform and is consequently absorbed into the
likelihood as irreducible model misspecification.

\begin{table*}[t]
    \centering
    \resizebox{\textwidth}{!}{\footnotesize
    \setlength{\tabcolsep}{1mm}
    \begin{tabular}{c|c|cccccccccc}
    \toprule
        Model & Metric & ETTm1-L & ETTm2-L & ETTh1-L & ETTh2-L & Electricity-L & Traffic-L & Weather-L & Exchange-L & ILI-L \\ \midrule
        \multirow{2}{*}{FITS} & CRPS &$0.305\scriptstyle\pm0.024$
        &$0.449\scriptstyle\pm0.034$	&$0.348\scriptstyle\pm0.025$	&$0.314\scriptstyle\pm0.022$	&$0.115\scriptstyle\pm0.024$	&$0.374\scriptstyle\pm0.004$	&$0.267\scriptstyle\pm0.003$	& ${{0.074}\scriptstyle\pm0.011}$
        &$0.211 \scriptstyle{\pm 0.011}$ \\
        ~ & NMAE &$0.406\scriptstyle\pm0.072$	&$0.540\scriptstyle\pm0.052$	&$0.468\scriptstyle\pm0.012$	&$0.401\scriptstyle\pm0.022$	&$0.149\scriptstyle\pm0.012$	&$0.453\scriptstyle\pm0.022$	&$0.317\scriptstyle\pm0.021$
        &${{0.097}\scriptstyle\pm0.011}$&$0.245\scriptstyle\pm0.017$  \\ \hline
        \multirow{2}{*}{PatchTST} & CRPS &$0.304\scriptstyle\pm0.029$	&${{0.229}\scriptstyle\pm0.036}$	&$0.323\scriptstyle\pm0.020$	&$0.304\scriptstyle\pm0.018$	&$0.127\scriptstyle\pm0.015$	&$0.214\scriptstyle\pm0.001$	&$0.142\scriptstyle\pm0.005$	&$0.097\scriptstyle\pm0.007$
        &$0.233 \scriptstyle{\pm 0.019}$ \\
        ~ & NMAE & $0.382\scriptstyle\pm0.066$	&${{0.288}\scriptstyle\pm0.034}$	&$0.428\scriptstyle\pm0.024$	&$0.371\scriptstyle\pm0.021$	&$0.164\scriptstyle\pm0.024$	&${{0.253}\scriptstyle\pm0.012}$	&$0.152\scriptstyle\pm0.029$	&$0.126\scriptstyle\pm0.001$	&$0.287\scriptstyle\pm0.023$ \\ \hline
        \multirow{2}{*}{iTrans.} & CRPS & $0.455\scriptstyle\pm0.021$	&$0.311\scriptstyle\pm0.024$	&$0.350\scriptstyle\pm0.019$	&$0.542\scriptstyle\pm0.015$	&$0.109\scriptstyle\pm0.044$	&$0.284\scriptstyle\pm0.004$	&$0.133\scriptstyle\pm0.004$	&$0.087\scriptstyle\pm0.023$
        &$0.222 \scriptstyle{\pm 0.020}$ \\
        ~ & NMAE &$0.490\scriptstyle\pm0.038$	&$0.385\scriptstyle\pm0.042$	&$0.449\scriptstyle\pm0.022$	&$0.667\scriptstyle\pm0.012$	&$0.140\scriptstyle\pm0.009$	&$0.361\scriptstyle\pm0.030$	&$0.147\scriptstyle\pm0.019$	&$0.113\scriptstyle\pm0.015$	&$0.278\scriptstyle\pm0.017$ \\ \hline
        \multirow{2}{*}{Koopa} & CRPS & ${{0.295}\scriptstyle\pm0.027}$	&$0.233\scriptstyle\pm0.025$	&${{0.318}\scriptstyle\pm0.009}$	&${{0.293}\scriptstyle\pm0.026}$	&$0.113\scriptstyle\pm0.018$	&$0.358\scriptstyle\pm0.022$	&$0.140\scriptstyle\pm0.007$	&$0.091\scriptstyle\pm0.012$
        &$0.228 \scriptstyle{\pm 0.022}$ \\
        ~ & NMAE &${{0.377}\scriptstyle\pm0.037}$	&$0.290\scriptstyle\pm0.033$	&${{0.412}\scriptstyle\pm0.008}$	&${{0.286}\scriptstyle\pm0.042}$	&$0.149\scriptstyle\pm0.025$	&$0.432\scriptstyle\pm0.032$	&$0.162\scriptstyle\pm0.009$	&$0.116\scriptstyle\pm0.022$	&$0.288\scriptstyle\pm0.031$ \\ \hline
        \multirow{2}{*}{TSDiff} & CRPS &$0.478\scriptstyle\pm0.027$	&$0.344\scriptstyle\pm0.046$	&$0.516\scriptstyle\pm0.027$	&$0.406\scriptstyle\pm0.056$	&$0.478\scriptstyle\pm0.005$	&$0.391\scriptstyle\pm0.002$	&$0.152\scriptstyle\pm0.003$	&$0.082\scriptstyle\pm0.010$
        &$0.263 \scriptstyle{\pm 0.022}$  \\
        ~ & NMAE & $0.622\scriptstyle\pm0.045$	&$0.416\scriptstyle\pm0.065$	&$0.657\scriptstyle\pm0.017$	&$0.482\scriptstyle\pm0.022$	&$0.622\scriptstyle\pm0.142$	&$0.478\scriptstyle\pm0.006$	&$0.141\scriptstyle\pm0.026$	&$0.142\scriptstyle\pm0.009$	&$0.272\scriptstyle\pm0.020$ \\ \hline
        \multirow{2}{*}{GRU NVP} & CRPS &$0.546\scriptstyle\pm0.036$	&$0.561\scriptstyle\pm0.273$	&$0.502\scriptstyle\pm0.039$	&$0.539\scriptstyle\pm0.090$	&$0.114\scriptstyle\pm0.013$	&${{0.211}\scriptstyle\pm0.004}$	&$0.110\scriptstyle\pm0.004$	&$0.079\scriptstyle\pm0.009$	&$0.307\scriptstyle\pm0.005$  \\
        ~ & NMAE & $0.707\scriptstyle\pm0.050$	&$0.749\scriptstyle\pm0.385$	&$0.643\scriptstyle\pm0.046$	&$0.688\scriptstyle\pm0.161$	&$0.144\scriptstyle\pm0.017$	&$0.264\scriptstyle\pm0.006$	&$0.135\scriptstyle\pm0.008$	&$0.103\scriptstyle\pm0.009$	&$0.333\scriptstyle\pm0.005$ \\ \hline
        \multirow{2}{*}{GRU MAF} & CRPS &$0.536\scriptstyle\pm0.033$	&$0.272\scriptstyle\pm0.029$	&$0.393\scriptstyle\pm0.043$	&$0.990\scriptstyle\pm0.023$	&${{0.106}\scriptstyle\pm0.007}$	& \textemdash	&$0.122\scriptstyle\pm0.006$	&$0.160\scriptstyle\pm0.019$
        &$0.172 \scriptstyle{\pm 0.034}$ \\
        ~ & NMAE & $0.711\scriptstyle\pm0.081$	&$0.355\scriptstyle\pm0.048$	&$0.496\scriptstyle\pm0.019$	&$1.092\scriptstyle\pm0.019$	&$0.136\scriptstyle\pm0.098$
        &\textemdash	&$0.149\scriptstyle\pm0.034$	&$0.182\scriptstyle\pm0.010$	&$0.216\scriptstyle\pm0.014$  \\ \hline
        \multirow{2}{*}{Trans MAF} & CRPS & $0.688\scriptstyle\pm0.043$	&$0.355\scriptstyle\pm0.043$	&$0.363\scriptstyle\pm0.053$	&$0.327\scriptstyle\pm0.033$	&\textemdash	&\textemdash	&$0.113\scriptstyle\pm0.004$	&$0.148\scriptstyle\pm0.017$	&${{0.155} \scriptstyle{\pm 0.018}}$ \\
        ~ & NMAE &$0.822\scriptstyle\pm0.034$	&$0.475\scriptstyle\pm0.029$	&$0.455\scriptstyle\pm0.025$	&$0.412\scriptstyle\pm0.020$	&\textemdash	&\textemdash	&$0.148\scriptstyle\pm0.040$	&$0.191\scriptstyle\pm0.006$	&${{0.183}\scriptstyle\pm0.019}$ \\ \hline
        \multirow{2}{*}{TimeGrad} & CRPS & $0.621\scriptstyle\pm0.037$	&$0.470\scriptstyle\pm0.054$	&$0.523\scriptstyle\pm0.027$	&$0.445\scriptstyle\pm0.016$	&$0.108\scriptstyle\pm0.003$	&$0.220\scriptstyle\pm0.002$	&$0.113\scriptstyle\pm0.011$	&$0.099\scriptstyle\pm0.015$	&$0.295\scriptstyle\pm0.083$ \\
        ~ & NMAE & $0.793\scriptstyle\pm0.034$	&$0.561\scriptstyle\pm0.044$	&$0.672\scriptstyle\pm0.015$	&$0.550\scriptstyle\pm0.018$	&${{0.134}\scriptstyle\pm0.004}$	&$0.263\scriptstyle\pm0.001$	&$0.136\scriptstyle\pm0.020$	&$0.113\scriptstyle\pm0.016$	&$0.325\scriptstyle\pm0.068$  \\ \hline
        \multirow{2}{*}{CSDI} & CRPS & $0.448\scriptstyle\pm0.038$	&$0.239\scriptstyle\pm0.035$	&$0.528\scriptstyle\pm0.012$	&$0.302\scriptstyle\pm0.040$	&\textemdash	&\textemdash	&${{0.087}\scriptstyle\pm0.003}$	&$0.143\scriptstyle\pm0.020$	&$0.283\scriptstyle\pm0.012$  \\
        ~ & NMAE & $0.578\scriptstyle\pm0.051$	&$0.306\scriptstyle\pm0.040$	&$0.657\scriptstyle\pm0.014$	&$0.382\scriptstyle\pm0.030$	&\textemdash	&\textemdash	&$ {{0.102}\scriptstyle\pm0.005}$	&$0.173\scriptstyle\pm0.020$	&$0.299\scriptstyle\pm0.013$ \\ \hline
        \multirow{2}{*}{$K^2$VAE} & CRPS & $\underline{\textit{0.294}}\scriptstyle\pm0.026$	&$\underline{\textit{0.221}}\scriptstyle\pm0.023$	&$\underline{\textit{0.314}}\scriptstyle\pm0.011$	&$\underline{\textit{0.280}}\scriptstyle\pm0.014$	&$\textbf{0.057}\scriptstyle\pm0.005$	&$\textbf{0.200}\scriptstyle\pm0.001$	&$\underline{\textit{0.084}}\scriptstyle\pm0.003$	&$\underline{\textit{0.069}}\scriptstyle\pm0.005$	&$\underline{\textit{0.142}} \scriptstyle{\pm 0.008}$ \\
        ~ & NMAE & $\underline{\textit{0.373}}\scriptstyle\pm0.032$	&$\underline{\textit{0.275}}\scriptstyle\pm0.035$	&$\underline{\textit{0.396}}\scriptstyle\pm0.012$	&$\underline{\textit{0.278}}\scriptstyle\pm0.020$	&$\underline{\textit{0.117}}\scriptstyle\pm0.019$	&$\underline{\textit{0.248}}\scriptstyle\pm0.010$	&$\underline{\textit{0.099}}\scriptstyle\pm0.009$	&$\underline{\textit{0.084}}\scriptstyle\pm0.017$	&$\underline{\textit{0.167}}\scriptstyle\pm0.007$\\ \hline
        \multirow{2}{*}{\model{}} & CRPS & $\textbf{0.279}\scriptstyle\pm0.001$	&$\textbf{0.166}\scriptstyle\pm0.000$	&$\textbf{0.286}\scriptstyle\pm0.001$	&$\textbf{0.181}\scriptstyle\pm0.000$	&$\underline{\textit{0.081}}\scriptstyle\pm0.000$	&$\underline{\textit{0.203}}\scriptstyle\pm0.001$	&$\textbf{0.079}\scriptstyle\pm0.001$	&$\textbf{0.059}\scriptstyle\pm0.000$	&$\textbf{0.131} \scriptstyle{\pm 0.000}$ \\
        ~ & NMAE & $\textbf{0.354}\scriptstyle\pm0.003$	&$\textbf{0.206}\scriptstyle\pm0.001$	&$\textbf{0.367}\scriptstyle\pm0.001$	&$\textbf{0.229}\scriptstyle\pm0.020$	&$\textbf{0.108}\scriptstyle\pm0.003$	&$\textbf{0.242}\scriptstyle\pm0.004$	&$\textbf{0.096}\scriptstyle\pm0.001$	&$\textbf{0.080}\scriptstyle\pm0.000$	&$\textbf{0.158}\scriptstyle\pm0.007$\\ \hline
        \multirow{2}{*}{Imp (\%)} & CRPS & $+5.1$ & $+24.9$ & $+8.9$ & $+35.4$ & $-42.1$ & $-1.5$ & $+6.0$ & $+14.5$ & $+7.7$ \\
        ~ & NMAE & $+5.1$ & $+25.1$ & $+7.3$ & $+17.6$ & $+7.7$ & $+2.4$ & $+3.0$ & $+4.8$ & $+5.4$ \\
        \bottomrule
        \multicolumn{8}{l}{Due to the excessive time and memory consumption, some results are unavailable and denoted as -.} \\
    \end{tabular}}
    \caption{Comparison on long-term probabilistic forecasting (forecasting horizon L=720) scenarios across nine real-world datasets. Lower CRPS or NMAE values indicate better predictions. The means and standard errors are based on 5 independent runs of retraining and evaluation. \textbf{Bold}: the best, \underline{\textit{underlined italic}}: the 2nd best. The full results of all four horizons {96, 192, 336, 720} are listed in \Cref{sec:ablation:allhorizonresults}.}
    \label{tab: long-term:appendix}

\end{table*}

\begin{table*}[t]
    \centering
    \footnotesize
    \setlength{\tabcolsep}{1.5mm}
    \begin{tabular}{c|c|cccccccc}
    \toprule
        Model & Metric & Exchange-S & Solar-S & Electricity-S & Traffic-S & ETTh1-S & ETTh2-S & ETTm1-S & ETTm2-S  \\ \midrule
        \multirow{2}{*}{FITS} & CRPS & $0.012\scriptstyle{\pm 0.002}$ & $0.516\scriptstyle{\pm 0.011} $& $0.068\scriptstyle{\pm 0.003}$ & $0.298\scriptstyle{\pm 0.022}$ & $0.320\scriptstyle{\pm 0.017}$ & $0.212\scriptstyle{\pm 0.012}$ & $0.193\scriptstyle{\pm 0.005}$ & $0.199\scriptstyle{\pm 0.003}$  \\
        ~ & NMAE & $0.017\scriptstyle{\pm 0.003}$ & $0.701\scriptstyle{\pm 0.014}$ & $0.092\scriptstyle{\pm 0.004}$ & $0.392\scriptstyle{\pm 0.028}$ & $0.423\scriptstyle{\pm 0.033}$ & $0.278\scriptstyle{\pm 0.009}$ & $0.249\scriptstyle{\pm 0.007}$ & $0.260\scriptstyle{\pm 0.011}$ \\ \hline
        \multirow{2}{*}{PatchTST} & CRPS & $0.052\scriptstyle{\pm 0.016}$ & $0.491\scriptstyle{\pm 0.008}$ & $0.063\scriptstyle{\pm 0.003}$ & $0.278\scriptstyle{\pm 0.018}$ & $0.314\scriptstyle{\pm 0.022}$ & $0.207\scriptstyle{\pm 0.006}$ & $0.234\scriptstyle{\pm 0.011}$ & $0.212\scriptstyle{\pm 0.018}$  \\
        ~ & NMAE & $0.069\scriptstyle{\pm 0.013}$ & $0.663\scriptstyle{\pm 0.010}$ & $0.085\scriptstyle{\pm 0.006}$ & $0.363\scriptstyle{\pm 0.023}$ & $0.407\scriptstyle{\pm 0.030}$ & $0.260\scriptstyle{\pm 0.009}$ & $0.271\scriptstyle{\pm 0.009}$ & $0.257\scriptstyle{\pm 0.011}$  \\ \hline
        \multirow{2}{*}{iTrans.} & CRPS & $0.059\scriptstyle{\pm 0.018}$ & $0.504\scriptstyle{\pm 0.012}$ & $0.066\scriptstyle{\pm 0.004}$ & $0.244\scriptstyle{\pm 0.011}$ & $0.317\scriptstyle{\pm 0.020}$ & $0.219\scriptstyle{\pm 0.008}$ & $0.254\scriptstyle{\pm 0.012}$ & $0.201\scriptstyle{\pm 0.018}$  \\
        ~ & NMAE & $0.081\scriptstyle{\pm 0.022}$ & $0.695\scriptstyle{\pm 0.017}$ & $0.087\scriptstyle{\pm 0.006}$ & $0.319\scriptstyle{\pm 0.019}$ & $0.408\scriptstyle{\pm 0.028}$ & $0.276\scriptstyle{\pm 0.017}$ & $0.291\scriptstyle{\pm 0.017}$ & $0.242\scriptstyle{\pm 0.009}$
        \\ \hline
        \multirow{2}{*}{Koopa} & CRPS & $0.012\scriptstyle{\pm 0.001}$ & $0.545\scriptstyle{\pm 0.016}$ & $0.085\scriptstyle{\pm 0.014}$ & $0.253\scriptstyle{\pm 0.018}$ & $0.326\scriptstyle{\pm 0.013}$ & $0.211\scriptstyle{\pm 0.019}$ & $0.288\scriptstyle{\pm 0.022}$ & $0.220\scriptstyle{\pm 0.015}$  \\
        ~ & NMAE & $0.015\scriptstyle{\pm 0.002}$ & $0.742\scriptstyle{\pm 0.022}$ & $0.112\scriptstyle{\pm 0.019}$ & $0.330\scriptstyle{\pm 0.019}$ & $0.423\scriptstyle{\pm 0.017}$ & $0.266\scriptstyle{\pm 0.022}$ & $0.329\scriptstyle{\pm 0.026}$ & $0.278\scriptstyle{\pm 0.022}$  \\ \hline
        \multirow{2}{*}{TSDiff} & CRPS & $0.077\scriptstyle{\pm 0.019}$ & $0.568\scriptstyle{\pm 0.015}$ & $0.111\scriptstyle{\pm 0.013}$ & $0.189\scriptstyle{\pm 0.009}$ & $0.304\scriptstyle{\pm 0.016}$ & $0.204\scriptstyle{\pm 0.006}$ &$0.209\scriptstyle{\pm 0.013}$ & ${{0.124}\scriptstyle{\pm 0.008}}$  \\
        ~ & NMAE & $0.096\scriptstyle{\pm 0.024}$ & $0.635\scriptstyle{\pm 0.012}$ & $0.115\scriptstyle{\pm 0.018}$ & $0.206\scriptstyle{\pm 0.011}$ & $0.400\scriptstyle{\pm 0.025}$ & $0.272\scriptstyle{\pm 0.015}$ & $0.276\scriptstyle{\pm 0.008}$ & ${{0.162}\scriptstyle{\pm 0.008}}$  \\ \hline
        \multirow{2}{*}{$D^3$VAE} & CRPS & $0.011\scriptstyle{\pm 0.002}$ & $0.769\scriptstyle{\pm 0.029}$ & $0.071\scriptstyle{\pm 0.009}$ & $0.143\scriptstyle{\pm 0.008}$ & $0.324\scriptstyle{\pm 0.019}$ & $0.216\scriptstyle{\pm 0.015}$ &$0.198\scriptstyle{\pm 0.015}$ & $0.303\scriptstyle{\pm 0.024}$  \\
        ~ & NMAE & $\underline{\textit{0.012}}\scriptstyle{\pm 0.002}$ & $0.998\scriptstyle{\pm 0.049}$ & $0.092\scriptstyle{\pm 0.013}$ & $0.178\scriptstyle{\pm 0.013}$ & $0.410\scriptstyle{\pm 0.016}$ & $0.267\scriptstyle{\pm 0.018}$ & $0.250\scriptstyle{\pm 0.018}$ & $0.378\scriptstyle{\pm 0.031}$  \\ \hline
        \multirow{2}{*}{GRU NVP} & CRPS & $0.019\scriptstyle{\pm 0.006}$ & $0.530\scriptstyle{\pm 0.008}$ & $0.062\scriptstyle{\pm 0.003}$ & $0.168\scriptstyle{\pm 0.008}$ & $0.398\scriptstyle{\pm 0.034}$ & $0.309\scriptstyle{\pm 0.023}$ & $0.455\scriptstyle{\pm 0.029}$ & $0.276\scriptstyle{\pm 0.014}$  \\
        ~ & NMAE & $0.024\scriptstyle{\pm 0.007}$ & $0.670\scriptstyle{\pm 0.011}$ & $0.081\scriptstyle{\pm 0.006}$ & $0.209\scriptstyle{\pm 0.013}$ & $0.477\scriptstyle{\pm 0.040}$ & $0.375\scriptstyle{\pm 0.024}$ & $0.584\scriptstyle{\pm 0.047}$ & $0.349\scriptstyle{\pm 0.028}$  \\ \hline
        \multirow{2}{*}{GRU MAF} & CRPS & $0.012\scriptstyle{\pm 0.003}$ & $0.486\scriptstyle{\pm 0.007}$ & $0.056\scriptstyle{\pm 0.002}$ & $0.144\scriptstyle{\pm 0.022}$ & ${{0.258}\scriptstyle{\pm 0.013}}$ & $0.160\scriptstyle{\pm 0.008}$ & $0.151\scriptstyle{\pm 0.009}$ & $0.146\scriptstyle{\pm 0.011}$  \\
        ~ & NMAE & $0.016\scriptstyle{\pm 0.002}$ & $0.603\scriptstyle{\pm 0.009}$ & $0.073\scriptstyle{\pm 0.004}$ & $0.182\scriptstyle{\pm 0.029}$ & ${{0.326}\scriptstyle{\pm 0.016}}$ & $0.208\scriptstyle{\pm 0.003}$ & $0.198\scriptstyle{\pm 0.004}$ & $0.193\scriptstyle{\pm 0.008}$  \\ \hline
        \multirow{2}{*}{Trans MAF} & CRPS & $0.012\scriptstyle{\pm 0.001} $& $0.442\scriptstyle{\pm 0.011}$ & $0.054\scriptstyle{\pm 0.002}$ & $0.133\scriptstyle{\pm 0.004}$ & $0.309\scriptstyle{\pm 0.009}$ & $0.200\scriptstyle{\pm 0.012}$ & ${{0.139}\scriptstyle{\pm 0.005}}$ & $0.180\scriptstyle{\pm 0.010}$  \\
        ~ & NMAE & $0.016\scriptstyle{\pm 0.001}$ & $0.577\scriptstyle{\pm 0.014}$ & $0.071\scriptstyle{\pm 0.003}$ & $0.160\scriptstyle{\pm 0.006}$ & $0.400\scriptstyle{\pm 0.011}$ & $0.256\scriptstyle{\pm 0.009}$ & ${{0.162}\scriptstyle{\pm 0.006}}$ & $0.224\scriptstyle{\pm 0.009}$  \\ \hline
        \multirow{2}{*}{TimeGrad} & CRPS & $\underline{\textit{0.009}}\scriptstyle{\pm 0.001}$& $0.465\scriptstyle{\pm 0.016}$ & $0.057\scriptstyle{\pm 0.002}$ & ${{0.130}\scriptstyle{\pm 0.005}}$ & $0.273\scriptstyle{\pm 0.007}$ &$0.184\scriptstyle{\pm 0.006}$ & $0.186\scriptstyle{\pm 0.003}$ & $0.148\scriptstyle{\pm 0.004}$  \\
        ~ & NMAE & $\underline{\textit{0.012}}\scriptstyle{\pm 0.002}$& $0.609\scriptstyle{\pm 0.015}$ & $0.073\scriptstyle{\pm 0.004}$ & $\textbf{0.155}\scriptstyle{\pm 0.007}$ & $0.356\scriptstyle{\pm 0.013}$ & $0.224\scriptstyle{\pm 0.014}$ &$ 0.246\scriptstyle{\pm 0.007}$ & $0.189\scriptstyle{\pm 0.006}$  \\ \hline
        \multirow{2}{*}{CSDI} & CRPS & $\underline{\textit{0.009}}\scriptstyle{\pm 0.001}$ & $\underline{\textit{0.392}}\scriptstyle{\pm 0.006}$ & $\textbf{0.051}\scriptstyle{\pm 0.001}$ & $0.147\scriptstyle{\pm 0.014}$ & $0.262\scriptstyle{\pm 0.012}$ & ${{0.133}\scriptstyle{\pm 0.006}}$ & $0.140\scriptstyle{\pm 0.012}$ & $0.144\scriptstyle{\pm 0.018}$  \\
        ~ & NMAE & $0.013\scriptstyle{\pm 0.001}$& $\underline{\textit{0.533}}\scriptstyle{\pm 0.007}$& $\textbf{0.066}\scriptstyle{\pm 0.001}$ & $0.175\scriptstyle{\pm 0.013}$ & $0.339\scriptstyle{\pm 0.009}$ & ${{0.161}\scriptstyle{\pm 0.013}}$ & $0.169\scriptstyle{\pm 0.021}$ & $0.181\scriptstyle{\pm 0.024}$  \\ \hline
        \multirow{2}{*}{$K^2$VAE} & CRPS & $\underline{\textit{0.009}}\scriptstyle{\pm 0.001}$& $\textbf{0.367}\scriptstyle{\pm 0.006}$ & $\underline{\textit{0.053}}\scriptstyle{\pm 0.002}$ & $\underline{\textit{0.129}}\scriptstyle{\pm 0.004}$ & $\underline{\textit{0.256}}\scriptstyle{\pm 0.008}$ & $\underline{\textit{0.128}}\scriptstyle{\pm 0.006}$ & $\underline{\textit{0.135}}\scriptstyle{\pm 0.008}$ & $\underline{\textit{0.122}}\scriptstyle{\pm 0.008}$  \\
        ~ & NMAE & $\textbf{0.009}\scriptstyle{\pm 0.001}$ & $\textbf{0.480}\scriptstyle{\pm 0.008}$ & $\underline{\textit{0.068}}\scriptstyle{\pm 0.002}$& $\underline{\textit{0.157}}\scriptstyle{\pm 0.007}$ & $\underline{\textit{0.312}}\scriptstyle{\pm 0.008}$ & $\underline{\textit{0.140}}\scriptstyle{\pm 0.007}$ & $\underline{\textit{0.152}}\scriptstyle{\pm 0.007}$ & $\underline{\textit{0.146}}\scriptstyle{\pm 0.009}$ \\ \hline
        \multirow{2}{*}{\model{}} & CRPS & $\textbf{0.008}\scriptstyle{\pm 0.000}$& ${{0.456}\scriptstyle{\pm 0.002}}$ & $\underline{\textit{0.053}}\scriptstyle{\pm 0.000}$ & $\textbf{0.126}\scriptstyle{\pm 0.000}$ & $\textbf{0.231}\scriptstyle{\pm 0.001}$ & $\textbf{0.106}\scriptstyle{\pm 0.000}$ & $\textbf{0.108}\scriptstyle{\pm 0.000}$ & $\textbf{0.085}\scriptstyle{\pm 0.000}$  \\
        ~ & NMAE & $\textbf{0.009}\scriptstyle{\pm 0.001}$ & ${{0.583}\scriptstyle{\pm 0.025}}$ & ${{0.070}\scriptstyle{\pm 0.001}}$& ${{0.159}\scriptstyle{\pm 0.000}}$ & $\textbf{0.286}\scriptstyle{\pm 0.001}$ & $\textbf{0.132}\scriptstyle{\pm 0.001}$ & $\textbf{0.139}\scriptstyle{\pm 0.001}$ & $\textbf{0.105}\scriptstyle{\pm 0.003}$ \\ \hline
        \multirow{2}{*}{Imp (\%)} & CRPS & $+11.1$ & $-24.3$ & $-3.9$ & $+2.3$ & $+9.8$ & $+17.2$ & $+20.0$ & $+30.3$ \\
        ~ & NMAE & $+0.0$ & $-21.5$ & $-6.1$ & $-2.6$ & $+8.3$ & $+5.7$ & $+8.6$ & $+28.1$ \\
        \bottomrule
    \end{tabular}
    \caption{Comparison on short-term probabilistic forecasting scenarios across eight real-world datasets. Lower CRPS or NMAE values indicate better predictions. The means and standard errors are based on 5 independent runs of retraining and evaluation. \textbf{Bold}: the best, \underline{\textit{underlined italic}}: the 2nd best.}
    \label{tab: short-term:appendix}
\end{table*}

\section{Odd RealNVP Construction}
\label{sec:ablation:oddRealNVP}

This appendix details the construction of \textit{w OddRealNVP}
variant of \model{}: \ross\ is replaced
by a RealNVP-style affine coupling layer~\citep{Dinh2016DensityEU}, while \scl\
 is left unchanged, so the ablation isolates the coupling transform alone.

\paragraph{Restoring oddness in a RealNVP coupling.}
At each time step $h$, the coupling layer masks the $C$ channels with
$\mathbf{m}\in\{0,1\}^C$ and applies $v_{h,c} = s_{h,c}\cdot u_{h,c} + t_{h,c}$
to the unmasked channels, with $s_{h,c}, t_{h,c}$ predicted from the masked
channels $\mathbf{u}^{\mathbf m}_h := \mathbf{m}\odot\mathbf{U}_{h,:}$ and
context $\mathbf{F}_{h,:}$ with the help of an MLP network $r$. Left unconstrained, this need not satisfy the
odd condition. Substituting
$\mathbf{U}_{h,:}\mapsto-\mathbf{U}_{h,:}$ shows the map is odd (with
$\mathbf{F}$ fixed) precisely when $s_{h,c}$ is an \emph{even} and $t_{h,c}$
an \emph{odd} function of $\mathbf{u}^{\mathbf m}_h$:
\begin{align}
	s_{h,c}(-\mathbf{u}^{\mathbf m}_h) &= s_{h,c}(\mathbf{u}^{\mathbf m}_h), \label{eq:rnvp-even}\\
	t_{h,c}(-\mathbf{u}^{\mathbf m}_h) &= -t_{h,c}(\mathbf{u}^{\mathbf m}_h), \label{eq:rnvp-odd}
\end{align}
the same even/odd requirement \citet{Kobayashi2023DesignOR} give for
restricting a RealNVP coupling to be odd. Any raw, unconstrained network
output $r(\cdot)$ is split into such a pair by evaluating it on both $+$
and $-$ its input:
\begin{align}
	s_{h,c} &= \tfrac{1}{2}\big[r_s(\mathbf{u}^{\mathbf m}_h) + r_s(-\mathbf{u}^{\mathbf m}_h)\big], \label{eq:rnvp-sym-s}\\
	t_{h,c} &= \tfrac{1}{2}\big[r_t(\mathbf{u}^{\mathbf m}_h) - r_t(-\mathbf{u}^{\mathbf m}_h)\big], \label{eq:rnvp-sym-t}
\end{align}
exactly what the \texttt{EvenNetwork}/\texttt{OddNetwork} pair computes:
each runs the same backbone once on $+\mathbf{u}^{\mathbf m}_h$ and once on
$-\mathbf{u}^{\mathbf m}_h$, recombining with a $+$ (scale) or $-$
(translation) sign.

\begin{lstlisting}[language=Python, caption={Odd-making step of OddRealNVP.}, label={lst:realnvp}]
xc_pos = r(x)
xc_neg = r(-x)
t = 0.5 * (xc_pos - xc_neg)  # odd
s = 0.5 * (xc_pos + xc_neg)  # even
\end{lstlisting}

\section{Mixture of Odd Flows Construction}
\label{sec:ablation:mixOddFlows}

Single-component \model{} is restricted to \emph{symmetric} residual densities
(\Cref{sec:method:twostage}, Assumption~2), so it cannot represent skewed
residuals. TORF-Mixture lifts this restriction by mixing $M$ independently
parameterized copies of the \scl$\to$\ross\ flow, while keeping the exact
zero-mean guarantee of the single-component model.

\paragraph{Shifted components.}
Let $T^{(i)}(\cdot\,;\mathbf{F})$, $i=1,\dots,M$, be $M$ independent
\scl$\to$\ross\ flows, each with its own
parameters and each individually odd. Instead of applying $T^{(i)}$ directly,
component $i$ first shifts $u_{h,c}$ by its own learned location
$\beta^{(i)}_{h,c}(\mathbf{F})$:
\begin{equation}
	v^{(i)}_{h,c},\; \ell^{(i)}_{h,c} = T^{(i)}\!\left(u_{h,c} - \beta^{(i)}_{h,c}(\, \mathbf{F})\right),
	\label{eq:mix-component}
\end{equation}
with $\ell^{(i)}_{h,c}$ the log-determinant of $T^{(i)}$ at $(h,c)$. Since
$T^{(i)}$ is odd and $\mathcal{N}(0,1)$ is symmetric, component $i$'s density over $u_{h,c}$ is symmetric about $\beta^{(i)}_{h,c}$ rather than about $0$.

\paragraph{Weights and the centering constraint.}
Two lightweight MLPs on $\mathbf{F}$, applied per time step and shared across
channels as ScaleNet is (\Cref{sec:method:scl}), output the raw weight logits and shifts
$\beta^{(i)}_{h,c}$.
Weights are a softmax floored at $\delta=\text{\texttt{min\_weight}}$ so no
component's weight vanishes during training:
\begin{equation}
	\pi^{(i)}_{h,c} = (1 - M\delta)\cdot\mathrm{softmax}_i(\cdot) + \delta,
	\qquad \textstyle\sum_{i=1}^{M}\pi^{(i)}_{h,c} = 1.
	\label{eq:mix-weights}
\end{equation}
The raw shifts are then re-centered against these weights,
\begin{equation}
	\beta^{(i)}_{h,c} \;\leftarrow\; \beta^{(i)}_{h,c} - \textstyle\sum_{j=1}^{M}\pi^{(j)}_{h,c}\,\beta^{(j)}_{h,c},
	\label{eq:mix-center}
\end{equation}
which forces $\sum_{i=1}^{M}\pi^{(i)}_{h,c}\beta^{(i)}_{h,c} = 0$ exactly, for
every $(h,c)$ and $\mathbf{F}$, by construction rather than by training. The
shift head is zero-initialized, so training starts from $M$ coincident
components equivalent to single-component \model, and only splits them apart
when doing so improves likelihood.

Each $T^{(i)}$ acts elementwise, the mixture
factorizes exactly per residual entry and, under Conditional Independence, training
minimizes $\mathcal{L} = -\sum_{h,c}\log p(u_{h,c}\mid\mathbf{F})$ as before.
Sampling draws a component $z_{h,c}\sim\mathrm{Categorical}(\pi^{(1)}_{h,c},\dots,\pi^{(M)}_{h,c})$.

\paragraph{Mean preservation despite asymmetry.}
Component $i$ has mean $\mathbb{E}[u_{h,c}\mid z_{h,c}=i,\mathbf{F}] = \beta^{(i)}_{h,c}$,
so the mixture mean is
\begin{equation}
	\mathbb{E}[u_{h,c}\mid\mathbf{F}]
	= \sum_{i=1}^{M} \pi^{(i)}_{h,c}\,\beta^{(i)}_{h,c} = 0
\end{equation}
by the centering constraint \eqref{eq:mix-center} alone -- for any weights, any
$M$, and any component shapes, so
$\mathbb{E}_{p_{\theta_2}}[\boldsymbol{\varepsilon}\mid\mathbf{X}]=\mathbf{0}$
carries over from single-component \model{} unchanged.

This same construction is also what lets the mixture be asymmetric. Components
symmetric about a \emph{common} point would keep the mixture symmetric about
that point regardless of weights. Here each component is symmetric about its
own $\beta^{(i)}_{h,c}$, and only their weighted \emph{average} is pinned to
zero, so the resulting mixture can be asymmetric, as in classical two-piece
and split-normal distributions, here with $M$ learned, flow-shaped components.
\textbf{One consequence is that the guarantee is now on the \emph{mean}, not the median
once the mixture is genuinely asymmetric, its median need no longer coincide
with $f_{\theta_1}(\mathbf{X})$.} Setting $M=1$ collapses \model{}-Mixture back to \model{} exactly,
with $\pi^{(1)}_{h,c}=1$ and
$\beta^{(1)}_{h,c}=0$. TORF-Mixture is a strict generalization, not a separate
model family.

\begin{table*}[!htbp]
    \centering
    \setlength\tabcolsep{1mm}
    \scriptsize
    \resizebox{\textwidth}{!}{
  \begin{tabular}{c|c|c|c|c|c|c|c|c|c|c|c|c|c}
    \toprule
        Dataset & Horizon & Koopa & iTransformer & FITS & PatchTST & GRU MAF & Trans MAF & TSDiff & CSDI & TimeGrad & GRU NVP & $K^2$VAE  & \model{}\\ \midrule
        \multirow{4}{*}{ETTm1-L} & 96 &$0.285\scriptstyle\pm0.018$ & $0.301\scriptstyle\pm0.033$ & $0.267\scriptstyle\pm0.023$ & $0.261\scriptstyle\pm0.051$ & $0.295\scriptstyle\pm0.055$ & $0.313\scriptstyle\pm0.045$ & $0.344\scriptstyle\pm0.050$ & $0.236\scriptstyle\pm0.006$ & $0.522\scriptstyle\pm0.105$ & $0.383\scriptstyle\pm0.053$ & $\underline{\textit{0.232}}\scriptstyle\pm0.010$   & $\textbf{0.2088}\scriptstyle\pm0.0011$\\
        ~ & 192 &$0.289\scriptstyle\pm0.024$& $0.314\scriptstyle\pm0.023$ & $0.261\scriptstyle\pm0.022$ & $0.275\scriptstyle\pm0.030$ & $0.389\scriptstyle\pm0.033$ & $0.424\scriptstyle\pm0.029$ & $0.345\scriptstyle\pm0.035$ & $0.291\scriptstyle\pm0.025$ & $0.603\scriptstyle\pm0.092$ & $0.396\scriptstyle\pm0.030$ & $\underline{\textit{0.259}}\scriptstyle\pm0.013$   & $\textbf{0.2310}\scriptstyle\pm0.0003$\\
        ~ & 336 &$0.286\scriptstyle\pm0.035$& $0.311\scriptstyle\pm0.029$ & $0.275\scriptstyle\pm0.030$ & $0.285\scriptstyle\pm0.028$ & $0.429\scriptstyle\pm0.021$ & $0.481\scriptstyle\pm0.019$ & $0.462\scriptstyle\pm0.043$ & $0.322\scriptstyle\pm0.033$ & $0.601\scriptstyle\pm0.028$ & $0.486\scriptstyle\pm0.032$ & $\underline{\textit{0.262}}\scriptstyle\pm0.030$   & $\textbf{0.2457}\scriptstyle\pm0.0005$\\
        ~ & 720 &$0.295\scriptstyle\pm0.027$ & $0.455\scriptstyle\pm0.021$ & $0.305\scriptstyle\pm0.024$ & $0.304\scriptstyle\pm0.029$ & $0.536\scriptstyle\pm0.033$ & $0.688\scriptstyle\pm0.043$ & $0.478\scriptstyle\pm0.027$ & $0.448\scriptstyle\pm0.038$ & $0.621\scriptstyle\pm0.037$ & $0.546\scriptstyle\pm0.036$ & $\underline{\textit{0.294}}\scriptstyle\pm0.026$  & $\textbf{0.2796}\scriptstyle\pm0.0007$ \\ \midrule

        \multirow{4}{*}{ETTm2-L} & 96 &$0.178\scriptstyle\pm0.023$ & $0.181\scriptstyle\pm0.031$ & $0.162\scriptstyle\pm0.053$ & $0.142\scriptstyle\pm0.034$ & $0.177\scriptstyle\pm0.024$ & $0.227\scriptstyle\pm0.013$ & $0.175\scriptstyle\pm0.019$ & $\underline{\textit{0.115}}\scriptstyle\pm0.009$ & $0.427\scriptstyle\pm0.042$ & $0.319\scriptstyle\pm0.044$ & $0.126\scriptstyle\pm0.007$ & $\textbf{0.1091}\scriptstyle\pm0.0001$  \\
        ~ & 192 &$0.185\scriptstyle\pm0.014$& $0.190\scriptstyle\pm0.010$ &$0.185\scriptstyle\pm0.053$ & $0.172\scriptstyle\pm0.023$ & $0.411\scriptstyle\pm0.026$ & $0.253\scriptstyle\pm0.037$ & $0.255\scriptstyle\pm0.029$ & $\underline{\textit{0.147}}\scriptstyle\pm0.008$ & $0.424\scriptstyle\pm0.061$ & $0.326\scriptstyle\pm0.025$ & $0.148\scriptstyle\pm0.009$   & $\textbf{0.1279}\scriptstyle\pm0.0001$\\
        ~ & 336 &$0.198\scriptstyle\pm0.015$& $0.206\scriptstyle\pm0.055$ & $0.218\scriptstyle\pm0.053$ & $0.195\scriptstyle\pm0.042$ & $0.377\scriptstyle\pm0.023$ & $0.253\scriptstyle\pm0.013$ & $0.328\scriptstyle\pm0.047$ & $0.190\scriptstyle\pm0.018 $& $0.469\scriptstyle\pm0.049$ & $0.449\scriptstyle\pm0.145$ & $\underline{\textit{0.164}}\scriptstyle\pm0.010$  & $\textbf{0.1441}\scriptstyle\pm0.0001 $\\
        ~ & 720 &$0.233\scriptstyle\pm0.025$& $0.311\scriptstyle\pm0.024$ & $0.449\scriptstyle\pm0.034$ & $0.229\scriptstyle\pm0.036$ & $0.272\scriptstyle\pm0.029$ & $0.355\scriptstyle\pm0.043$ & $0.344\scriptstyle\pm0.046$ & $0.239\scriptstyle\pm0.035$ & $0.470\scriptstyle\pm0.054$ & $0.561\scriptstyle\pm0.273$ & $\underline{\textit{0.221}}\scriptstyle\pm0.023$   & $\textbf{0.1656}\scriptstyle\pm0.0002$\\ \midrule

        \multirow{4}{*}{ETTh1-L} & 96 &$0.307\scriptstyle\pm0.033$ & $0.292\scriptstyle\pm0.032$ & $0.294\scriptstyle\pm0.023$ & $0.312\scriptstyle\pm0.036$ & $0.293\scriptstyle\pm0.037$ & $0.333\scriptstyle\pm0.045$ & $0.395\scriptstyle\pm0.052$ & $0.437\scriptstyle\pm0.018$ & $0.455\scriptstyle\pm0.046$ & $0.379\scriptstyle\pm0.030$ & $\underline{\textit{0.264}}\scriptstyle\pm0.020$   & $\textbf{0.2426}\scriptstyle\pm0.0005$\\
        ~ & 192 &$0.301\scriptstyle\pm0.014$ & $0.298\scriptstyle\pm0.020$ & $0.304\scriptstyle\pm0.028$ & $0.313\scriptstyle\pm0.034$ & $0.348\scriptstyle\pm0.075$ & $0.351\scriptstyle\pm0.063$ & $0.467\scriptstyle\pm0.044$ & $0.496\scriptstyle\pm0.051$ & $0.516\scriptstyle\pm0.038$ & $0.425\scriptstyle\pm0.019$ & $\underline{\textit{0.290}}\scriptstyle\pm0.016$   & $\textbf{0.2606}\scriptstyle\pm0.0006$\\
        ~ & 336 &$0.312\scriptstyle\pm0.019$& $0.327\scriptstyle\pm0.043$ & $0.318\scriptstyle\pm0.023$ & $0.319\scriptstyle\pm0.035$ & $0.377\scriptstyle\pm0.026$ & $0.371\scriptstyle\pm0.031$ & $0.450\scriptstyle\pm0.027$ & $0.454\scriptstyle\pm0.025$ & $0.512\scriptstyle\pm0.026$ & $0.458\scriptstyle\pm0.054$ & $\underline{\textit{0.308}}\scriptstyle\pm0.021$  & $\textbf{0.2752}\scriptstyle\pm0.0010$\\
        ~ & 720 &$0.318\scriptstyle\pm0.009$& $0.350\scriptstyle\pm0.019$ & $0.348\scriptstyle\pm0.025$ & $0.323\scriptstyle\pm0.020$& $0.393\scriptstyle\pm0.043$ & $0.363\scriptstyle\pm0.053$ & $0.516\scriptstyle\pm0.027$ & $0.528\scriptstyle\pm0.012$ & $0.523\scriptstyle\pm0.027$ & $0.502\scriptstyle\pm0.039$ & $\underline{\textit{0.314}}\scriptstyle\pm0.011$   & $\textbf{0.2860}\scriptstyle\pm0.0006$\\ \midrule

        \multirow{4}{*}{ETTh2-L} & 96 &$0.199\scriptstyle\pm0.012$ & $0.185\scriptstyle\pm0.013$ & $0.187\scriptstyle\pm0.011$ & $0.197\scriptstyle\pm0.021$ & $0.239\scriptstyle\pm0.019$ & $0.263\scriptstyle\pm0.020$ & $0.336\scriptstyle\pm0.021$ & $0.164\scriptstyle\pm0.013$ & $0.358\scriptstyle\pm0.026$ & $0.432\scriptstyle\pm0.141$ & $\underline{\textit{0.162}}\scriptstyle\pm0.009$  & $\textbf{0.1377}\scriptstyle\pm0.0004$ \\
        ~ & 192 &$0.198\scriptstyle\pm0.022$& $0.199\scriptstyle\pm0.019$ & $0.195\scriptstyle\pm0.022$ & $0.204\scriptstyle\pm0.055$ & $0.313\scriptstyle\pm0.034$ & $0.273\scriptstyle\pm0.024$ & $0.265\scriptstyle\pm0.043$ & $0.226\scriptstyle\pm0.018$ & $0.457\scriptstyle\pm0.081$ & $0.625\scriptstyle\pm0.170$ &$\underline{\textit{0.186}}\scriptstyle\pm0.018$  & $\textbf{0.1606}\scriptstyle\pm0.0001$\\
        ~ & 336 &$0.262\scriptstyle\pm0.019$& $0.271\scriptstyle\pm0.033$ & $\underline{\textit{0.246}}\scriptstyle\pm0.044$ & $0.277\scriptstyle\pm0.054$ & $0.376\scriptstyle\pm0.034$ & $0.265\scriptstyle\pm0.042$ & $0.350\scriptstyle\pm0.031$ & $0.274\scriptstyle\pm0.022$ & $0.481\scriptstyle\pm0.078$ & $0.793\scriptstyle\pm0.319$ & $0.257\scriptstyle\pm0.023$   & $\textbf{0.1766}\scriptstyle\pm0.0007$\\
        ~ & 720 &$0.293\scriptstyle\pm0.026$& $0.542\scriptstyle\pm0.015$ & $0.314\scriptstyle\pm0.022$ & $0.304\scriptstyle\pm0.018$ & $0.990\scriptstyle\pm0.023$ & $0.327\scriptstyle\pm0.033$ & $0.406\scriptstyle\pm0.056$ & $0.302\scriptstyle\pm0.040$ & $0.445\scriptstyle\pm0.016$ & $0.539\scriptstyle\pm0.090$ & $\underline{\textit{0.280}}\scriptstyle\pm0.014$   & $\textbf{0.1809}\scriptstyle\pm0.0002$\\ \midrule

        \multirow{4}{*}{Electricity-L} & 96 & $0.110\scriptstyle\pm0.004$& $0.102\scriptstyle\pm0.004$ & $0.105\scriptstyle\pm0.006$ & $0.126\scriptstyle\pm0.005$ & $0.083\scriptstyle\pm0.009$ & $0.088\scriptstyle\pm0.014$ & $0.344\scriptstyle\pm0.006$ & $0.153\scriptstyle\pm0.137$ & $0.096\scriptstyle\pm0.002$ & $0.094\scriptstyle\pm0.003$ & $\underline{\textit{0.073}}\scriptstyle\pm0.002$  & $\textbf{0.0619}\scriptstyle\pm0.0001$ \\
        ~ & 192 &$0.109\scriptstyle\pm0.011$ & $0.104\scriptstyle\pm0.014$ & $0.112\scriptstyle\pm0.104$ & $0.123\scriptstyle\pm0.032$ & $0.093\scriptstyle\pm0.024$ & $0.097\scriptstyle\pm0.009$ & $0.345\scriptstyle\pm0.006$ & $0.200\scriptstyle\pm0.094$ & $0.100\scriptstyle\pm0.004$ & $0.097\scriptstyle\pm0.002$ & $\underline{\textit{0.080}}\scriptstyle\pm0.004$  & $\textbf{0.0687}\scriptstyle\pm0.0001$ \\
        ~ & 336 &$0.121\scriptstyle\pm0.011$& $0.104\scriptstyle\pm0.010$ & $0.111\scriptstyle\pm0.014$ & $0.131\scriptstyle\pm0.024$ & $0.095\scriptstyle\pm0.001$ & - & $0.462\scriptstyle\pm0.054$ & - & $0.102\scriptstyle\pm0.007$ & $0.099\scriptstyle\pm0.001$ & $\textbf{0.054}\scriptstyle\pm0.001$   & $\underline{\textit{0.0742}}\scriptstyle\pm0.0001$\\
        ~ & 720 &$0.113\scriptstyle\pm0.018$ & $0.109\scriptstyle\pm0.044$ &  $0.115\scriptstyle\pm0.024$ & $0.127\scriptstyle\pm0.015$ & $0.106\scriptstyle\pm0.007$ & - & $0.478\scriptstyle\pm0.005$ & - & $0.108\scriptstyle\pm0.003$ & $0.114\scriptstyle\pm0.013$ & $\textbf{0.057}\scriptstyle\pm0.005$   & $\underline{\textit{0.0807}}\scriptstyle\pm0.0001$\\ \midrule

        \multirow{4}{*}{Traffic-L} & 96 &$0.297\scriptstyle\pm0.019$& $0.256\scriptstyle\pm0.004$ & $0.258\scriptstyle\pm0.004$ & $0.194\scriptstyle\pm0.002$ & $0.215\scriptstyle\pm0.003$ & $0.208\scriptstyle\pm0.004$ & $0.294\scriptstyle\pm0.003$ & - & $0.202\scriptstyle\pm0.004$ & $\underline{\textit{0.187}}\scriptstyle\pm0.002$ & $\textbf{0.086}\scriptstyle\pm0.001$   & $0.1925\scriptstyle\pm0.0002$\\
        ~ & 192 &$0.308\scriptstyle\pm0.009$& $0.250\scriptstyle\pm0.002$ & $0.275\scriptstyle\pm0.003$ & $0.198\scriptstyle\pm0.004$ & - & - & $0.306\scriptstyle\pm0.004$ & - & $0.208\scriptstyle\pm0.003 $& $\underline{\textit{0.192}}\scriptstyle\pm0.001$ & $\textbf{0.088}\scriptstyle\pm0.002$   & $0.1961\scriptstyle\pm0.0003$\\
        ~ & 336 &$0.334\scriptstyle\pm0.017$ & $0.261\scriptstyle\pm0.001$ & $0.327\scriptstyle\pm0.001$ & $0.204\scriptstyle\pm0.002$ & - & - & $0.317\scriptstyle\pm0.006$ & - & $0.213\scriptstyle\pm0.003$ & $0.201\scriptstyle\pm0.004$ & $\underline{\textit{0.195}}\scriptstyle\pm0.003$  & $\textbf{0.1936}\scriptstyle\pm0.0003$ \\
        ~ & 720 &$0.358\scriptstyle\pm0.022$& $0.284\scriptstyle\pm0.004$ & $0.374\scriptstyle\pm0.004$ & $0.214\scriptstyle\pm0.001$ & - & - & $0.391\scriptstyle\pm0.002$ & - & $0.220\scriptstyle\pm0.002$ & $0.211\scriptstyle\pm0.004$ & $\textbf{0.200}\scriptstyle\pm0.001$  & $\underline{\textit{0.2027}}\scriptstyle\pm0.0001$ \\ \midrule

        \multirow{4}{*}{Weather-L} & 96 &$0.132\scriptstyle\pm0.008$& $0.131\scriptstyle\pm0.011$ & $0.210\scriptstyle\pm0.013$ & $0.131\scriptstyle\pm0.007$ & $0.139\scriptstyle\pm0.008$ & $0.105\scriptstyle\pm0.011$ & $0.104\scriptstyle\pm0.020$ & $\textbf{0.068}\scriptstyle\pm0.008$ & $0.130\scriptstyle\pm0.017$ & $0.116\scriptstyle\pm0.013 $& $0.080\scriptstyle\pm0.007$   & $\underline{\textit{0.0693}}\scriptstyle\pm0.0010$\\
        ~ & 192 &$0.133\scriptstyle\pm0.017$& $0.132\scriptstyle\pm0.018$ & $0.205\scriptstyle\pm0.019$ & $0.131\scriptstyle\pm0.014$ & $0.143\scriptstyle\pm0.020$ & $0.142\scriptstyle\pm0.022$ & $0.134\scriptstyle\pm0.012$ & $\textbf{0.068}\scriptstyle\pm0.006$ & $0.127\scriptstyle\pm0.019$ & $0.122\scriptstyle\pm0.021$ & $0.079\scriptstyle\pm0.009$  & $\underline{\textit{0.0752}}\scriptstyle\pm0.0012$ \\
        ~ & 336 &$0.136\scriptstyle\pm0.021$& $0.132\scriptstyle\pm0.010$ & $0.221\scriptstyle\pm0.005$ & $0.137\scriptstyle\pm0.008$ & $0.129\scriptstyle\pm0.012$ & $0.133\scriptstyle\pm0.014$ & $0.137\scriptstyle\pm0.010$ & $0.083\scriptstyle\pm0.002$ & $0.130\scriptstyle\pm0.006$ & $0.128\scriptstyle\pm0.011$ &$\underline{\textit{0.082}}\scriptstyle\pm0.010$  & $\textbf{0.0781}\scriptstyle\pm0.0017$ \\
        ~ & 720 &$0.140\scriptstyle\pm0.007$ & $0.133\scriptstyle\pm0.004$ & $0.267\scriptstyle\pm0.003$ & $0.142\scriptstyle\pm0.005$ & $0.122\scriptstyle\pm0.006$ & $0.113\scriptstyle\pm0.004$ & $0.152\scriptstyle\pm0.003$ & $0.087\scriptstyle\pm0.003$ & $0.113\scriptstyle\pm0.011$ & $0.110\scriptstyle\pm0.004$ & $\underline{\textit{0.084}}\scriptstyle\pm0.003$  & $\textbf{0.0792}\scriptstyle\pm0.0007$ \\ \midrule

        \multirow{4}{*}{Exchange-L} & 96 &$0.063\scriptstyle\pm0.006$& $0.061\scriptstyle\pm0.003$ & $0.048\scriptstyle\pm0.004$ & $0.063\scriptstyle\pm0.006$ & $0.026\scriptstyle\pm0.010$ & $0.028\scriptstyle\pm0.002$ &$0.079\scriptstyle\pm0.007$ & $\underline{\textit{0.028}}\scriptstyle\pm0.003$ & $0.068\scriptstyle\pm0.003$ & $0.071\scriptstyle\pm0.006$ & $0.031\scriptstyle\pm0.002$  & $\textbf{0.0198}\scriptstyle\pm0.0001$ \\
        ~ & 192 &$0.065\scriptstyle\pm0.020$ & $0.062\scriptstyle\pm0.010$ & $0.049\scriptstyle\pm0.011$ & $0.067\scriptstyle\pm0.008$ & $0.034\scriptstyle\pm0.009$ & $0.046\scriptstyle\pm0.017$ & $0.093\scriptstyle\pm0.011$ & $0.045\scriptstyle\pm0.003$ & $0.087\scriptstyle\pm0.013$ & $0.068\scriptstyle\pm0.004$ & $\underline{\textit{0.032}}\scriptstyle\pm0.010$  & $\textbf{0.0265}\scriptstyle\pm0.0001$ \\
        ~ & 336 &$0.072\scriptstyle\pm0.008$& $0.067\scriptstyle\pm0.008$ & $0.052\scriptstyle\pm0.013$ & $0.071\scriptstyle\pm0.017$ & $0.058\scriptstyle\pm0.023$ & $\underline{\textit{0.045}}\scriptstyle\pm0.010$ & $0.081\scriptstyle\pm0.007$ & $0.060\scriptstyle\pm0.004$ & $0.074\scriptstyle\pm0.009$ & $0.072\scriptstyle\pm0.002$ & $0.048\scriptstyle\pm0.004$  & $\textbf{0.0382}\scriptstyle\pm0.0001$ \\
        ~ & 720 &$0.091\scriptstyle\pm0.012$& $0.087\scriptstyle\pm0.023$ & $0.074\scriptstyle\pm0.011$ & $0.097\scriptstyle\pm0.007$ & $0.160\scriptstyle\pm0.019$ & $0.148\scriptstyle\pm0.017$ & $0.082\scriptstyle\pm0.010$ & $0.143\scriptstyle\pm0.020$ & $0.099\scriptstyle\pm0.015$ & $0.079\scriptstyle\pm0.009$ & $\underline{\textit{0.069}}\scriptstyle\pm0.005$  & $\textbf{0.0591}\scriptstyle\pm0.0002$ \\ \midrule

        \multirow{4}{*}{ILI-L} & 24 &$0.245\scriptstyle\pm0.018$& $0.212 \scriptstyle{\pm 0.013}$ & $0.233 \scriptstyle{\pm 0.015}$ & $0.312 \scriptstyle{\pm 0.014}$ & $0.097 \scriptstyle{\pm 0.010}$ & $0.092 \scriptstyle{\pm 0.019}$ & $0.228 \scriptstyle{\pm 0.024}$ & $0.250\scriptstyle\pm0.013$ & $0.275\scriptstyle\pm0.047$ & $0.257\scriptstyle\pm0.003$ & $\underline{\textit{0.087}}\scriptstyle\pm0.003$  & $\textbf{0.0813}\scriptstyle\pm0.0007$ \\
        ~ & 36 &$0.214\scriptstyle\pm0.008$& $0.182 \scriptstyle{\pm 0.016}$ & $0.217 \scriptstyle{\pm 0.023}$ & $0.241 \scriptstyle{\pm 0.021}$ & $0.117 \scriptstyle{\pm 0.017}$ & $0.115 \scriptstyle{\pm 0.011}$ & $0.235 \scriptstyle{\pm 0.010}$ & $0.285\scriptstyle\pm0.010$ & $0.272\scriptstyle\pm0.057$ & $0.281\scriptstyle\pm0.004$ & $\underline{\textit{0.113}} \scriptstyle{\pm 0.005}$  & $\textbf{0.1122}\scriptstyle\pm0.0009$ \\
        ~ & 48 &$0.271\scriptstyle\pm0.021$& $0.213 \scriptstyle{\pm 0.012}$ & $0.185 \scriptstyle{\pm 0.026}$ & $0.242 \scriptstyle{\pm 0.018}$ & $0.128 \scriptstyle{\pm 0.019}$ & $0.133 \scriptstyle{\pm 0.022}$ & $0.265 \scriptstyle{\pm 0.039}$ & $0.285\scriptstyle\pm0.036$ & $0.295\scriptstyle\pm0.033$ & $0.288\scriptstyle\pm0.008$ & $\underline{\textit{0.124}} \scriptstyle{\pm 0.010}$  & $\textbf{0.1073}\scriptstyle\pm0.0008$ \\
        ~ & 60 &$0.228\scriptstyle\pm0.022$& $0.222 \scriptstyle{\pm 0.020}$ &$0.211 \scriptstyle{\pm 0.011}$ & $0.233 \scriptstyle{\pm 0.019}$ & $0.172 \scriptstyle{\pm 0.034}$ & $0.155 \scriptstyle{\pm 0.018}$ & $0.263 \scriptstyle{\pm 0.022}$ & $0.283\scriptstyle\pm0.012$ & $0.295\scriptstyle\pm0.083$ & $0.307\scriptstyle\pm0.005$ & $\underline{\textit{0.142}} \scriptstyle{\pm 0.008}$  & $\textbf{0.1313}\scriptstyle\pm0.0003$ \\ \bottomrule
        \multicolumn{13}{l}{Due to the excessive time and memory consumption, some results are unavailable in our implementation  and denoted as -.}
    \end{tabular}}
    \caption{Results of CRPS ($\textrm{mean}_{\textrm{std}}$) on long-term forecasting scenarios, each containing five independent runs with different seeds. The context length is set to 36 for the ILI-L dataset and 96 for the others. Lower CRPS values indicate better predictions. The means and standard errors are based on 5 independent runs of retraining and evaluation. \textbf{Bold}: the best, \underline{\textit{italics}}: the 2nd best.}
  \label{tab:long_term_fore_CRPS}
  \end{table*}

\begin{table*}
    \centering
    \setlength\tabcolsep{2pt}
    \scriptsize
    \resizebox{\textwidth}{!}{
  \begin{tabular}{c|c|c|c|c|c|c|c|c|c|c|c|c|c}
    \toprule
        Dataset & Horizon & Koopa & iTransformer & FITS & PatchTST & GRU MAF & Trans MAF & TSDiff & CSDI & TimeGrad & GRU NVP & $K^2$VAE & \shortstack{\model{} \\ (SimpleTM)} \\  \midrule
        \multirow{4}{*}{ETTm1-L} & 96 &$0.362\scriptstyle\pm0.022$ & $0.369\scriptstyle\pm0.029$ & $0.349\scriptstyle\pm0.032$ & $0.329\scriptstyle\pm0.100$ & $0.402\scriptstyle\pm0.087$ & $0.456\scriptstyle\pm0.042$ & $0.441\scriptstyle\pm0.021$ & $0.308\scriptstyle\pm0.005$ & $0.645\scriptstyle\pm0.129$ & $0.488\scriptstyle\pm0.058$ & $\underline{\textit{0.284}}\scriptstyle\pm0.011$ & $\textbf{0.2630}\scriptstyle\pm0.0010$ \\
        ~ & 192 &$0.365\scriptstyle\pm0.032$ & $0.384\scriptstyle\pm0.041$ & $0.341\scriptstyle\pm0.032$ & $0.338\scriptstyle\pm0.022$ & $0.476\scriptstyle\pm0.046$ & $0.553\scriptstyle\pm0.012$ & $0.441\scriptstyle\pm0.019$ & $0.377\scriptstyle\pm0.026$ & $0.748\scriptstyle\pm0.084$ & $0.514\scriptstyle\pm0.042$ & $\underline{\textit{0.323}}\scriptstyle\pm0.020$ & $\textbf{0.2950}\scriptstyle\pm0.0020$ \\
        ~ & 336 &$0.364\scriptstyle\pm0.026$ & $0.380\scriptstyle\pm0.020$ & $0.356\scriptstyle\pm0.022$ & $0.344\scriptstyle\pm0.013$ & $0.522\scriptstyle\pm0.019$ & $0.590\scriptstyle\pm0.047$ & $0.571\scriptstyle\pm0.033$ & $0.419\scriptstyle\pm0.042$ & $0.759\scriptstyle\pm0.015$ & $0.630\scriptstyle\pm0.029$ & $\underline{\textit{0.330}}\scriptstyle\pm0.014$ & $\textbf{0.3140}\scriptstyle\pm0.0020$ \\
        ~ & 720 &$0.377\scriptstyle\pm0.037$ & $0.490\scriptstyle\pm0.038$ & $0.406\scriptstyle\pm0.072$ & $0.382\scriptstyle\pm0.066$ & $0.711\scriptstyle\pm0.081$ & $0.822\scriptstyle\pm0.034$ & $0.622\scriptstyle\pm0.045$ & $0.578\scriptstyle\pm0.051$ & $0.793\scriptstyle\pm0.034$ & $0.707\scriptstyle\pm0.050$ & $\underline{\textit{0.373}}\scriptstyle\pm0.032$ & $\textbf{0.3540}\scriptstyle\pm0.0040$ \\ \midrule

        \multirow{4}{*}{ETTm2-L} & 96 &$0.225\scriptstyle\pm0.039$ & $0.221\scriptstyle\pm0.039$ & $0.210\scriptstyle\pm0.040$ & $0.216\scriptstyle\pm0.035$ & $0.212\scriptstyle\pm0.082$ & $0.279\scriptstyle\pm0.031$ & $0.224\scriptstyle\pm0.033$ & $0.146\scriptstyle\pm0.012$ & $0.525\scriptstyle\pm0.047$ & $0.413\scriptstyle\pm0.059$ & $\underline{\textit{0.144}}\scriptstyle\pm0.011$ & $\textbf{0.1350}\scriptstyle\pm0.0010$ \\
        ~ & 192 &$0.233\scriptstyle\pm0.026$ & $0.229\scriptstyle\pm0.031$ & $0.234\scriptstyle\pm0.038$ & $0.215\scriptstyle\pm0.022$ & $0.535\scriptstyle\pm0.029$ & $0.292\scriptstyle\pm0.041$ & $0.316\scriptstyle\pm0.040$ & $0.189\scriptstyle\pm0.012$ & $0.530\scriptstyle\pm0.060$ & $0.427\scriptstyle\pm0.033$ & $\underline{\textit{0.170}}\scriptstyle\pm0.009$ & $\textbf{0.1600}\scriptstyle\pm0.0010$ \\
        ~ & 336 &$0.267\scriptstyle\pm0.023$ & $0.245\scriptstyle\pm0.049$ & $0.276\scriptstyle\pm0.019$ & $0.234\scriptstyle\pm0.024$ & $0.407\scriptstyle\pm0.043$ & $0.309\scriptstyle\pm0.032$ & $0.397\scriptstyle\pm0.051$ & $0.248\scriptstyle\pm0.024$ & $0.566\scriptstyle\pm0.047$ & $0.580\scriptstyle\pm0.169$ & $\underline{\textit{0.187}}\scriptstyle\pm0.021$ & $\textbf{0.1800}\scriptstyle\pm0.0010$ \\
        ~ & 720 &$0.290\scriptstyle\pm0.033$ & $0.385\scriptstyle\pm0.042$ & $0.540\scriptstyle\pm0.052$ & $0.288\scriptstyle\pm0.034$ & $0.355\scriptstyle\pm0.048$ & $0.475\scriptstyle\pm0.029$ & $0.416\scriptstyle\pm0.065$ & $0.306\scriptstyle\pm0.040$ & $0.561\scriptstyle\pm0.044$ & $0.749\scriptstyle\pm0.385$ & $\underline{\textit{0.275}}\scriptstyle\pm0.035$ & $\textbf{0.2060}\scriptstyle\pm0.0010$ \\ \midrule

        \multirow{4}{*}{ETTh1-L} & 96 &$0.407\scriptstyle\pm0.052$ & $0.386\scriptstyle\pm0.092$ & $0.393\scriptstyle\pm0.142$ & $0.407\scriptstyle\pm0.022$ & $0.371\scriptstyle\pm0.034$ & $0.423\scriptstyle\pm0.047$ & $0.510\scriptstyle\pm0.029$ & $0.557\scriptstyle\pm0.022$ & $0.585\scriptstyle\pm0.058$ & $0.481\scriptstyle\pm0.037$ & $\underline{\textit{0.336}}\scriptstyle\pm0.041$ & $\textbf{0.3150}\scriptstyle\pm0.0020$ \\
        ~ & 192 &$0.396\scriptstyle\pm0.022$ & $0.388\scriptstyle\pm0.041$ & $0.406\scriptstyle\pm0.079$ & $0.405\scriptstyle\pm0.088$ & $0.430\scriptstyle\pm0.022$ & $0.451\scriptstyle\pm0.012$ & $0.596\scriptstyle\pm0.056$ & $0.625\scriptstyle\pm0.065$ & $0.680\scriptstyle\pm0.058$ & $0.531\scriptstyle\pm0.018$ & $\underline{\textit{0.372}}\scriptstyle\pm0.023$ & $\textbf{0.3420}\scriptstyle\pm0.0020$ \\
        ~ & 336 &$0.406\scriptstyle\pm0.028$ & $0.415\scriptstyle\pm0.022$ & $0.410\scriptstyle\pm0.063$ & $0.412\scriptstyle\pm0.024$ & $0.462\scriptstyle\pm0.049$ & $0.481\scriptstyle\pm0.041$ & $0.581\scriptstyle\pm0.035$ & $0.574\scriptstyle\pm0.026$ & $0.666\scriptstyle\pm0.047$ & $0.580\scriptstyle\pm0.064$ & $\underline{\textit{0.394}}\scriptstyle\pm0.022$ & $\textbf{0.3610}\scriptstyle\pm0.0020$ \\
        ~ & 720 &$0.412\scriptstyle\pm0.008$ & $0.449\scriptstyle\pm0.022$ & $0.468\scriptstyle\pm0.012$ & $0.428\scriptstyle\pm0.024$ & $0.496\scriptstyle\pm0.019$ & $0.455\scriptstyle\pm0.025$ & $0.657\scriptstyle\pm0.017$ & $0.657\scriptstyle\pm0.014$ & $0.672\scriptstyle\pm0.015$ & $0.643\scriptstyle\pm0.046$ & $\underline{\textit{0.396}}\scriptstyle\pm0.012$ & $\textbf{0.3670}\scriptstyle\pm0.0020$ \\ \midrule

        \multirow{4}{*}{ETTh2-L} & 96 &$0.249\scriptstyle\pm0.015$ & $0.234\scriptstyle\pm0.011$ & $0.243\scriptstyle\pm0.009$ & $0.247\scriptstyle\pm0.028$ & $0.292\scriptstyle\pm0.012$ & $0.345\scriptstyle\pm0.042$ & $0.421\scriptstyle\pm0.033$ & $0.214\scriptstyle\pm0.018$ & $0.448\scriptstyle\pm0.031$ & $0.548\scriptstyle\pm0.158$ & $\underline{\textit{0.189}}\scriptstyle\pm0.010$ & $\textbf{0.1750}\scriptstyle\pm0.0040$ \\
        ~ & 192 &$0.249\scriptstyle\pm0.032$ & $0.247\scriptstyle\pm0.040$ & $0.252\scriptstyle\pm0.022$ & $0.265\scriptstyle\pm0.091$ & $0.376\scriptstyle\pm0.112$ & $0.343\scriptstyle\pm0.044$ & $0.339\scriptstyle\pm0.033$ & $0.294\scriptstyle\pm0.027$ & $0.575\scriptstyle\pm0.089$ & $0.766\scriptstyle\pm0.223$ & $\underline{\textit{0.213}}\scriptstyle\pm0.021$ & $\textbf{0.2030}\scriptstyle\pm0.0010$ \\
        ~ & 336 &$0.274\scriptstyle\pm0.027$ & $0.297\scriptstyle\pm0.029$ & $0.291\scriptstyle\pm0.032$ & $0.314\scriptstyle\pm0.045$ & $0.454\scriptstyle\pm0.057$ & $0.333\scriptstyle\pm0.078$ & $0.427\scriptstyle\pm0.041$ & $0.353\scriptstyle\pm0.028$ & $0.606\scriptstyle\pm0.095$ & $0.942\scriptstyle\pm0.408$ & $\underline{\textit{0.263}}\scriptstyle\pm0.039$ & $\textbf{0.2250}\scriptstyle\pm0.0010$ \\
        ~ & 720 &$0.286\scriptstyle\pm0.042$ & $0.667\scriptstyle\pm0.012$ & $0.401\scriptstyle\pm0.022$ & $0.371\scriptstyle\pm0.021$ & $1.092\scriptstyle\pm0.019$ & $0.412\scriptstyle\pm0.020$ & $0.482\scriptstyle\pm0.022$ & $0.382\scriptstyle\pm0.030$ & $0.550\scriptstyle\pm0.018$ & $0.688\scriptstyle\pm0.161$ & $\underline{\textit{0.278}}\scriptstyle\pm0.020$ & $\textbf{0.2290}\scriptstyle\pm0.0030$ \\ \midrule

        \multirow{4}{*}{Electricity-L} & 96 &$0.146\scriptstyle\pm0.015$ & $0.134\scriptstyle\pm0.002$ & $0.137\scriptstyle\pm0.002$ & $0.168\scriptstyle\pm0.012$ & $0.108\scriptstyle\pm0.009$ & $0.114\scriptstyle\pm0.010$ & $0.441\scriptstyle\pm0.013$ & $0.203\scriptstyle\pm0.189$ & $0.119\scriptstyle\pm0.003$ & $0.118\scriptstyle\pm0.003$ & $\underline{\textit{0.093}}\scriptstyle\pm0.002$ & $\textbf{0.0794}\scriptstyle\pm0.0005$ \\
        ~ & 192 &$0.143\scriptstyle\pm0.023$ & $0.137\scriptstyle\pm0.022$ & $0.143\scriptstyle\pm0.112$ & $0.163\scriptstyle\pm0.032$ & $0.120\scriptstyle\pm0.033$ & $0.131\scriptstyle\pm0.008$ & $0.441\scriptstyle\pm0.005$ & $0.264\scriptstyle\pm0.129$ & $0.124\scriptstyle\pm0.005$ & $0.121\scriptstyle\pm0.003$ & $\underline{\textit{0.102}}\scriptstyle\pm0.010$ & $\textbf{0.0896}\scriptstyle\pm0.0011$ \\
        ~ & 336 &$0.151\scriptstyle\pm0.017$ & $0.136\scriptstyle\pm0.002$ & $0.139\scriptstyle\pm0.002$ & $0.168\scriptstyle\pm0.010$ & $0.122\scriptstyle\pm0.018$ & - & $0.571\scriptstyle\pm0.022$ & - & $0.126\scriptstyle\pm0.008$ & $0.123\scriptstyle\pm0.001$ & $\underline{\textit{0.107}}\scriptstyle\pm0.002$ & $\textbf{0.0966}\scriptstyle\pm0.0007$ \\
        ~ & 720 &$0.149\scriptstyle\pm0.025$ & $0.140\scriptstyle\pm0.009$ & $0.149\scriptstyle\pm0.012$ & $0.164\scriptstyle\pm0.024$ & $0.136\scriptstyle\pm0.098$ & - & $0.622\scriptstyle\pm0.142$ & - & $0.134\scriptstyle\pm0.004$ & $0.144\scriptstyle\pm0.017$ & $\underline{\textit{0.117}}\scriptstyle\pm0.019$ & $\textbf{0.1085}\scriptstyle\pm0.0032$ \\ \midrule

        \multirow{4}{*}{Traffic-L} & 96 &$0.377\scriptstyle\pm0.024$ & $0.332\scriptstyle\pm0.008$ & $0.332\scriptstyle\pm0.007$ & $\textbf{0.228}\scriptstyle\pm0.010$ & $0.274\scriptstyle\pm0.012$ & $0.265\scriptstyle\pm0.007$ & $0.342\scriptstyle\pm0.042$ & - & $0.234\scriptstyle\pm0.006$ & $0.231\scriptstyle\pm0.003$ & $\underline{\textit{0.230}}\scriptstyle\pm0.010$ & $0.2450\scriptstyle\pm0.0010$ \\
        ~ & 192 &$0.388\scriptstyle\pm0.011$ & $0.326\scriptstyle\pm0.009$ & $0.350\scriptstyle\pm0.010$ & $\textbf{0.225}\scriptstyle\pm0.012$ & - & - & $0.354\scriptstyle\pm0.012$ & - & $0.239\scriptstyle\pm0.004$ & $0.236\scriptstyle\pm0.002$ & $\underline{\textit{0.234}}\scriptstyle\pm0.003$ & $0.2470\scriptstyle\pm0.0020$ \\
        ~ & 336 &$0.416\scriptstyle\pm0.028$ & $0.335\scriptstyle\pm0.010$ & $0.405\scriptstyle\pm0.011$ & $\underline{\textit{0.242}}\scriptstyle\pm0.022$ & - & - & $0.392\scriptstyle\pm0.006$ & - & $0.246\scriptstyle\pm0.003$ & $0.248\scriptstyle\pm0.006$ & $\textbf{0.242}\scriptstyle\pm0.007$ & $0.2450\scriptstyle\pm0.0030$ \\
        ~ & 720 &$0.432\scriptstyle\pm0.032$ & $0.361\scriptstyle\pm0.030$ & $0.453\scriptstyle\pm0.022$ & $0.253\scriptstyle\pm0.012$ & - & - & $0.478\scriptstyle\pm0.006$ & - & $0.263\scriptstyle\pm0.001$ & $0.264\scriptstyle\pm0.006$ & $\textbf{0.248}\scriptstyle\pm0.010$ & $\underline{\textit{0.2520}}\scriptstyle\pm0.0020$ \\ \midrule

        \multirow{4}{*}{Weather-L} & 96 &$0.146\scriptstyle\pm0.019$ & $0.144\scriptstyle\pm0.017$ & $0.279\scriptstyle\pm0.027$ & $0.145\scriptstyle\pm0.016$ & $0.176\scriptstyle\pm0.011$ & $0.139\scriptstyle\pm0.010$ & $0.113\scriptstyle\pm0.022$ & $0.087\scriptstyle\pm0.012$ & $0.164\scriptstyle\pm0.023$ & $0.145\scriptstyle\pm0.017$ & $\underline{\textit{0.086}}\scriptstyle\pm0.011$ & $\textbf{0.0850}\scriptstyle\pm0.0010$ \\
        ~ & 192 &$0.148\scriptstyle\pm0.022$ & $0.145\scriptstyle\pm0.015$ & $0.264\scriptstyle\pm0.013$ & $0.144\scriptstyle\pm0.012$ & $0.166\scriptstyle\pm0.022$ & $0.160\scriptstyle\pm0.037$ & $0.144\scriptstyle\pm0.020$ & $\underline{\textit{0.086}}\scriptstyle\pm0.007$ & $0.158\scriptstyle\pm0.024$ & $0.147\scriptstyle\pm0.025$ & $\textbf{0.083}\scriptstyle\pm0.011$ & $0.0880\scriptstyle\pm0.0010$ \\
        ~ & 336 &$0.152\scriptstyle\pm0.032$ & $0.146\scriptstyle\pm0.011$ & $0.283\scriptstyle\pm0.021$ & $0.149\scriptstyle\pm0.023$ & $0.168\scriptstyle\pm0.014$ & $0.170\scriptstyle\pm0.027$ & $0.138\scriptstyle\pm0.033$ & $0.098\scriptstyle\pm0.002$ & $0.162\scriptstyle\pm0.006$ & $0.160\scriptstyle\pm0.012$ & $\underline{\textit{0.093}}\scriptstyle\pm0.010$ & $\textbf{0.0910}\scriptstyle\pm0.0010$ \\
        ~ & 720 &$0.162\scriptstyle\pm0.009$ & $0.147\scriptstyle\pm0.019$ & $0.317\scriptstyle\pm0.021$ & $0.152\scriptstyle\pm0.029$ & $0.149\scriptstyle\pm0.034$ & $0.148\scriptstyle\pm0.040$ & $0.141\scriptstyle\pm0.026$ & $0.102\scriptstyle\pm0.005$ & $0.136\scriptstyle\pm0.020$ & $0.135\scriptstyle\pm0.008$ & $\underline{\textit{0.099}}\scriptstyle\pm0.009$ & $\textbf{0.0960}\scriptstyle\pm0.0010$ \\ \midrule

        \multirow{4}{*}{Exchange-L} & 96 &$0.079\scriptstyle\pm0.005$ & $0.077\scriptstyle\pm0.001$ & $0.069\scriptstyle\pm0.007$ & $0.079\scriptstyle\pm0.002$ & $0.033\scriptstyle\pm0.003$ & $0.036\scriptstyle\pm0.009$ & $0.090\scriptstyle\pm0.010$ & $0.036\scriptstyle\pm0.005$ & $0.079\scriptstyle\pm0.002$ & $0.091\scriptstyle\pm0.009$ & $\underline{\textit{0.032}}\scriptstyle\pm0.002$ & $\textbf{0.0250}\scriptstyle\pm0.0010$ \\
        ~ & 192 &$0.081\scriptstyle\pm0.015$ & $0.078\scriptstyle\pm0.008$ & $0.069\scriptstyle\pm0.007$ & $0.081\scriptstyle\pm0.002$ & $0.044\scriptstyle\pm0.004$ & $0.058\scriptstyle\pm0.007$ & $0.106\scriptstyle\pm0.010$ & $0.058\scriptstyle\pm0.005$ & $0.100\scriptstyle\pm0.019$ & $0.087\scriptstyle\pm0.005$ & $\underline{\textit{0.040}}\scriptstyle\pm0.005$ & $\textbf{0.0360}\scriptstyle\pm0.0020$ \\
        ~ & 336 &$0.086\scriptstyle\pm0.003$ & $0.083\scriptstyle\pm0.005$ & $0.071\scriptstyle\pm0.005$ & $0.085\scriptstyle\pm0.010$ & $0.074\scriptstyle\pm0.017$ & $0.058\scriptstyle\pm0.009$ & $0.106\scriptstyle\pm0.010$ & $0.076\scriptstyle\pm0.006$ & $0.086\scriptstyle\pm0.008$ & $0.091\scriptstyle\pm0.002$ & $\underline{\textit{0.054}}\scriptstyle\pm0.001$ & $\textbf{0.0520}\scriptstyle\pm0.0010$ \\
        ~ & 720 &$0.116\scriptstyle\pm0.022$ & $0.113\scriptstyle\pm0.015$ & $0.097\scriptstyle\pm0.011$ & $0.126\scriptstyle\pm0.001$ & $0.182\scriptstyle\pm0.010$ & $0.191\scriptstyle\pm0.006$ & $0.142\scriptstyle\pm0.009$ & $0.173\scriptstyle\pm0.020$ & $0.113\scriptstyle\pm0.016$ & $0.103\scriptstyle\pm0.009$ & $\underline{\textit{0.084}}\scriptstyle\pm0.017$ & $\textbf{0.0800}\scriptstyle\pm0.0000$ \\ \midrule

        \multirow{4}{*}{ILI-L} & 24 &$0.303\scriptstyle\pm0.021$ & $0.265\scriptstyle\pm0.027$ & $0.271\scriptstyle\pm0.032$ & $0.382\scriptstyle\pm0.018$ & $0.124\scriptstyle\pm0.019$ & $0.118\scriptstyle\pm0.033$ & $0.242\scriptstyle\pm0.086$ & $0.263\scriptstyle\pm0.012$ & $0.296\scriptstyle\pm0.044$ & $0.283\scriptstyle\pm0.001$ & $\underline{\textit{0.116}}\scriptstyle\pm0.011$ & $\textbf{0.1119}\scriptstyle\pm0.0094$ \\
        ~ & 36 &$0.262\scriptstyle\pm0.013$ & $0.222\scriptstyle\pm0.047$ & $0.258\scriptstyle\pm0.058$ & $0.286\scriptstyle\pm0.037$ & $0.144\scriptstyle\pm0.011$ & $0.143\scriptstyle\pm0.089$ & $0.246\scriptstyle\pm0.117$ & $0.298\scriptstyle\pm0.011$ & $0.298\scriptstyle\pm0.048$ & $0.307\scriptstyle\pm0.007$ & $\underline{\textit{0.142}}\scriptstyle\pm0.008$ & $\textbf{0.1140}\scriptstyle\pm0.0104$ \\
        ~ & 48 &$0.334\scriptstyle\pm0.028$ & $0.262\scriptstyle\pm0.023$ & $0.225\scriptstyle\pm0.043$ & $0.291\scriptstyle\pm0.032$ & $0.159\scriptstyle\pm0.020$ & $0.160\scriptstyle\pm0.039$ & $0.275\scriptstyle\pm0.044$ & $0.301\scriptstyle\pm0.034$ & $0.320\scriptstyle\pm0.025$ & $0.314\scriptstyle\pm0.009$ & $\underline{\textit{0.152}}\scriptstyle\pm0.017$ & $\textbf{0.1457}\scriptstyle\pm0.0078$ \\
        ~ & 60 &$0.288\scriptstyle\pm0.031$ & $0.278\scriptstyle\pm0.017$ & $0.245\scriptstyle\pm0.017$ & $0.287\scriptstyle\pm0.023$ & $0.216\scriptstyle\pm0.014$ & $0.183\scriptstyle\pm0.019$ & $0.272\scriptstyle\pm0.020$ & $0.299\scriptstyle\pm0.013$ & $0.325\scriptstyle\pm0.068$ & $0.333\scriptstyle\pm0.005$ & $\underline{\textit{0.167}}\scriptstyle\pm0.007$ & $\textbf{0.1584}\scriptstyle\pm0.0071$ \\ \bottomrule

        \multicolumn{12}{l}{Due to the excessive time and memory consumption, some results are unavailable in our implementation  and denoted as -.}
  \end{tabular}}
  \caption{Results of NMAE ($\textrm{mean}_{\textrm{std}}$) on long-term forecasting scenarios, each containing five independent runs with different seeds. The context length is set to 36 for the ILI-L dataset and 96 for the others. Lower NMAE values indicate better predictions. The means and standard errors are based on 5 independent runs of retraining and evaluation. \textbf{Bold}: the best, \underline{\textit{italics}}: the 2nd best.}
  \label{tab:long_term_fore_NMAE}
  \end{table*}

\begin{table*}[t!]
\centering
\resizebox{\textwidth}{!}{\begin{tabular}{cl | cccccc | cccccc}
\toprule
 & & \multicolumn{6}{c}{EnergyScore} & \multicolumn{6}{c}{CRPS} \\
\cmidrule(lr){3-8} \cmidrule(lr){9-14}
 & & \multicolumn{2}{c}{ETTm1} & \multicolumn{3}{c}{Exchange} & \multicolumn{1}{c}{Solar} & \multicolumn{2}{c}{ETTm1} & \multicolumn{3}{c}{Exchange} & \multicolumn{1}{c}{Solar} \\
\cmidrule(lr){3-4} \cmidrule(lr){5-7} \cmidrule(lr){8-8} \cmidrule(lr){9-10} \cmidrule(lr){11-13} \cmidrule(lr){14-14}
Type & Methods & 96 & 336 & 96 & 336 & 30 & 24 & 96 & 336 & 96 & 336 & 30 & 24 \\
\midrule
\multirow{2}{*}{\rotatebox{90}{}}
& MVE-1S & 0.6020{\scriptsize$\pm$0.0144} & 0.3832{\scriptsize$\pm$0.0078} & 0.0064{\scriptsize$\pm$0.0001} & \underline{\textit{0.0065}}{\scriptsize$\pm$0.0000} & \underline{\textit{0.0047}}{\scriptsize$\pm$0.0001} & 59.2827{\scriptsize$\pm$3.4358} & 0.2430{\scriptsize$\pm$0.0050} & 0.3301{\scriptsize$\pm$0.0066} & 0.0210{\scriptsize$\pm$0.0010} & 0.0399{\scriptsize$\pm$0.0006} & 0.0077{\scriptsize$\pm$0.0002} & 0.4774{\scriptsize$\pm$0.0156} \\
& K2VAE & 0.4990{\scriptsize$\pm$0.0103} & 0.3079{\scriptsize$\pm$0.0014} & 0.0099{\scriptsize$\pm$0.0009} & 0.0082{\scriptsize$\pm$0.0009} & 0.0066{\scriptsize$\pm$0.0002} & 49.0566{\scriptsize$\pm$0.9311} & 0.2359{\scriptsize$\pm$0.0072} & 0.2698{\scriptsize$\pm$0.0011} & 0.0322{\scriptsize$\pm$0.0027} & 0.0519{\scriptsize$\pm$0.0043} & 0.0106{\scriptsize$\pm$0.0004} & \textbf{0.4129}{\scriptsize$\pm$0.0167} \\
\midrule
\multirow{6}{*}{\rotatebox{90}{2-stage}}
& MVE-2S & 0.4697{\scriptsize$\pm$0.0021} & 0.2997{\scriptsize$\pm$0.0015} & 0.0063{\scriptsize$\pm$0.0001} & \textbf{0.0064}{\scriptsize$\pm$0.0000} & 0.0048{\scriptsize$\pm$0.0000} & 48.6663{\scriptsize$\pm$0.2327} & 0.2117{\scriptsize$\pm$0.0009} & 0.2491{\scriptsize$\pm$0.0008} & 0.0204{\scriptsize$\pm$0.0000} & 0.0391{\scriptsize$\pm$0.0001} & \textbf{0.0074}{\scriptsize$\pm$0.0001} & 0.4705{\scriptsize$\pm$0.0035} \\
& TORF & 0.4651{\scriptsize$\pm$0.0010} & \underline{\textit{0.2933}}{\scriptsize$\pm$0.0002} & \textbf{0.0061}{\scriptsize$\pm$0.0000} & \textbf{0.0064}{\scriptsize$\pm$0.0000} & \underline{\textit{0.0047}}{\scriptsize$\pm$0.0000} & 48.2329{\scriptsize$\pm$0.0590} & 0.2088{\scriptsize$\pm$0.0011} & 0.2457{\scriptsize$\pm$0.0005} & \textbf{0.0198}{\scriptsize$\pm$0.0001} & \textbf{0.0382}{\scriptsize$\pm$0.0001} & 0.0077{\scriptsize$\pm$0.0000} & \underline{\textit{0.4558}}{\scriptsize$\pm$0.0020} \\
& \quad w RealNVP & 0.4674{\scriptsize$\pm$0.0008} & 0.2956{\scriptsize$\pm$0.0008} & 0.0063{\scriptsize$\pm$0.0002} & 0.0066{\scriptsize$\pm$0.0000} & \textbf{0.0046}{\scriptsize$\pm$0.0002} & 48.4335{\scriptsize$\pm$0.1224} & 0.2101{\scriptsize$\pm$0.0004} & 0.2517{\scriptsize$\pm$0.0036} & 0.0210{\scriptsize$\pm$0.0007} & 0.0412{\scriptsize$\pm$0.0001} & 0.0077{\scriptsize$\pm$0.0002} & 0.4774{\scriptsize$\pm$0.0156} \\
& \quad w Spline & \underline{\textit{0.4587}}{\scriptsize$\pm$0.0028} & 0.2938{\scriptsize$\pm$0.0006} & 0.0065{\scriptsize$\pm$0.0002} & \underline{\textit{0.0065}}{\scriptsize$\pm$0.0001} & \underline{\textit{0.0047}}{\scriptsize$\pm$0.0000} & 48.2329{\scriptsize$\pm$0.0590} & \underline{\textit{0.2051}}{\scriptsize$\pm$0.0009} & 0.2452{\scriptsize$\pm$0.0010} & 0.0213{\scriptsize$\pm$0.0007} & 0.0394{\scriptsize$\pm$0.0007} & 0.0078{\scriptsize$\pm$0.0001} & \underline{\textit{0.4558}}{\scriptsize$\pm$0.0020} \\
& \quad w OddRealNVP & 0.4656{\scriptsize$\pm$0.0018} & 0.2946{\scriptsize$\pm$0.0003} & \underline{\textit{0.0062}}{\scriptsize$\pm$0.0000} & \textbf{0.0064}{\scriptsize$\pm$0.0000} & \textbf{0.0046}{\scriptsize$\pm$0.0001} & \underline{\textit{48.1577}}{\scriptsize$\pm$0.1690} & 0.2086{\scriptsize$\pm$0.0004} & \underline{\textit{0.2450}}{\scriptsize$\pm$0.0009} & \underline{\textit{0.0203}}{\scriptsize$\pm$0.0001} & 0.0397{\scriptsize$\pm$0.0000} & \textbf{0.0074}{\scriptsize$\pm$0.0001} & 0.4588{\scriptsize$\pm$0.0029} \\
& \quad w Mixture & \textbf{0.4559}{\scriptsize$\pm$0.0007} & \textbf{0.2930}{\scriptsize$\pm$0.0006} & \underline{\textit{0.0062}}{\scriptsize$\pm$0.0001} & \textbf{0.0064}{\scriptsize$\pm$0.0001} & 0.0048{\scriptsize$\pm$0.0001} & \textbf{48.0786}{\scriptsize$\pm$0.0729} & \textbf{0.2038}{\scriptsize$\pm$0.0005} & \textbf{0.2436}{\scriptsize$\pm$0.0004} & 0.0208{\scriptsize$\pm$0.0002} & \underline{\textit{0.0390}}{\scriptsize$\pm$0.0000} & \underline{\textit{0.0076}}{\scriptsize$\pm$0.0001} & 0.4653{\scriptsize$\pm$0.0038} \\
\bottomrule
\end{tabular}}
\caption{Energy Score and CRPS results. \textbf{Bold} = best (lowest); \underline{\textit{italic}} = second best.}
\label{tab:energyscore_crps:appendix}
\end{table*}

\begin{table}[t]
\centering
\resizebox{\columnwidth}{!}{\begin{tabular}{l c c cc}
\toprule
& ETTm1 & \multicolumn{2}{c}{Exchange} & {Solar} \\
\cmidrule(lr){2-2} \cmidrule(lr){3-4} \cmidrule(lr){5-5}
Method & 336 & 336 & 30 & 24 \\
\midrule
TORF-1S & 0.5122{\scriptsize$\pm$0.0079} & 0.1390{\scriptsize$\pm$0.0032} & 0.1430{\scriptsize$\pm$0.0005} & 0.6078{\scriptsize$\pm$0.0269} \\
TORF & \textbf{0.2457}{\scriptsize$\pm$0.0005} & \textbf{0.0382}{\scriptsize$\pm$0.0001} & \textbf{0.0077}{\scriptsize$\pm$0.0000} & \textbf{0.4558}{\scriptsize$\pm$0.0020} \\
\quad w/o \scl{} & 0.2508{\scriptsize$\pm$0.0015} & 0.0399{\scriptsize$\pm$0.0018} & 0.0081{\scriptsize$\pm$0.0003} & 0.4670{\scriptsize$\pm$0.0031} \\
\quad w/o CNN-ROSS & \underline{\textit{0.2483}}{\scriptsize$\pm$0.0004} & \underline{\textit{0.0383}}{\scriptsize$\pm$0.0000} & \underline{\textit{0.0078}}{\scriptsize$\pm$0.0001} & \textbf{0.4558}{\scriptsize$\pm$0.0020} \\
\bottomrule
\end{tabular}}
\caption{CRPS results with standard deviation over 5 runs.}
\label{tab:ablation:appendix}
\end{table}

\begin{table*}[t]
\centering
\resizebox{\textwidth}{!}{\begin{tabular}{l l r c c r p{7cm}}
\toprule
Horizon & Dataset & \#var. & Range & Freq. & Timesteps & Description \\
\midrule
\multirow{7}{*}{Long-term}
& ETTh1-L / ETTh2-L & 7 & $\mathbb{R}^+$ & H & 17{,}420 & Electricity transformer temperature, recorded hourly \\
& ETTm1-L / ETTm2-L & 7 & $\mathbb{R}^+$ & 15min & 69{,}680 & Electricity transformer temperature, recorded every 15 minutes \\
& Electricity-L & 321 & $\mathbb{R}^+$ & H & 26{,}304 & Household electricity consumption (kWh) \\
& Traffic-L & 862 & $(0,1)$ & H & 17{,}544 & Road occupancy rates \\
& Exchange-L & 8 & $\mathbb{R}^+$ & Business day & 7{,}588 & Daily exchange rates for 8 countries \\
& ILI-L & 7 & $(0,1)$ & W & 966 & Fraction of patients presenting influenza-like illness \\
& Weather-L & 21 & $\mathbb{R}^+$ & 10min & 52{,}696 & Local climatological measurements \\
\midrule
\multirow{6}{*}{Short-term}
& ETTh1-S / ETTh2-S & 7 & $\mathbb{R}^+$ & H & 17{,}420 & Electricity transformer temperature, recorded hourly \\
& ETTm1-S / ETTm2-S & 7 & $\mathbb{R}^+$ & 15min & 69{,}680 & Electricity transformer temperature, recorded every 15 minutes \\
& Exchange-S & 8 & $\mathbb{R}^+$ & Business day & 6{,}071 & Daily exchange rates for 8 countries \\
& Solar-S & 137 & $\mathbb{R}^+$ & H & 7{,}009 & Solar power production records \\
& Electricity-S & 370 & $\mathbb{R}^+$ & H & 5{,}833 & Household electricity consumption \\
& Traffic-S & 963 & $(0,1)$ & H & 4{,}001 & Road occupancy rates \\
\bottomrule
\end{tabular}}
\caption{Statistical information of the datasets.}
\label{tab:dataset_stats}
\end{table*}

\begin{table}[t]
\centering
\resizebox{\columnwidth}{!}{\begin{tabular}{cl cccc}
\toprule
 & & ETTm1 & \multicolumn{2}{c}{Exchange} & Solar \\
\cmidrule(lr){3-3} \cmidrule(lr){4-5} \cmidrule(lr){6-6}
Type & Methods & 336 & 336 & 30 & 24 \\
\midrule
\multirow{2}{*}{\rotatebox{90}{1-stage}}
& TORF-1S & 0.512{\scriptsize$\pm$0.008} & 0.169{\scriptsize$\pm$0.003} & 0.188{\scriptsize$\pm$0.000} & 0.834{\scriptsize$\pm$0.008} \\
& MVE-1S & 0.349{\scriptsize$\pm$0.006} & 0.056{\scriptsize$\pm$0.001} & 0.011{\scriptsize$\pm$0.000} & 0.599{\scriptsize$\pm$0.011} \\
& K2VAE & 0.339{\scriptsize$\pm$0.002} & 0.055{\scriptsize$\pm$0.002} & 0.011{\scriptsize$\pm$0.000} & \textbf{0.531}{\scriptsize$\pm$0.018} \\
\midrule
\multirow{4}{*}{\rotatebox{90}{2-stage}}
& MVE-2S & \textbf{0.314}{\scriptsize$\pm$0.002} & \textbf{0.051}{\scriptsize$\pm$0.000} & \textbf{0.010}{\scriptsize$\pm$0.000} & \underline{\textit{0.583}}{\scriptsize$\pm$0.025} \\
& Odd TORF & \textbf{0.314}{\scriptsize$\pm$0.002} & \textbf{0.051}{\scriptsize$\pm$0.000} & \textbf{0.010}{\scriptsize$\pm$0.000} & \underline{\textit{0.583}}{\scriptsize$\pm$0.025} \\
& \quad w RealNVP & \underline{\textit{0.315}}{\scriptsize$\pm$0.001} & \underline{\textit{0.052}}{\scriptsize$\pm$0.001} & \textbf{0.010}{\scriptsize$\pm$0.000} & 0.630{\scriptsize$\pm$0.001} \\
& \quad w Spline & 0.319{\scriptsize$\pm$0.001} & 0.052{\scriptsize$\pm$0.000} & \textbf{0.010}{\scriptsize$\pm$0.000} & 0.629{\scriptsize$\pm$0.002} \\
\bottomrule
\end{tabular}}
\caption{NMAE results with standard deviation over 5 runs.}
\label{tab:nmae:appendix}
\end{table}

\section{All Baselines - Main results}
\label{sec:mainResults_allBaselines}

\Cref{tab: long-term:appendix} and \Cref{tab: short-term:appendix} report the full probabilistic forecasting results on all the 12 baselines. The best result in each column is shown in \textbf{bold} and the second-best is shown in \underline{\textit{italics}}. We highlight the following results below.

In the long-term setting, \model{}\ wins on CRPS for 7 of 9 datasets and on NMAE for all 9 datasets. The largest gains reach $+35.4\%$ on CRPS (ETTh2-L) and $+25.1\%$ on NMAE (ETTm2-L). In the short-term setting, \model{}\ wins on CRPS for 6 of 8 datasets and on NMAE for 5 of 8 datasets. The largest gains reach $+30.3\%$ on CRPS (ETTm2-S) and $+28.1\%$ on NMAE (ETTm2-S). This reflects the advantage of the two-stage design. \model{}\ inherits an accurate, analytically exact mean from Stage~1. It then models a flexible residual distribution around that same mean, without resorting to expensive sampling.

Looking beyond $K^2$VAE, \model{}\ also dominates every other baseline in
\Cref{tab: long-term:appendix,tab: short-term:appendix}. In the long-term
setting, \model{}\ beats every non-$K^2$VAE baseline on both CRPS and NMAE
for all 9 datasets, a clean 9/9 sweep. No single alternative is a
consistent runner-up across datasets, either: Koopa is the closest
competitor on three of the four ETT variants (ETTm1-L, ETTh1-L, ETTh2-L),
PatchTST on ETTm2-L, GRU~MAF on Electricity-L, GRU~NVP on Traffic-L etc. Each
general-purpose flow, diffusion, or point-forecaster-turned-probabilistic
baseline is only competitive on a narrow subset of dataset types, whereas
\model{}\ is uniformly strong across all of them.

In the short-term setting, against whichever baseline is actually
strongest in each column (not always $K^2$VAE), \model{}\ still wins 6 of 8
columns on CRPS and 5 of 8 on NMAE, losing only on Solar-S (CRPS and
NMAE), Electricity-S (CRPS and NMAE), and Traffic-S NMAE alone by a
$2.6\%$ margin (\model{}\ still wins Traffic-S CRPS by $2.3\%$). Notably,
on Electricity-S the strongest baseline overall is CSDI ($0.051$ CRPS,
$0.066$ NMAE), not $K^2$VAE ($0.053$, $0.068$, the latter exactly tied
with \model{}\ on CRPS); the reported improvement percentages already
reflect this, since they are computed against whichever baseline is
strongest per column. CSDI, a diffusion-based method, is the strongest
baseline specifically on Electricity-S, but is not competitive elsewhere
in either table, underscoring the same pattern as the long-term setting.
Individual alternative methods win narrow, dataset-specific pockets, while
\model{}\ is competitive across the full spread of
short- and long-horizon, univariate and highly multivariate datasets.

\section{All Horizon Results}
\label{sec:ablation:allhorizonresults}

We extend the analysis to include~\Cref{tab:long_term_fore_CRPS} and~\Cref{tab:long_term_fore_NMAE}, where we compare the performance of \model{}
across all the datasets, horizons and baselines for long-horizon forecasting.

Across all four horizons, \model{}'s advantage is stable rather than an
artifact of the $H{=}720$ headline comparison: on the four ETT variants,
Exchange-L, and ILI-L, \model{}\ is the best or second-best CRPS at every
horizon, and the outright best NMAE at every horizon without exception
and on several of these, its margin over $K^2$VAE actually \emph{widens}
with horizon rather than shrinking (e.g.\ ETTh2-L). This reflects a
structural advantage of \model{}'s two-stage design: SimpleTM is trained
purely to produce the best point forecast as Stage~1, whereas the
probabilistic baselines' point predictions are only a byproduct of a
density-estimation objective they were never trained to optimize.

The two datasets where \model{}\ actually struggles show clear,
horizon-dependent patterns. On \textbf{Traffic-L}, \model{}\ trails both
$K^2$VAE and GRU~NVP on CRPS at the shorter horizons ($H{=}96,192$; NMAE
ranks it 3rd there too), but its ranking improves steadily with horizon:
it overtakes $K^2$VAE outright at $H{=}336$ and is a close 2nd by
$H{=}720$. On \textbf{Weather-L}, CSDI, and not $K^2$VAE is the CRPS
leader at $H{=}96,192$, but \model{}\ overtakes it by $H{=}336,720$. NMAE
follows almost the same pattern, with a single-horizon dip to 3rd at
$H{=}192$. \textbf{Electricity-L} runs the other way on CRPS: \model{}\
wins clearly at $H{=}96,192$ but trails $K^2$VAE at $H{=}336,720$,
even though \model{}'s NMAE stays best across all four Electricity-L
horizons.

\section{Results with Standard Deviations}

\subsubsection{Probabilistic Analysis}

We extend the ablation with a shorter horizon ($H{=}96$) for ETTm1 and
Exchange, in addition to the existing horizon-336 settings, comparing
\model{}\ and its variants against $K^2$VAE~\citep{wu2025kvae} and the
Gaussian MVE-1S/MVE-2S baselines described earlier.
\Cref{tab:energyscore_crps:appendix} reports Energy Score, which measures
joint distribution quality, and CRPS, which measures per-channel marginal
quality, together with standard deviations across the 5 independent runs.

\model{}\ and its variants are state-of-the-art on 11 of the 12 columns.
The one exception is Solar-24 CRPS, where $K^2$VAE remains best,
consistent with the Stage-1 SimpleTM limitation on Solar-S already noted
in \Cref{tab: short-term:appendix}.

Among the 1-stage baselines, MVE-1S beats $K^2$VAE on Exchange on both
Energy Score and CRPS, even though their NMAE is comparable
(\Cref{tab:nmae:appendix}), so the gap cannot come from point-forecast
accuracy; it suggests Exchange's residual is close to Gaussian, a shape
$K^2$VAE's more complex structure fails to exploit. MVE-2S already beats
$K^2$VAE on 11 of the 12 columns, losing only on the same Solar-24 CRPS
column every method loses on. Comparing MVE-1S against MVE-2S shows that
training the mean and the distributional head sequentially, rather than
jointly, avoids the optimization pitfalls of combined NLL training,
confirming the two-stage approach is the better training strategy even
before any flow is introduced.

Among the \model{}\ variants, all are state-of-the-art on 11 of 12 tasks,
losing only the same Solar-24 CRPS column. \model{}\ and TORF-Mixture are
the strongest, winning 4/12 and 6/12 columns respectively. $K^2$VAE wins
the univariate Solar-24 CRPS, but every \model{}\ variant still beats it
on the corresponding multivariate Solar-24 Energy Score, evidence of
better joint distributional modeling even where the marginal residual is
harder to fit.

TORF-Mixture wins the most columns overall, but its margin over \model{}\
is mostly insignificant except on Solar-24 and ETTm1-96, and, as
\Cref{tab:resource_trimmed} shows, it costs substantially more
compute. We therefore propose plain \model{}\ as the default, with
TORF-Mixture as the better choice where the infrastructure allows, since
it can model asymmetric residual distributions that a single odd spline
cannot.

Among the non-odd, non-mixture alternatives, TORF-w/Spline wins more
columns than TORF-w/RealNVP. TORF-w/OddRealNVP is a compelling baseline on
its own, winning outright on 3/12 columns using only an affine transform,
 reaffirming the odd residual flow as a strong, resource-efficient
choice for distribution modeling.

\begin{table}[t]
\centering
\resizebox{\columnwidth}{!}{\begin{tabular}{ll cc cc}
\toprule
& & \multicolumn{2}{c}{NMAE} & \multicolumn{2}{c}{CRPS} \\
\cmidrule(lr){3-4} \cmidrule(lr){5-6}
Dataset & Horizon & iTrans & SimpleTM & iTrans & SimpleTM \\
\midrule
\multirow{2}{*}{ETTm1} & 96 & 0.2762{\scriptsize$\pm$0.0048} & \textbf{0.263}{\scriptsize$\pm$0.001} & 0.2209{\scriptsize$\pm$0.0011} & \textbf{0.2088}{\scriptsize$\pm$0.0011} \\
& 336 & 0.3211{\scriptsize$\pm$0.0042} & \textbf{0.314}{\scriptsize$\pm$0.002} & 0.2581{\scriptsize$\pm$0.0003} & \textbf{0.2457}{\scriptsize$\pm$0.0005} \\
\multirow{2}{*}{Exchange} & 96 & \textbf{0.0241}{\scriptsize$\pm$0.0012} & 0.025{\scriptsize$\pm$0.001} & \textbf{0.0186}{\scriptsize$\pm$0.0001} & 0.0198{\scriptsize$\pm$0.0001} \\
& 336 & \textbf{0.0449}{\scriptsize$\pm$0.0007} & 0.052{\scriptsize$\pm$0.001} & \textbf{0.03532}{\scriptsize$\pm$0.0001} & 0.0382{\scriptsize$\pm$0.0001} \\
\bottomrule
\end{tabular}}
\caption{First stage method comparison across datasets and horizons (NMAE and CRPS, mean with std in scriptsize). Best per column in \textbf{bold}.}
\label{tab:backbone_transfer}
\end{table}

\section{Backbone Transfer Analysis}
\label{sec:appendix:stage1Transfer}
We test whether \model{}'s flow hyperparameters need to be re-tuned whenever the Stage-1 point forecaster changes. We take the flow configuration
(including the number of \ross\ blocks $K$) tuned for \model{}\ with SimpleTM as Stage-1, and reuse it unchanged with iTransformer
substituted as Stage-1, without any additional hyperparameter search for the flow. \Cref{tab:backbone_transfer} reports both point-forecast
accuracy (NMAE) and probabilistic accuracy (CRPS) for both backbones on ETTm1 and Exchange.

The pattern is consistent across all four dataset-horizon combinations. Whichever backbone is the better point forecaster is also the one that
yields the better CRPS once plugged into the unchanged flow. On ETTm1, SimpleTM has the lower NMAE at both horizons and also the lower CRPS. On Exchange the ranking flips. iTransformer is the better point forecaster and correspondingly achieves the better CRPS.
Simply swapping in whichever backbone is the stronger point forecaster for the dataset at hand, without re-tuning \model{}'s flow hyperparameters, is enough to improve probabilistic performance.

This indicates that \model{}'s flow hyperparameters transfer across different pretrained Stage-1 models without requiring their own expensive re-tuning.
Probabilistic performance tracks the quality of whichever point forecaster is plugged in, rather than being tied to the specific backbone the flow
happened to be tuned on.

\begin{table}[t]
\centering
\footnotesize
\setlength{\tabcolsep}{1mm}
\resizebox{\columnwidth}{!}{\begin{tabular}{l c c c c c c c c}
\toprule
Dataset & $H$ & LR & Weight decay & $C_{\mathrm{hid}}$ & $K$ & $N_b$ & $k$ & $k_f$ \\
\midrule
ETTh1-S & 24 & $3.44\times10^{-4}$ & $2.42\times10^{-4}$ & 32 & 8 & 16 & 4 & 8 \\
ETTh2-S & 24 & $5.28\times10^{-5}$ & $4.86\times10^{-4}$ & 512 & 8 & 16 & 4 & 8 \\
ETTm1-S & 24 & $4.13\times10^{-4}$ & $8.19\times10^{-3}$ & 256 & 8 & 32 & 7 & 4 \\
ETTm2-S & 24 & $7.91\times10^{-4}$ & $4.85\times10^{-3}$ & 256 & 8 & 32 & 7 & 4 \\
Electricity-S & 24 & $3.18\times10^{-4}$ & $8.97\times10^{-3}$ & 512 & 8 & 16 & 7 & 4 \\
Traffic-S & 24 & $3.87\times10^{-4}$ & $9.48\times10^{-3}$ & 256 & 8 & 32 & 7 & 4 \\
Solar-S & 24 & $6.59\times10^{-4}$ & $8.39\times10^{-6}$ & 32 & 2 & 32 & 7 & 4 \\
Exchange-S & 30 & $9.69\times10^{-4}$ & $7.97\times10^{-4}$ & 128 & 0 & 32 & 2 & 16 \\
\bottomrule
\end{tabular}}
\caption{Best hyperparameters selected by Optuna for each short-term dataset (horizon $H$).}
\label{tab:hp_short_term}
\end{table}

\begin{table}[t]
\centering
\footnotesize
\setlength{\tabcolsep}{1mm}
\resizebox{\columnwidth}{!}{\begin{tabular}{l c c c c c c c c}
\toprule
Dataset & $H$ & LR & Weight decay & $C_{\mathrm{hid}}$ & $K$ & $N_b$ & $k$ & $k_f$ \\
\midrule
\multirow{4}{*}{ETTh1-L}
 & 96 & $2.16\times10^{-4}$ & $3.61\times10^{-6}$ & 64 & 2 & 16 & 25 & 4 \\
 & 192 & $1.08\times10^{-4}$ & $1.20\times10^{-5}$ & 64 & 2 & 16 & 13 & 16 \\
 & 336 & $1.24\times10^{-4}$ & $6.52\times10^{-5}$ & 128 & 1 & 16 & 43 & 8 \\
 & 720 & $7.93\times10^{-4}$ & $8.59\times10^{-3}$ & 256 & 0 & 32 & 46 & 16 \\
\midrule
\multirow{4}{*}{ETTh2-L}
 & 96 & $4.28\times10^{-5}$ & $2.49\times10^{-3}$ & 256 & 8 & 32 & 25 & 4 \\
 & 192 & $9.29\times10^{-4}$ & $1.16\times10^{-4}$ & 32 & 1 & 32 & 7 & 32 \\
 & 336 & $3.70\times10^{-4}$ & $1.97\times10^{-5}$ & 64 & 1 & 16 & 43 & 8 \\
 & 720 & $3.16\times10^{-6}$ & $1.67\times10^{-4}$ & 512 & 8 & 4 & 46 & 16 \\
\midrule
\multirow{4}{*}{ETTm1-L}
 & 96 & $7.15\times10^{-4}$ & $7.98\times10^{-3}$ & 256 & 8 & 32 & 25 & 4 \\
 & 192 & $7.15\times10^{-4}$ & $7.98\times10^{-3}$ & 256 & 8 & 32 & 49 & 4 \\
 & 336 & $3.09\times10^{-5}$ & $1.07\times10^{-6}$ & 512 & 8 & 32 & 85 & 4 \\
 & 720 & $1.24\times10^{-4}$ & $6.52\times10^{-5}$ & 128 & 1 & 16 & 91 & 8 \\
\midrule
\multirow{4}{*}{ETTm2-L}
 & 96 & $2.65\times10^{-4}$ & $2.47\times10^{-3}$ & 32 & 8 & 32 & 13 & 8 \\
 & 192 & $6.36\times10^{-5}$ & $2.00\times10^{-3}$ & 256 & 8 & 16 & 25 & 8 \\
 & 336 & $6.66\times10^{-5}$ & $1.37\times10^{-6}$ & 256 & 8 & 16 & 22 & 16 \\
 & 720 & $1.59\times10^{-5}$ & $8.64\times10^{-6}$ & 128 & 8 & 16 & 91 & 8 \\
\midrule
\multirow{4}{*}{Electricity-L}
 & 96 & $3.15\times10^{-4}$ & $1.42\times10^{-6}$ & 128 & 8 & 16 & 13 & 8 \\
 & 192 & $1.35\times10^{-4}$ & $7.99\times10^{-6}$ & 32 & 8 & 4 & 25 & 8 \\
 & 336 & $3.35\times10^{-5}$ & $1.24\times10^{-4}$ & 128 & 8 & 8 & 85 & 4 \\
 & 720 & $1.18\times10^{-4}$ & $1.48\times10^{-6}$ & 128 & 8 & 4 & 91 & 8 \\
\midrule
\multirow{4}{*}{Traffic-L}
 & 96 & $3.07\times10^{-4}$ & $1.30\times10^{-6}$ & 512 & 8 & 32 & 25 & 4 \\
 & 192 & $9.69\times10^{-4}$ & $1.67\times10^{-4}$ & 32 & 4 & 32 & 7 & 32 \\
 & 336 & $3.46\times10^{-4}$ & $1.04\times10^{-6}$ & 128 & 1 & 16 & 43 & 8 \\
 & 720 & $3.46\times10^{-4}$ & $1.04\times10^{-6}$ & 128 & 1 & 16 & 91 & 8 \\
\midrule
\multirow{4}{*}{Exchange-L}
 & 96 & $1.58\times10^{-5}$ & $2.79\times10^{-6}$ & 512 & 4 & 8 & 4 & 32 \\
 & 192 & $9.95\times10^{-4}$ & $8.19\times10^{-3}$ & 256 & 0 & 32 & 13 & 16 \\
 & 336 & $3.53\times10^{-5}$ & $1.89\times10^{-6}$ & 512 & 4 & 4 & 11 & 32 \\
 & 720 & $4.89\times10^{-4}$ & $2.95\times10^{-5}$ & 64 & 2 & 32 & 181 & 4 \\
\midrule
\multirow{4}{*}{Weather-L}
 & 96 & $9.60\times10^{-4}$ & $2.37\times10^{-4}$ & 256 & 8 & 16 & 13 & 8 \\
 & 192 & $4.89\times10^{-4}$ & $2.95\times10^{-5}$ & 64 & 2 & 32 & 49 & 4 \\
 & 336 & $3.82\times10^{-4}$ & $8.19\times10^{-3}$ & 256 & 8 & 32 & 85 & 4 \\
 & 720 & $2.35\times10^{-4}$ & $8.39\times10^{-6}$ & 32 & 2 & 32 & 91 & 8 \\
\midrule
\multirow{4}{*}{ILI-L}
 & 24 & $9.79\times10^{-4}$ & $2.82\times10^{-4}$ & 512 & 0 & 4 & 1 & 32 \\
 & 36 & $9.69\times10^{-4}$ & $4.01\times10^{-4}$ & 256 & 4 & 16 & 2 & 32 \\
 & 48 & $9.95\times10^{-4}$ & $4.58\times10^{-4}$ & 256 & 0 & 32 & 7 & 8 \\
 & 60 & $3.77\times10^{-4}$ & $1.97\times10^{-6}$ & 32 & 8 & 32 & 8 & 8 \\
\bottomrule
\end{tabular}}
\caption{Best hyperparameters selected by Optuna for each long-term dataset, per forecasting horizon $H$.}
\label{tab:hp_long_term}
\end{table}

\section{Hyperparameter selection and Implementation details}
\label{sec:appendix:hpselectImplementation}

For each dataset--horizon pair, we tune the learning rate, weight decay,
SplineNet-Conv1D hidden width $C_{\mathrm{hid}}$, number of
\ross{} blocks $K$, number of spline bins $N_b$, and kernel size $k$ via
Optuna described in the main text (20 trials per
configuration). We use AdamW optimizer with weight decay for training.
\Cref{tab:hp_short_term,tab:hp_long_term} report the
resulting best configuration for every short-term and long-term
dataset--horizon pair, respectively. Each selected configuration is then
retrained and evaluated with five seeds ($\{0,1,2,3,4\}$), and we report
the mean and standard deviation across these runs throughout the paper.

\paragraph{Search space.} Each trial draws the learning rate and weight
decay log-uniformly from $[10^{-6}, 10^{-3}]$ and $[10^{-6}, 10^{-2}]$
respectively, and the remaining flow hyperparameters from fixed
categorical grids: $C_{\mathrm{hid}} \in \{32, 64, 128, 256, 512\}$,
$K \in \{0, 1, 2, 4, 8\}$, and $N_b \in \{4, 8, 16, 32\}$. SplineNet's
Conv1D kernel size is not tuned directly; instead Optuna selects a kernel
factor $k_f \in \{4, 8, 16, 32\}$, and the kernel size is derived as
$k = \lfloor H/k_f \rfloor + 1$ for horizon $H$, so it automatically scales
with the forecast length. \Cref{tab:hp_short_term,tab:hp_long_term} report
both the derived $k$ and the underlying tuned $k_f$ for every selected
configuration.

\paragraph{Compute infrastructure.} All experiments ran on an internal
Linux SLURM cluster using Python~3.10 and PyTorch~2.13, spanning several
GPU generations, from NVIDIA GTX~1080~Ti (11\,GB) up to NVIDIA
A40 (48\,GB). Jobs were assigned by memory requirement rather than
uniformly at random: short-horizon runs and datasets with few channels
ran on the smaller-VRAM 1080~Ti/2080~Ti/A4000 nodes, while the two
highest-channel-count datasets like Electricity (321 channels long-term,
370 short-term) and Traffic (862 long-term, 963 short-term) and every
$H{=}720$ run, whose per-batch tensors are the largest in the benchmark,
were scheduled on the higher-VRAM 3090/4090/A40 nodes (24--48\,GB) to
avoid out-of-memory failures. For the memory requirement and inference time
analysis (ablation), all the experiments were run on the same NVIDIA RTX 2070 SUPER with
8 GB of memory for comparability.

\section{Visual Analysis}
\label{sec:appendix:visualanalysis}

\begin{figure*}[t]
    \centering
    \begin{subfigure}{0.24\textwidth}
        \centering
        \includegraphics[width=\linewidth]{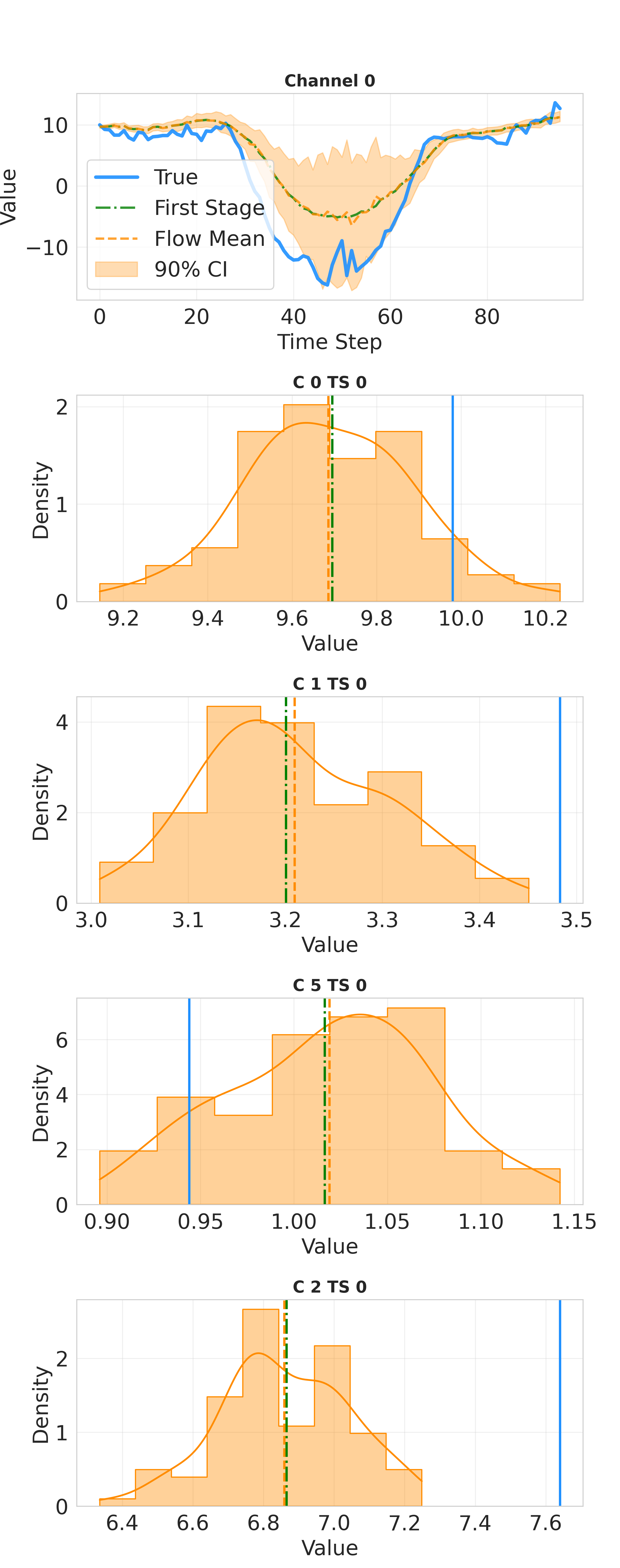}
        \caption{MVE-2S (Gaussian)}
        \label{fig:visual_mve2s}
    \end{subfigure}
    \hfill
    \begin{subfigure}{0.235\textwidth}
        \centering
        \includegraphics[width=\linewidth]{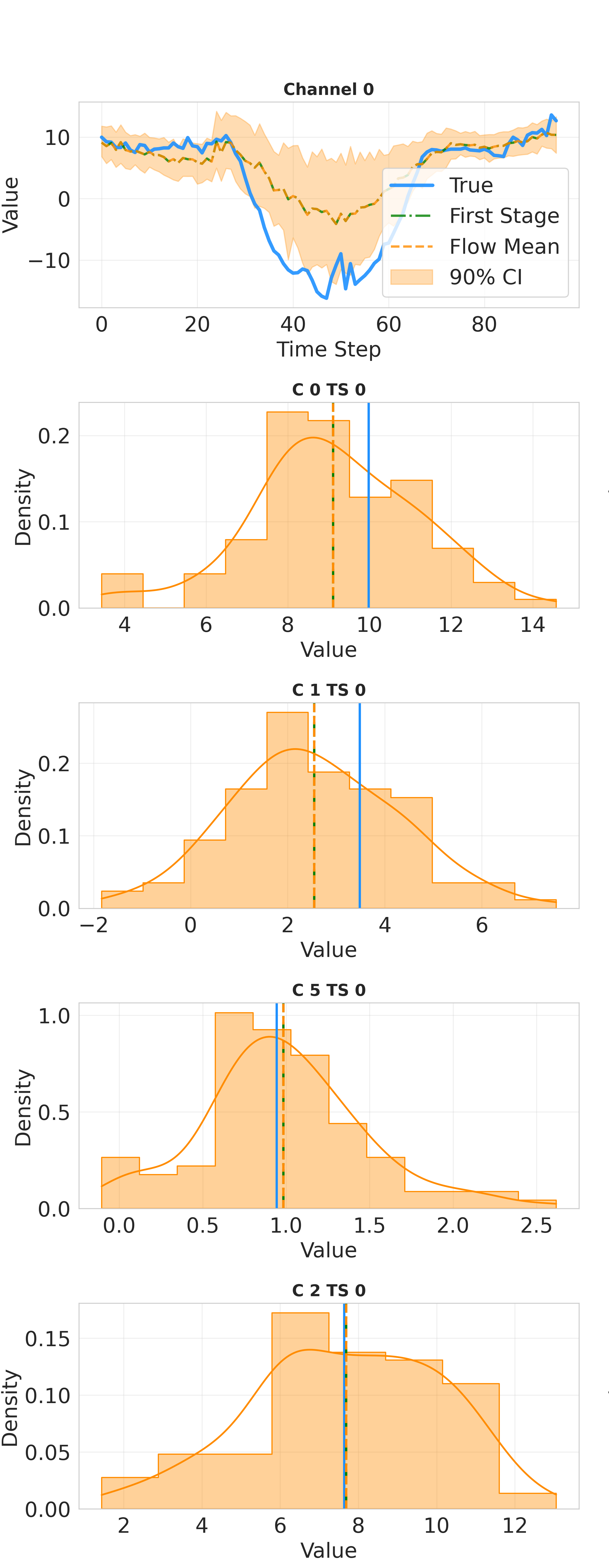}
        \caption{$K^2$VAE}
        \label{fig:visual_k2vae}
    \end{subfigure}
    \hfill
    \begin{subfigure}{0.24\textwidth}
        \centering
        \includegraphics[width=\linewidth]{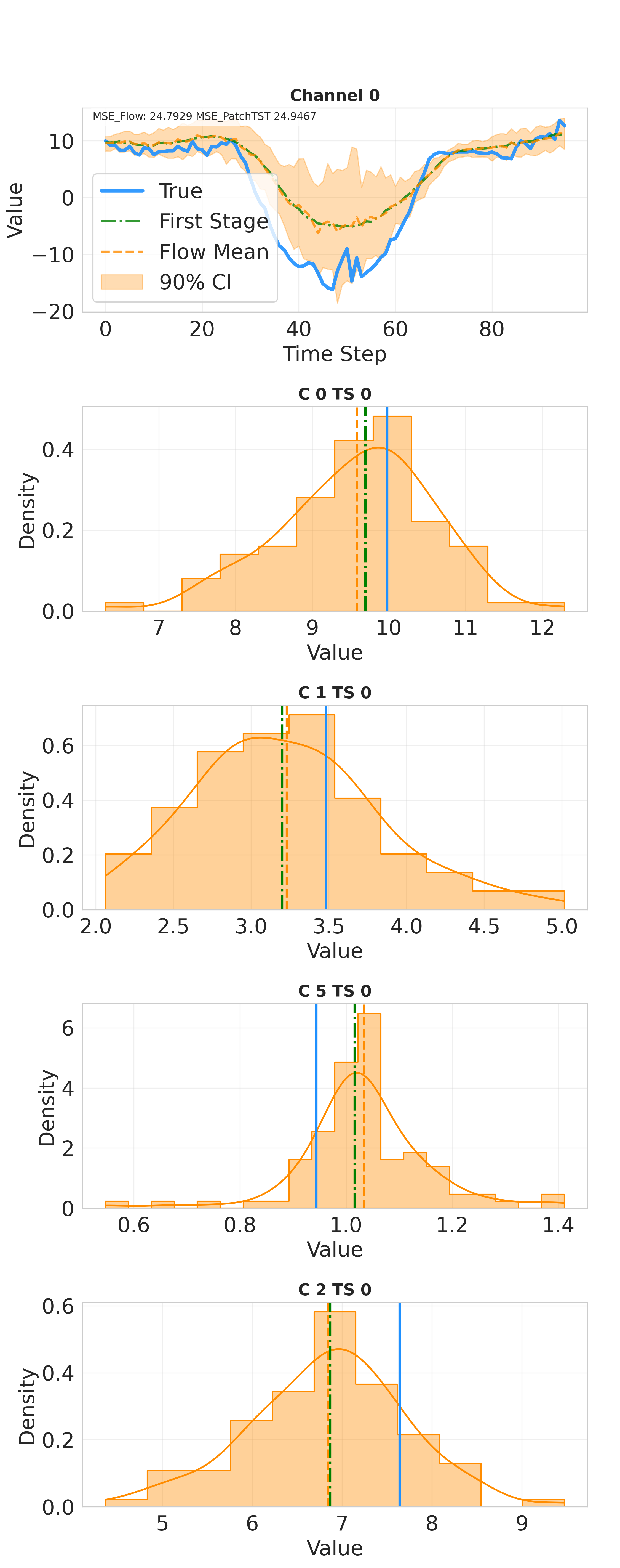}
        \caption{\model{}}
        \label{fig:visual_torf}
    \end{subfigure}
    \hfill
    \begin{subfigure}{0.24\textwidth}
        \centering
        \includegraphics[width=\linewidth]{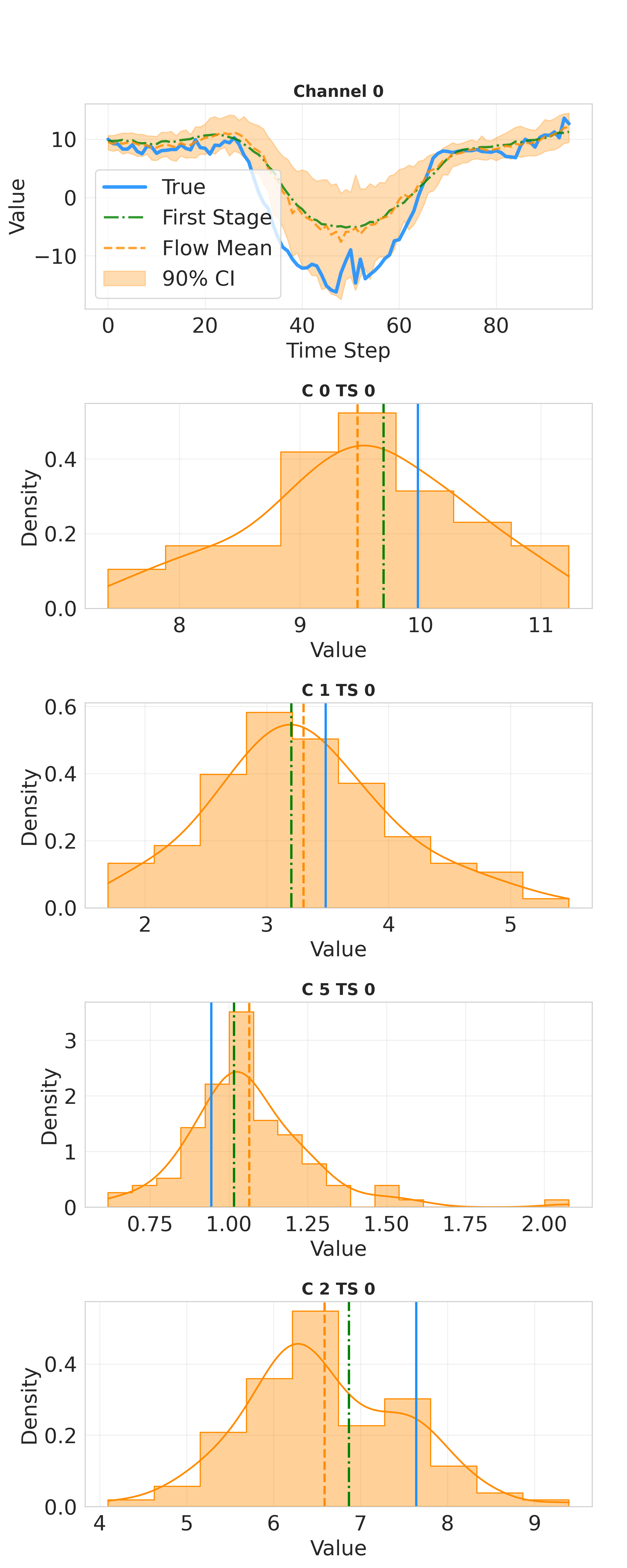}
        \caption{TORF-Mixture}
        \label{fig:visual_mixture}
    \end{subfigure}
    \caption{Qualitative comparison of predictive distributions for four
    methods (MVE-2Stage, $K^2$VAE, \model{}, \model-Mixture) on the same ETTm1 ($H{=}96$) test window and channel. In each
    panel, the top plot shows the forecast horizon (true trajectory in
    blue, Stage-1 point forecast in green, predictive mean in dashed
    orange, 90\% CI shaded); the four plots above show the predictive
    density at fixed (channel, time step) pairs C0/C1/C5/C2 at TS~0, with
    the true value (blue), Stage-1 forecast (green), and predictive mean from the flow
    (dashed orange) marked as vertical lines.}
    \label{fig:visual_comparison}
\end{figure*}

\Cref{fig:visual_comparison} compares all four methods on the exact same
test window and channel, so the differences in predictive shape are
directly attributable to the model rather than to the underlying data.
We use $n{=}100$ samples to generate each predictive distribution shown,
consistent with the convention in prior work~\citep{wu2025kvae}; with
this few samples, the histograms are noisy estimates of the true
underlying density rather than exact densities, so fine-grained
differences in shape should be read qualitatively.
MVE-2S's residual histograms are the narrowest of the four, consistent
with it being an exactly Gaussian model (\Cref{sec:eqscalinggauss}): on
C1~TS0 and C2~TS0 in particular, the true value falls almost entirely
outside its 90\% mass, a visual symptom of the overconfidence that a
purely affine, unimodal transform produces once the true residual is not
Gaussian. $K^2$VAE sits at the opposite extreme: its densities are
visibly flatter and wider, and on C1~TS0 its mass extends into negative
values even though every true value of that channel is positive,
suggesting its predictive distribution is not well constrained by the
data's actual support. \model{} and TORF-Mixture both produce densities
that are noticeably tighter around the true value than either baseline
while still crossing it, and their Stage-1 (green) and predictive-mean
(orange) lines stay visually coincident in every panel, the
mean-preservation guarantee of \Cref{lem:odd-flow} holding exactly, even
though the shape around that mean adapts per channel. The one visible
difference between the two \model{} variants is on C2~TS0, where
TORF-Mixture's density is noticeably more skewed than the
single-component \model{}'s, illustrating in a single test window the
asymmetric residual modeling that \Cref{sec:ablation:mixOddFlows}
motivates analytically.

\end{document}